\documentclass[10pt]{article}

\usepackage[a4paper, margin=1in]{geometry}

\usepackage{graphicx}
\usepackage{amsmath, amssymb, amsthm, mathtools}
\usepackage[hidelinks]{hyperref}
\usepackage{titletoc}
\usepackage[numbers,sort&compress]{natbib}
\usepackage{booktabs}
\usepackage{multirow}
\usepackage{siunitx}
\usepackage{float}
\usepackage{algorithm}
\usepackage{algpseudocode}
\usepackage{subcaption}
\usepackage{longtable}

\DeclareCaptionFont{paper}{\normalsize}
\newcommand{\DisplaySpace}[1]{%
  \par\begingroup\dimen0=#1\relax
  \ifdim\dimexpr\pagegoal-\pagetotal\relax<\dimen0\newpage\fi
  \endgroup
}

\theoremstyle{plain}
\newtheorem{theorem}{Theorem}[section]

\newtheorem{lemma}[theorem]{Lemma}
\newtheorem{corollary}[theorem]{Corollary}
\theoremstyle{definition}

\newtheorem{claim}[theorem]{Claim}
\newtheorem{assumption}[theorem]{Assumption}
\theoremstyle{remark}

\newcommand{\FullCent}{\textsc{FullCent}}
\DeclareRobustCommand{\eFullCent}{%
  \ifmmode
    \text{\emph{Enhanced}-\textsc{FullCent}}%
  \else
    \emph{Enhanced}-\textsc{FullCent}%
  \fi
}

\newcommand{\AdaCent}{\textsc{AdaCent}}
\DeclareRobustCommand{\eAdaCent}{%
  \ifmmode
    \text{\emph{Enhanced}-\textsc{AdaCent}}%
  \else
    \emph{Enhanced}-\textsc{AdaCent}%
  \fi
}

\newcommand{\KCenter}{\textsc{KCenter}}
\DeclareRobustCommand{\eKCenter}{%
  \ifmmode
    \text{\emph{Enhanced}-\textsc{KCenter}}%
  \else
    \emph{Enhanced}-\textsc{KCenter}%
  \fi
}

\newcommand{\x}{\mathbf{x}}

\newcommand{\bc}{\mathbf{c}}
\newcommand{\C}{\mathcal{C}}
\newcommand{\A}{\mathcal{A}}
\newcommand{\X}{\mathcal{X}}
\newcommand{\cl}{\mathcal{S}}
\newcommand{\dti}{\tilde{d}}
\newcommand{\copt}{\cl_{\text{opt}}}
\newcommand{\xopt}{\x_{\text{opt}}}
\newcommand{\xsub}{\x_{\text{sub}}}
\newcommand{\alg}{Alg}
\newcommand{\pmiss}{p_{\text{miss}}}
\newcommand{\pnew}{p_{\text{new}}}
\newcommand{\nopt}{k_{\text{non-opt}}}
\newcommand{\eps}{\epsilon}
\newcommand{\inveps}{\frac{1}{\epsilon}}

\newcommand{\pr}{\mathbb{P}}
\newcommand{\R}{\mathbb{R}}

\title{Hyperparameter Optimization via Budget Allocation
\\in Lipschitz Spaces}

\author{
MohammadHossein Bateni\thanks{Google Research, New York City, New York, USA. \texttt{bateni@google.com}} \and
Hossein Esfandiari\thanks{Google Research, London, UK. \texttt{esfandiari@google.com}} \and
Samira HosseinGhorban\thanks{Institute for Research in Fundamental Sciences, School of Computer Science, Tehran, Iran. \texttt{s.hosseinghorban@ipm.ir}} \and
Alireza Mirrokni\thanks{EPFL, Lausanne, Switzerland. \texttt{seyedalireza.mirroknibanadaki@epfl.ch}} \and
Radin Shahdaei\thanks{Aalto University, Espoo, Finland. \texttt{radin.shahdaei@aalto.fi}}
}
\date{\small Authors are listed in alphabetical order.}

\begin{document}

\maketitle

\begin{abstract}
Hyperparameter optimization and model selection often require allocating a limited evaluation budget across many candidate configurations, where performance improves with additional training or resources. We study this setting when each configuration has an unknown performance curve over budget. Similarity between configurations induces similarity between their curves, capturing geometric structure in the search space. We formalize the resulting decision problem as Unknown Value Probing and propose budget allocation strategies that combine principled exploration of the configuration space with value-informed concentration on regions supported by observed evidence. We prove worst-case approximation guarantees for our approach and show these guarantees are near-optimal by matching corresponding hardness barriers. Extensive experiments on large benchmark suites spanning more than 250 experimental settings demonstrate consistent gains over strong baselines in both anytime behavior and final incumbent performance.\footnote{Code is available at \url{https://github.com/alirezamirrokni/HPO-via-Budget-Allocation-in-Lipschitz-Spaces}.}
\end{abstract}

\section{Introduction}

Developing high-performing machine learning systems requires navigating a complex configuration space under a finite computational budget. Whether selecting architectures, tuning hyperparameters, or performing feature selection, practitioners must decide which candidates to evaluate and how much resource to allocate to each. Since performance depends on fidelity, such as training epochs or dataset size, fully training every candidate is often infeasible. Efficient strategies must therefore balance broad exploration with the depth needed to identify promising configurations.

This process is inherently uncertain: performance curves are initially unknown and revealed only through costly sequential evaluations. At the same time, these curves are not arbitrary. Nearby configurations in the hyperparameter space often exhibit similar performance profiles, so observations from one candidate can provide evidence about the potential of its neighborhood. We formalize this setting as \emph{Unknown Value Probing}, treating the search space as a metric space where distance captures similarity and guides budget allocation.

The resulting problem requires more than static coverage: a purely geometric strategy explores broadly but treats low- and high-potential regions alike. Our approach combines greedy coverage with value-informed feedback, using \emph{enhanced distances} to steer exploration away from low-value regions and pruning rules to discard configurations whose upper bounds cannot improve the current best value. Together, these mechanisms reallocate budget toward more informative parts of the search space.

These value-informed mechanisms are supported by worst-case approximation guarantees. We prove that the proposed algorithms match the problem's fundamental hardness barriers up to constant factors, and validate them in a large-scale study on YAHPO Gym spanning more than 250 experimental settings, where they consistently improve over strong HPO baselines.

\subsection{Problem Formulation}

We formally define the budget allocation problem, which we term \emph{Unknown Value Probing (UVP)}. Let
$A(\x, b) : \R^p \times \{0,\ldots,T\} \to [0,1]$
denote an unknown value function, where $\x\in\R^p$ represents a configuration embedding, and $b\in\{0,\ldots,T\}$ denotes its local fidelity or per-configuration budget level. Also let
\(
d(\cdot,\cdot):\R^p\times\R^p\to\R_{\geq 0}
\)
be a metric distance on \(\R^p\). In the context of Hyperparameter Optimization (HPO), $A(\x, b)$ represents the validation accuracy obtained by training configuration $\x$ up to fidelity level $b$, such as after a given number of epochs, time steps, or data-fraction increments. Given a finite set of configurations $\X = \{\x_1, \ldots, \x_n\} \subset \R^p$, the goal of UVP is to identify the configuration that achieves the maximum value under a fixed global budget constraint. Formally, we optimize:
\begin{equation}
\max_{b_1,\ldots,b_n} \left\{ \max_{i \in [n]} \; A(\x_i, b_i) \right\}
\quad \text{subject to} \quad \sum_{i=1}^n b_i \leq B,
\label{eq:budget_alloc}
\end{equation}
where $b_i \in \{0, \dots, T\}$ is the local budget allocated to configuration $\x_i$, and $B \ge T$ is the global budget available across all configurations. We define \(A(\mathbf{x},0)=0\) for every
\(\mathbf{x}\in\mathcal X\); therefore, \(b_i=0\) represents an
unevaluated configuration and contributes no observed value. At the onset, the function values $A(\x_i, \cdot)$ are unknown, and only the embeddings $\x_i$ are observable. This formulation naturally models HPO, where the objective is to discover the best-performing model configuration within a strict computational budget $B$.

\paragraph{\textbf{Assumptions.}} We first assume that increasing the allocated resource does not degrade performance.

\begin{assumption}[Monotonicity in Budget] \label{assump:assump1}
For any fixed configuration $\x_i \in \X$, the value function $A(\x_i,\cdot)$ is non-decreasing with respect to the local budget:
\[
b_1 \le b_2 \implies A(\x_i,b_1) \le A(\x_i,b_2).
\]
\end{assumption}

We also impose smoothness across configurations so that nearby embeddings have similar performance profiles.

\begin{assumption}[Smoothness across Configurations] \label{assump:assump2}
There exists a constant $\epsilon > 0$ such that for all $\x_i, \x_j \in \X$ and $b \in [T]$:
\[
A(\x_i,b) \ge \left( 1 - \epsilon d(\x_i, \x_j) \right) A(\x_j, b).
\]
As detailed in Appendix~\ref{sec:A1}, Assumption~\ref{assump:assump2} implies Lipschitz-type continuity on the finite set $\X$, and we verify its empirical validity in Section~\ref{sec:val_assump}.
\end{assumption}


\paragraph{\textbf{Oracle Access Model.}}
Access to $A(\x,b)$ is restricted through a sequential oracle, denoted by $\mathrm{eval}(\x)$. Each oracle query consumes one unit of the global budget and reveals the next unobserved point on the performance curve. Specifically, if $\mathrm{eval}(\x)$ has been queried $b$ times, the subsequent query returns $A(\x,b+1)$. Consequently, the fidelity assigned to a configuration is exactly the number of oracle queries allocated to it, and a fidelity level $b+1$ can be observed only after all preceding levels $1,\ldots,b$ have been evaluated. This sequential-prefix access and the shared-budget objective distinguish UVP from classical Lipschitz and metric-space optimization, which typically assume direct access to a full-fidelity objective together with a known smoothness constant.

\paragraph{\textbf{Notation.}}
We use bold lowercase letters, such as $\x$ and $\bc$, to denote configurations in $\R^p$, and calligraphic uppercase letters, such as $\X$, $\C$, and $\A$, to denote sets. The vector $\bc$ denotes a cluster center within $\X$, associated with the cluster $\cl(\bc)$. For a nonempty center set $\C\subseteq\X$, define its covering radius and the discrete $k$-center optimum by
\begin{equation}
 r(\C)=\max_{\x\in\X}\min_{\bc\in\C}d(\x,  \bc),
\qquad
r_k^\star=\min_{\substack{\C\subseteq\X\\|\C|=k}}r(\C).   
\end{equation}
The greedy farthest-first algorithm chooses an arbitrary first center and repeatedly adds a point maximizing its distance to the current center set; its final radius is at most $2r_k^\star$~\cite{gonzalez1985clustering}. We let
\(
\x^\star \in \operatorname*{arg\,max}_{\x\in\X} A(\x,T)
\)
denote an optimal configuration. An algorithm is an $\alpha$-approximation if its returned  value is at least $\alpha\mathrm{OPT}$ for every admissible instance, where $\mathrm{OPT}$ is the maximum value achievable.

\subsection{Our Results}

We introduce algorithms for UVP that combine greedy $k$-center coverage, adaptive pruning, and \emph{enhanced distances}, supported by worst-case approximation guarantees and large-scale experiments.

\paragraph{\textbf{Theory.}}
Let $k=\min\{|\X|,\lfloor B/T \rfloor\}$. We first analyze the full-fidelity clustering baselines (\FullCent{} and \eFullCent{}), and then the adaptive variant \AdaCent{} under diminishing returns.

\begin{theorem}[See Theorem~\ref{thm:FC-approx}, Theorem~\ref{thm:EFC-approx}, and Corollary~\ref{col:FC-hard}]
Under Assumptions~\ref{assump:assump1}--\ref{assump:assump2}, \FullCent{} and \eFullCent{} attain an approximation factor of $(1 - 2\epsilon r^\star_k)$. This dependence on the covering radius is optimal up to constant factors.
\end{theorem}

\begin{theorem}[See Theorem~\ref{thm:AC-approx} and Corollary~\ref{col:AC-hard}]
Under Assumptions~\ref{assump:assump1}--\ref{assump:assump2} and~\ref{assump:concavity}, with the completion-budget restriction of Algorithm~\ref{alg:adacent}, \AdaCent{} attains an approximation factor of $(1 - 2\epsilon r^\star_k)$. This dependence is optimal up to constant factors for concave learning curves.
\end{theorem}

\paragraph{\textbf{Experiments.}}
We conduct a large-scale evaluation under fixed global budgets on the YAHPO Gym multi-fidelity benchmarks, spanning 250+ task settings across deep-learning and classical-ML pipelines, with 30 independent runs per task. All methods use the same total budget and fidelity grids. We compare \eFullCent{}, \AdaCent{} and \eAdaCent{} against Hyperband, BOHB, SMAC, and random search. Across the benchmark suite, our methods improve anytime best-so-far performance and mean-rank trajectories, with \eAdaCent{} achieving the best overall rank across the budget horizon and stronger final incumbents (Figure~\ref{fig:meanrank:all}). Representative budget--accuracy curves appear in Figures~\ref{fig:acc:lcbench} and~\ref{fig:acc:rbv2}, and the complete results are provided in Appendix~\ref{sec:figs}. Additional synthetic and multi-fidelity Bayesian optimization comparisons, together with empirical validation of our structural assumptions, further support the observed gains.







\section{Related Work}
\label{sec:related_works}

Hyperparameter Optimization (HPO) has evolved from black-box strategies to multi-fidelity frameworks that exploit learning algorithm structure. We categorize the literature into classical methods, multi-fidelity strategies, and recent LLM-guided, semantic, and efficiency-oriented approaches.

The shift from Random Search \citep{bergstra2012random} to model-based approaches was pivotal for efficiency. Bayesian optimization uses probabilistic surrogate models to guide the
search \cite{snoek2012practical}. Sequential Model-based Algorithm Configuration (SMAC) \citep{hutter2011smac} uses Random Forests to model the response surface, handling discrete parameters and conditional dependencies. \cite{snoek2014input} further improved BO with Input Warping, addressing non-stationary landscapes by mapping inputs to a latent stationary space.

To reduce DNN training costs, recent gray-box methods use partial feedback. \cite{domhan2015speeding} and \cite{klein2017learning} pioneered learning curve extrapolation to early-stop unpromising configurations. This was formalized by \cite{jamieson2016non} with the Successive Halving Algorithm (SHA), which allocates exponentially more budget to promising arms. Hyperband \citep{li2017hyperband} builds on SHA by balancing exploration and exploitation, while BOHB \citep{falkner2018bohb} combines Hyperband with the TPE variant of Bayesian Optimization.

Recent HPO research has shifted toward agentic, semantic, and sustainable paradigms. Large Language Models (LLMs) are integrated into the optimization loop as reasoning agents in \emph{AgentHPO} \citep{liu2025agenthpo} or as semantic warm-start mechanisms in \emph{LLAMBO} \citep{liu2024llambo}.
Multi-objective frameworks such as \emph{PriMO} incorporate user priors into
multi-objective hyperparameter optimization \cite{primo2025}.
In the tabular domain, foundation models such as \emph{TabPFN-2.5} \citep{grinsztajn2025tabpfn25} challenge iterative tuning through in-context learning. Efficiency has also been pursued through structural approximations, such as \emph{Frozen Layers} \citep{carstensen2025frozen}, and transfer formulations such as \emph{Ordered Transfer HPO} \citep{hellan2025othpo}. Finally, the field has emphasized standardization to address pathologies like overtuning \citep{schneider2025overtuning}, using benchmarking suites such as \emph{carps} \citep{benjamins2025carps}.

Reproducible benchmarks are essential for validating HPO algorithms. \cite{pfisterer2022yahpo} introduced YAHPO Gym, a surrogate-based benchmark suite for realistic budget-constrained evaluation. It has been adopted for large-scale comparisons in CARPS \citep{benjamins2025carps}, multi-objective scenarios in PriMO \citep{primo2025}, and overtuning analysis \citep{schneider2025overtuning}.

\section{Methodology}
\label{sec:methodology}

In this section, we describe our approach for solving the UVP problem. Our goal is to identify the optimal configuration $\x^\star$, whose full-fidelity value $A(\x^\star,T)$ attains the UVP optimum.

\subsection{High-Level Strategy}
Our core objective is to \emph{cover the configuration space in the most efficient way} to locate the optimum $\x^\star$. To achieve this under a limited budget, our approach relies on three key principles:
\begin{enumerate}
    \item \textbf{Diversity via $k$-center:} We use a greedy geometric covering strategy to select diverse configurations and promote broad coverage of the search space.
    \item \textbf{Exploitation via Space Warping:} We warp the geometry of the search space using \emph{enhanced distances}, compressing the selection scores around suboptimal centers and steering exploration away from low-potential regions.
    \item \textbf{Efficiency via Pruning:} We stop allocating budget to configurations whose certified upper bounds fall below the current incumbent, and reallocate the saved budget to more promising candidates.
\end{enumerate}

 The algorithms realizing these principles are detailed below. 

\subsection{(\textit{Enhanced}-) \FullCent{}}
\label{subsec:fullcent}

We first introduce \FullCent{}, a baseline approach grounded in the classical $k$-center clustering problem. While conceptually straightforward, it establishes the theoretical foundation for our coverage strategy. The algorithm iteratively builds a set of evaluated configurations, denoted as centers $\mathcal{C}$. \FullCent{} employs a greedy max--min criterion: at each step, it selects the candidate $\x_{\text{new}}$ that is farthest from the previously selected centers:
\[
\x_{\text{new}} = \operatorname*{arg\,max}_{\x \in \mathcal{X} \setminus \mathcal{C}} \; \min_{\bc \in \mathcal{C}} \; d(\x, \bc).
\]
By maximizing the minimum distance to the set $\mathcal{C}$, the algorithm ensures that the selected configurations are spread as widely as possible across $\mathcal{X}$.

\paragraph{\textbf{Enhanced Distances.}}
To incorporate exploitation, enhanced variants warp the search geometry.
Each selected center $\bc$ stores two values: its latest observed value
$V_\bc$ and a fixed selection value $S_\bc$ used only for enhanced distance
computations. Since the standard metric is performance-agnostic, we replace
it with the enhanced distance score $\tilde d$, defined for
$S_\bc>0$ as
\begin{equation}
\tilde d(\x,\bc)
=
\min\left\{
d(\x,\bc),\;
\eta_\bc d(\x,\bc)
-{(\eta_\bc-1)}/{\epsilon}
\right\},
\label{eq:enhanced_dist}
\end{equation}
where $\eta_\bc={\max_{\bc'\in\C}S_{\bc'}}/{S_\bc}$.
If $S_\bc=0$, we set $\tilde d(\x,\bc)=d(\x,\bc)$ to avoid division by zero.
Operationally, for suboptimal centers with $\eta_\bc>1$, this transformation
lowers nearby distance scores. This expands their effective coverage in the
warped geometry and discourages the greedy max--min algorithm from selecting
new candidates in low-potential regions
(see Fig.~\ref{fig:enillustration}). In \eFullCent{}, every center used to
select the next center has reached fidelity $T$; after evaluating $\bc$, we
therefore set $S_\bc\gets A(\bc,T)$.
Section~\ref{sec:theory} provides the derivation and theoretical justification.
Appendix~\ref{app:illustrative-example} illustrates how enhanced distances
redirect selection toward a high-value region, and
Appendix~\ref{app:synthetic-fullcent} provides additional synthetic comparisons
between \FullCent{} and \eFullCent{}.
The pseudocode for (\textit{Enhanced}-) \FullCent{} is given in
Algorithm~\ref{alg:fullcent}\footnote{Throughout the pseudocodes, $R$ denotes the remaining global budget and is initialized as $R=B$. A call to $\mathrm{eval}$ is made only when $R>0$, and every call is followed by $R\gets R-1$. Therefore, neither algorithm can exceed the global budget $B$.}.

\begin{figure}[t]
    \centering
    \includegraphics[width=1\linewidth,trim={4.860bp 3.780bp 3.780bp 4.320bp},clip]{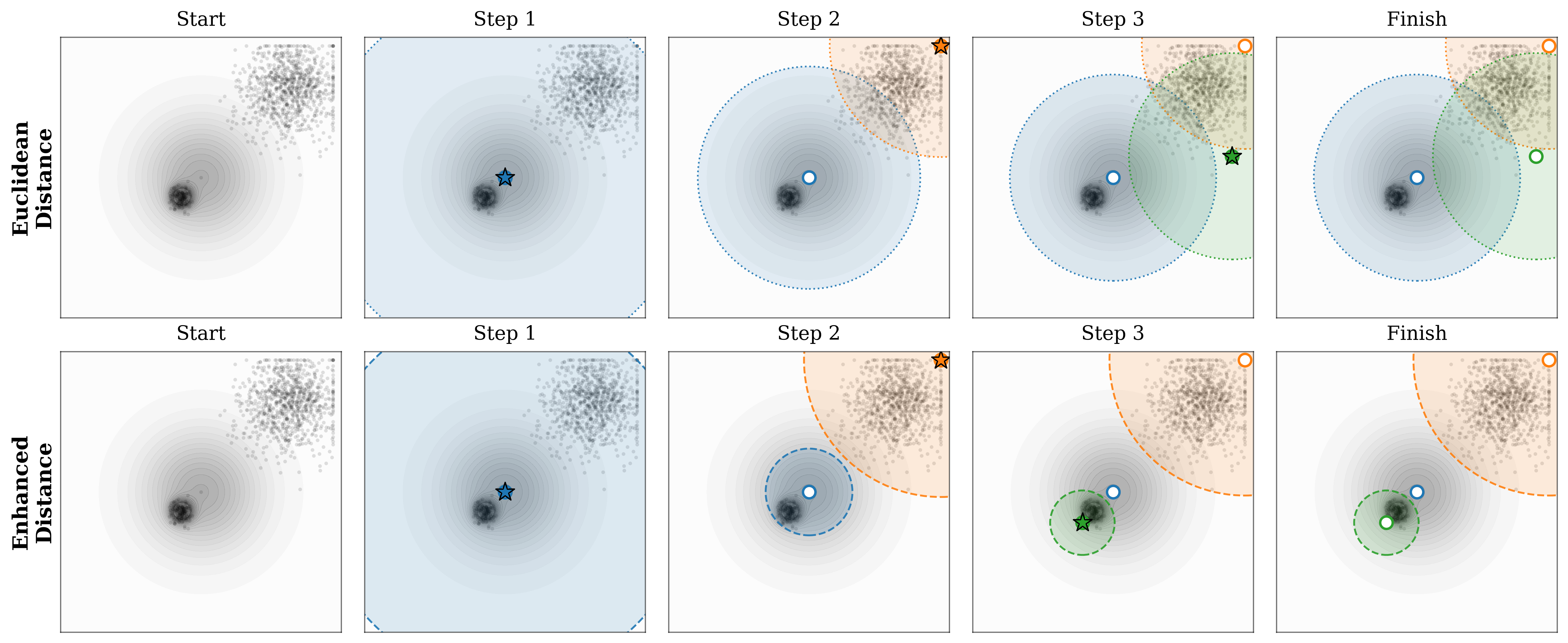}
\caption{Euclidean versus enhanced-distance greedy $k$-center. Gray points denote candidate configurations, and contours show $A(\mathbf{x},T)$. Columns trace three center selections; colors indicate selection order and stars mark the newly added center. Translucent regions show coverage under Euclidean distance (top) and enhanced distance $\tilde d$ (bottom), which expands coverage around low-value centers while preserving that of the best center, steering later selections away from low-potential regions.}
    \label{fig:enillustration}
\end{figure}

\begin{algorithm}[H]
\caption{(\textit{Enhanced}-) \FullCent{}}
\label{alg:fullcent}
\begin{algorithmic}[1]
\State \textbf{Input:} $\X,B,T$; $\epsilon$ for \eFullCent{}
\State $R\gets B$; $\C\gets\emptyset$; $S\gets\emptyset$

\While{$R>0$ and $\C\ne\X$}
    \State $\displaystyle
        \bc\gets\operatorname*{arg\,max}_{\x\in\X\setminus\C}
        \min_{\bc'\in\C}\{d(\x,\bc') \text{ or }  \tilde d(\x,\bc') \}$

    \State $\C\gets\C\cup\{\bc\}$

    \For{$b=1,\ldots,T$}
        \State \textbf{If} $R>0$:
        $V_\bc\gets\mathrm{eval}(\bc)$;
        $R\gets R-1$
    \EndFor
    \State $S_\bc\gets V_\bc$
\EndWhile

\State $\displaystyle
\hat\x\gets\operatorname*{arg\,max}_{\bc\in\C}V_\bc$
\State \Return $\displaystyle
\left(\hat\x,V_{\hat\x}\right)$
\end{algorithmic}
\end{algorithm}

\subsection{(\textit{Enhanced}-) \AdaCent{}}
\label{subsec:adacent}

While \FullCent{} ensures theoretical coverage, it is computationally wasteful as it fully trains clearly suboptimal candidates. \AdaCent{} overcomes this via adaptive pruning, terminating poor runs early to reallocate resources toward broader and more effective exploration.

The core of \AdaCent{} is a mechanism to reallocate resources. By discarding configurations that cannot outperform the current incumbent, we avoid wasting evaluations on suboptimal candidates. For any active configuration $\x$, we track a Lower Bound ($\mathrm{LB}$) and an Upper Bound ($\mathrm{UB}$) on its final value $A(\x, T)$. If the optimistic potential of $\x$ (its $\mathrm{UB}$) is strictly dominated by the guaranteed performance ($\mathrm{LB}$) of another active candidate, $\x$ is immediately pruned:
\[ \text{Prune } \x \text{ if} \quad UB(\x) < \max_{\x' \in \mathcal{A}} LB(\x'), \]
where $\mathcal{A}$ denotes the active set. The advantage of this mechanism is that it allows us to cover a larger portion of the space within the same total budget $B$. When many configurations are pruned early, the saved budget is reinvested to introduce and evaluate more configurations than would be possible under a static allocation.

To balance pruning with systematic coverage, \AdaCent{} proceeds in rounds.
In each round, we select a batch of $m$ configurations via the max--min
criterion, evaluate them incrementally, and prune underperformers.
This strategy retains the theoretical guarantees of \FullCent{} while
improving practical efficiency. In \eAdaCent{}, an initial exploration
budget $b_0=\max\{1,\lfloor\delta T\rfloor\}$, with $0<\delta\le1$, warps the metric via enhanced distances to guide
subsequent batches. The selection value $S_\bc$ used in
\eqref{eq:enhanced_dist} is fixed after this phase:
$S_\bc \gets A(\bc,b_0)$. Enhanced distance selection therefore compares
centers at a common fidelity, even though their latest values $V_\bc$ may
later differ due to pruning and continued evaluation. The method has two
tunable parameters: the batch size $m$ and exploration fraction $\delta$.

Pruning requires tight estimates of the terminal value $A(x,T)$. By
monotonicity, the current value $A(x,t)$ is a valid 
$\mathrm{LB}$. For the  $\mathrm{UB}$, we use a predictive
strategy that models the remaining gain by linearly projecting the recent
trajectory. This is motivated by learning curves typically exhibiting
diminishing returns and approaching saturation, where improvements become
sub-linear. Consequently, linear extrapolation provides a plausible and
optimistic estimate of the value achievable at budget $T$.

A scheduling detail is needed to make the adaptive guarantee hold in the worst case. Let
\[
k=\min\{|\X|,\lfloor B/T\rfloor\}.
\]
We protect the first $k$ greedy centers by processing them in batches of at most $m$, completing or pruning each batch before opening the next, and admitting no center beyond this prefix until all $k$ are resolved. These protected batches cost at most $kT\le B$, so budget exhaustion cannot leave a protected center unfinished and unpruned. Afterward, the remaining budget is used for additional batches. This is a simple implementation of the completion-budget restriction required by Theorem~\ref{thm:AC-approx}. In our experiments the restriction is essentially free: explicitly enforcing it changes final accuracy by less than $0.3\%$ on average in each benchmark suite (Appendix~\ref{app:completion-reserve}).

The pseudocode is given in Algorithm~\ref{alg:adacent}.

\begin{algorithm}[H]
\caption{(\textit{Enhanced}-) \AdaCent{}}
\label{alg:adacent}
\begin{algorithmic}[1]
\State \textbf{Input:} $\X\ne\emptyset$, $B\ge T$, $m\ge1$, $0<\delta\le1$; $\epsilon$ for \eAdaCent{}
\State $b_0\gets\max\{1,\lfloor\delta T\rfloor\}$; $k\gets\min\{|\X|,\lfloor B/T\rfloor\}$
\State $R\gets B$; $\C\gets\emptyset$; $M\gets0$
\While{$\C\ne\X$ and $R\ge b_0$}
    \State $q\gets\min\{m,\ |\X\setminus\C|,\ \lfloor R/b_0\rfloor\}$; $\A\gets\emptyset$
    \If{$|\C|<k$} $q\gets\min\{q,k-|\C|\}$ \Comment{Resolve the protected prefix first}
    \EndIf
    \For{$i=1,\ldots,q$}
        \State $\displaystyle\bc\gets\operatorname*{arg\,max}_{\x\in\X\setminus\C}\min_{\bc'\in\C}\rho(\x,\bc')$
        \State $\C\gets\C\cup\{\bc\}$; $\A\gets\A\cup\{\bc\}$; $b_\bc\gets0$; $V_\bc\gets0$
        \State $\operatorname{Advance}(\bc,b_0)$; $S_\bc\gets V_\bc$
    \EndFor
    \State $\A\gets\{\bc\in\A:b_\bc<T,\ UB(\bc)\ge M\}$
    \While{$\A\ne\emptyset$ and $R>0$}
        \ForAll{$\bc\in\A$ \textbf{while} $R>0$}
            \State $\operatorname{Advance}(\bc,1)$
        \EndFor
        \State $\A\gets\{\bc\in\A:b_\bc<T,\ UB(\bc)\ge M\}$
    \EndWhile
\EndWhile
\State \Return $\left(\operatorname*{arg\,max}_{\bc\in\C}V_\bc,\ M\right)$
\end{algorithmic}
\end{algorithm}

Here $\rho=d$ for \AdaCent{} and $\rho=\tilde d$ for \eAdaCent{}, with the enhanced score always computed from all previously selected centers and their frozen selection values. When $\C=\emptyset$, the first center is arbitrary. The operation $\operatorname{Advance}(\bc,h)$ makes $h$ sequential oracle calls, updating $b_\bc$, $V_\bc$, the incumbent $M=\max_{\bc'\in\C}V_{\bc'}$, and $R$ after each call. We use a valid terminal upper bound $UB(\bc)$ (Appendix~\ref{sec:adacent_UB}); before two observations are available, set $UB(\bc)=1$. Completed centers are removed from the active set but retained in $\C$ for selection and the final recommendation.

\section{Theoretical Results}
\label{sec:theory}

In this section, we analyze the proposed algorithms. We first prove worst-case approximation bounds for \FullCent{} and show that they match the hardness of UVP up to constant factors, before analyzing the enhanced and adaptive variants. For the analysis that follows, recall that \(r_k^\star\) denotes the optimal covering radius for \(k\) centers on \(\mathcal{X}\), while \(\x^\star\) denotes an optimal configuration.

\subsection{Analysis of \FullCent{}}
\label{subsec:theory_fullcent}

\FullCent{} selects \(k = \lfloor B/T \rfloor\) configurations via greedy \(k\)-center and evaluates each to budget \(T\). The following theorem bounds \FullCent{}'s solution relative to the optimum \(A(\x^\star,T)\).

\begin{theorem}
\label{thm:FC-approx}
Let $k = \lfloor B/T \rfloor$. Under Assumptions~\ref{assump:assump1}--\ref{assump:assump2}, \FullCent{} achieves an approximation factor of $(1 - 2\epsilon r_k^\star)$ for the UVP problem.
\end{theorem}

\paragraph{Proof idea.}
\FullCent{} fully evaluates \(k=\lfloor B/T\rfloor\) centers. Greedy \(k\)-center has radius at most \(2r_k^\star\), so a selected center \(\bc\) is within \(2r_k^\star\) of \(\x^\star\). Assumption~\ref{assump:assump2} transfers value from \(\x^\star\) to \(\bc\), and the algorithm returns the best selected center. The full proof is in Appendix~\ref{subsec:approx}. This result shows that good coverage yields a near-optimal solution, while Theorem~\ref{thm:FC-hard} gives a matching information-theoretic lower bound up to constant factors.

\begin{theorem}
\label{thm:FC-hard}
Let $\epsilon > 0$, $\beta > 1$, $T \in \mathbb{N}$, and $B \ge T$. Set $k = \lfloor B/T \rfloor$ and let $r > 0$ denote the optimal clustering radius for $\lceil \beta k \rceil$ clusters. There exists an instance of the UVP problem satisfying Assumptions~\ref{assump:assump1}--\ref{assump:assump2} such that no algorithm can guarantee expected accuracy exceeding
\[
\left(\frac{1 - \epsilon r}{\beta - 1}\right) \cdot A(\x^\star, T).
\]
\end{theorem}

\paragraph{Proof idea.}
Hide the optimal configuration in one of \(\lceil \beta k\rceil\) well-separated clusters. Since no useful information appears before budget \(T\), any algorithm can test only \(k\) clusters. Under a uniformly random hidden cluster, Yao's principle bounds the expected value by the probability of hitting the hidden cluster, plus the value of the best suboptimal point inside it. The full proof is in Appendix~\ref{subsec:hard}, and Corollary~\ref{col:FC-hard} follows by taking \(\beta=2\).

\begin{corollary}
\label{col:FC-hard}

 Let $k=\lfloor B/T\rfloor$. There exists a UVP instance satisfying Assumptions~\ref{assump:assump1}--\ref{assump:assump2}, with $2k$ clusters and optimal $2k$-clustering radius $r>0$, such that no algorithm can guarantee an expected approximation factor exceeding $(1 - \epsilon r)$.
\end{corollary}

Thus, the $(1 - 2\epsilon r_k^\star)$ factor is tight up to constants and arbitrarily small additive slack.

\subsection{Analysis of Enhanced-FullCent}
\label{subsec:theory_enhanced}

\FullCent{} treats all regions equally, whereas \eFullCent{} uses observed full-fidelity values to adapt the selection geometry.

\paragraph{\textbf{Derivation of Enhanced Distance.}}
Under Assumption~\ref{assump:assump2}, for any center $\mathbf{c}$ and point $\mathbf{x}$ satisfying $1-\epsilon d(\mathbf{x},\mathbf{c})>0$,
\[
A(\mathbf{x},T)
\le
\frac{A(\mathbf{c},T)}
     {1-\epsilon d(\mathbf{x},\mathbf{c})}.
\]
Let $\mathbf{c}^{\star}_{\mathrm{obs}}$ be the best center observed so far and define
\[
\eta_{\mathbf{c}}
=
{A(\mathbf{c}^{\star}_{\mathrm{obs}},T)}/
     {A(\mathbf{c},T)}
\]
when $A(\mathbf{c},T)>0$. We define $\hat d$ by equating the optimistic bound induced by $\mathbf{c}$ with that induced by the incumbent:
\[
\frac{A(\mathbf{c},T)}
     {1-\epsilon d(\mathbf{x},\mathbf{c})}
=
\frac{A(\mathbf{c}^{\star}_{\mathrm{obs}},T)}
     {1-\epsilon \hat d(\mathbf{x},\mathbf{c})}.
\]
Solving yields
\[
\hat d(\mathbf{x},\mathbf{c})
=
\eta_{\mathbf{c}}d(\mathbf{x},\mathbf{c})
-
{(\eta_{\mathbf{c}}-1)}/{\epsilon}.
\]
The algorithm uses
\[
\tilde d(\mathbf{x},\mathbf{c})
=
\min\!\left\{
d(\mathbf{x},\mathbf{c}),
\hat d(\mathbf{x},\mathbf{c})
\right\},
\]
with $\tilde d(\mathbf{x},\mathbf{c})=d(\mathbf{x},\mathbf{c})$ when $A(\mathbf{c},T)=0$.
The score is asymmetric and observation-dependent rather than a fixed metric. Larger $\eta_{\mathbf{c}}$ reduces the score near weak centers, discouraging redundant sampling in low-potential regions while preserving the worst-case guarantee.

\begin{theorem}
\label{thm:EFC-approx}
Let $k = \lfloor B/T \rfloor$. Under Assumptions~\ref{assump:assump1}--\ref{assump:assump2}, 
\eFullCent{} achieves an approximation factor of $(1 - 2\epsilon r_k^\star)$ for the UVP problem.
\end{theorem}

\paragraph{Proof idea.}
The enhanced distance is a value certificate: small enhanced distance from \(\x^\star\) to a selected center implies a large observed incumbent. If the selected centers occupy distinct optimal \(k\)-center clusters, the usual covering argument applies. Otherwise, at the first repeated cluster, the max--min rule implies that every point, including \(\x^\star\), has enhanced distance at most \(2r_k^\star\) from the selected set. The certificate then gives the same approximation factor as \FullCent{}. The full proof is in Appendix~\ref{subsec:EFC-approx}. Thus, enhanced distances bias selection toward promising regions while retaining the same worst-case guarantee as \FullCent{}.

\subsection{Analysis of \AdaCent{}}
\label{subsec:theory_adacent}

We now analyze the adaptive setting, where \AdaCent{} observes $A(\mathbf{x},b)$ incrementally. To justify pruning, we assume learning curves exhibit diminishing marginal returns.

\begin{assumption}[Diminishing Returns]
\label{assump:concavity}
For all configurations $\mathbf{x} \in \mathcal{X}$ and intermediate budgets $b \in \{2, \ldots, T-1\}$, the marginal gain is non-increasing:
\[
A(\mathbf{x}, b+1) - A(\mathbf{x}, b) \le A(\mathbf{x}, b) - A(\mathbf{x}, b-1).
\]
\end{assumption}

 This assumption captures the saturation trend needed for safe pruning. Under Assumption~\ref{assump:concavity}, linear extrapolation gives a valid upper bound $UB(\mathbf{x})$ on $A(\mathbf{x},T)$ (Appendix~\ref{sec:adacent_UB}). Pruning below the current incumbent therefore cannot remove a configuration unless an observed value already matches its certified terminal upper bound. We empirically assess this assumption  in
Appendix~\ref{app:diminishing-returns-validation}.

\begin{theorem}
\label{thm:AC-approx}
Let $k=\min\{|\X|,\lfloor B/T\rfloor\}$. Under Assumptions~\ref{assump:assump1}--\ref{assump:assump2} and \ref{assump:concavity}, and with the completion-budget restriction of Algorithm~\ref{alg:adacent} (the first $k$ greedy centers are protected),
\AdaCent{} achieves an approximation factor of
$(1-2\epsilon r_k^\star)$ for the UVP problem.
\end{theorem}

\paragraph{Proof idea.}
\AdaCent{} selects at least \(k\) greedy centers before the budget is exhausted, and the covering guarantee of greedy \(k\)-center ensures that one of the first \(k\) centers, say \(\bc\), lies within \(2r_k^\star\) of \(\x^\star\), so \(A(\bc,T)\ge(1-2\epsilon r_k^\star)A(\x^\star,T)\). By the completion-budget restriction, \(\bc\) is protected: it is either fully evaluated to fidelity \(T\), in which case this value is observed directly, or it is pruned, in which case the concavity-based upper bound certifies that the current incumbent already exceeds \(A(\bc,T)\). The full proof is in Appendix~\ref{subsec:approx-adacent}. Thus, the adaptive strategy retains the safety guarantee of the exhaustive baseline, and the following result establishes a matching barrier for concave functions.

\begin{theorem}
\label{thm:AC-hard}
Let $\epsilon > 0$, $\beta > 1$, $\theta \in (0,1)$, $T \in \mathbb{N}$, and $B \ge T$. Set $k = \lfloor B/T \rfloor + 1$. Let $r > 0$ denote the optimal clustering radius for $\lceil \beta k \rceil$ clusters. There exists an instance of the UVP problem satisfying Assumptions~\ref{assump:assump1}--\ref{assump:assump2} and  \ref{assump:concavity} such that no algorithm can guarantee expected accuracy exceeding
\[
\left(\theta+\frac{1}{\theta(\beta-1)}\right)(1-\epsilon r)
\cdot A(\x^\star,T).
\]
\end{theorem}

\paragraph{Proof idea.}
Use the same hidden-cluster construction, but choose learning curves that are concave and indistinguishable until about \(\theta T\) budget. Before that point, an algorithm can obtain only the baseline value; to obtain useful information, it must spend roughly \(\theta T\) per probe. Thus, it can make only about \(B/(\theta T)\) informative probes, and Yao's principle yields the stated lower bound. The full proof is in Appendix~\ref{subsec:hard-adacent}, and Corollary~\ref{col:AC-hard} follows directly.

\begin{corollary}
\label{col:AC-hard}
Let $k=\lfloor B/T\rfloor+1$. There exists a UVP instance satisfying Assumptions~\ref{assump:assump1}--\ref{assump:assump2} and Assumption~\ref{assump:concavity}, with $5k$ clusters and optimal $5k$-clustering radius $r>0$, such that no algorithm can guarantee an expected approximation factor exceeding $(1 - \epsilon r)$.

\end{corollary}

Thus, the approximation factor achieved by \AdaCent{} is tight up to constants and arbitrarily small additive slack for concave learning curves in this setting. For the combined method \eAdaCent{}, whose dynamic warping and pruning introduce additional dependencies, we provide a conditional guarantee: under early-to-terminal calibration and the completion-budget restriction, its attained value is at least $(1-2\epsilon r_k^\star)\,\mathrm{OPT}/\kappa$, where $\kappa\ge1$ quantifies the calibration gap between the early selection value $A(\x,b_0)$ and the terminal value $f(\x)=A(\x,T)$. The precise assumptions and proof appear in Appendix~\ref{sec:eadacent-cond}. Section~\ref{sec:experiments} shows that the combined method performs best in practice.

\begin{figure}[t]
  \centering
  \begin{subfigure}[t]{0.32\textwidth}
    \includegraphics[width=\textwidth,trim={3.360bp 3.360bp 8.880bp 3.360bp},clip]{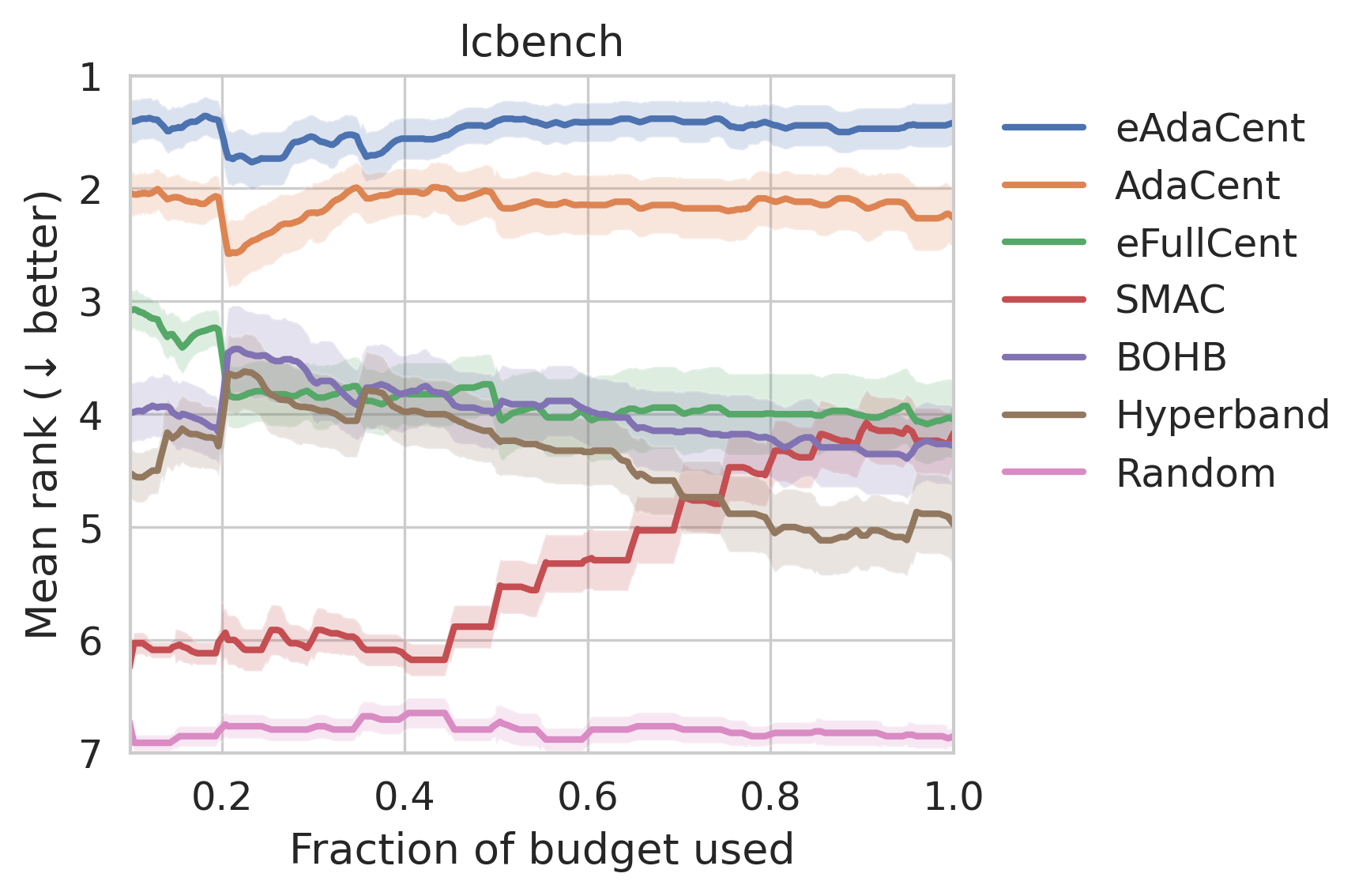}
    \label{fig:meanrank:lcbench}
  \end{subfigure}
  \hfill
  \begin{subfigure}[t]{0.32\textwidth}
    \includegraphics[width=\textwidth,trim={3.360bp 3.360bp 8.880bp 4.320bp},clip]{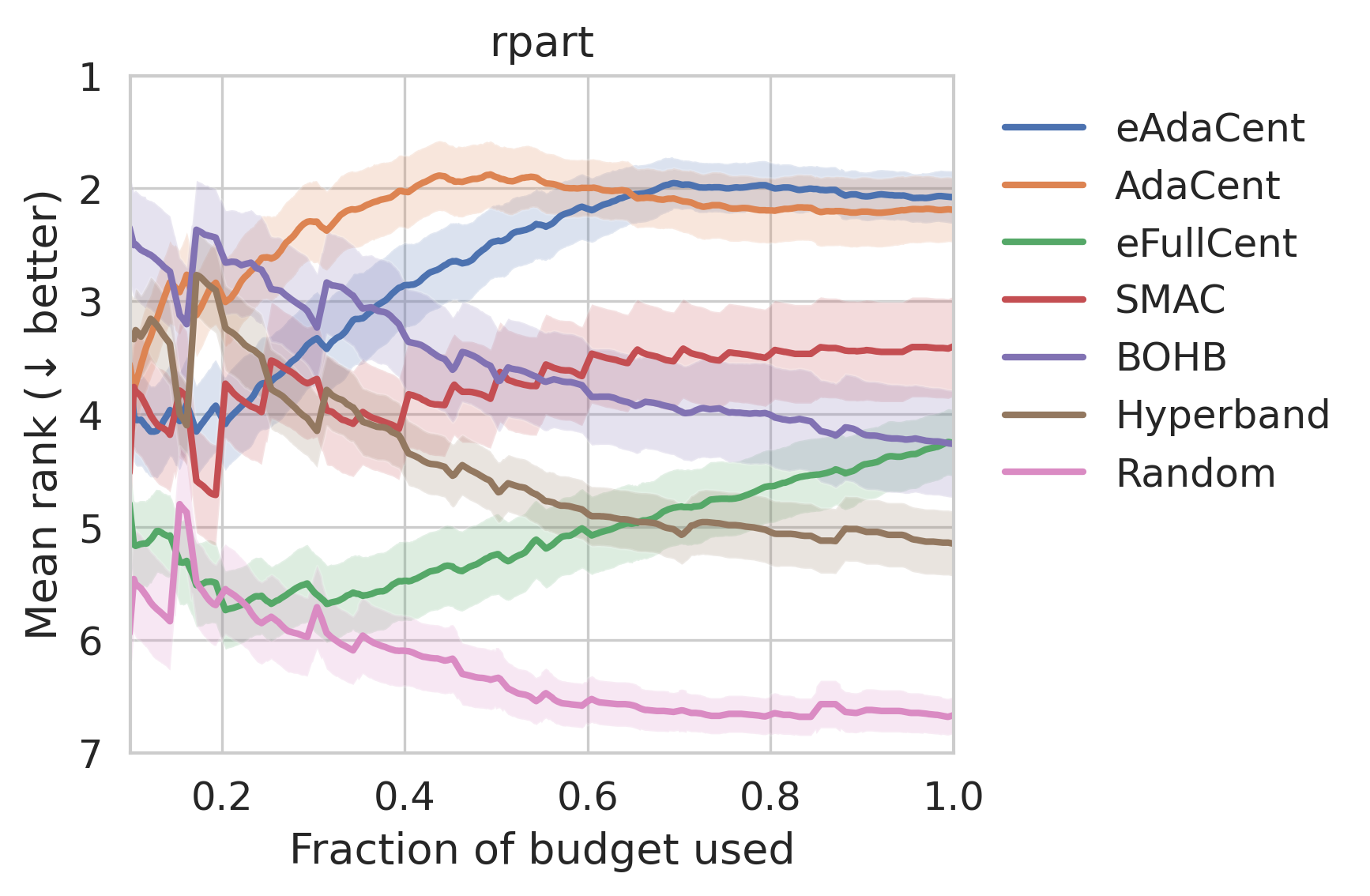}
    \label{fig:meanrank:rpart}
  \end{subfigure}
  \hfill
  \begin{subfigure}[t]{0.32\textwidth}
    \includegraphics[width=\textwidth,trim={3.360bp 3.360bp 8.880bp 3.360bp},clip]{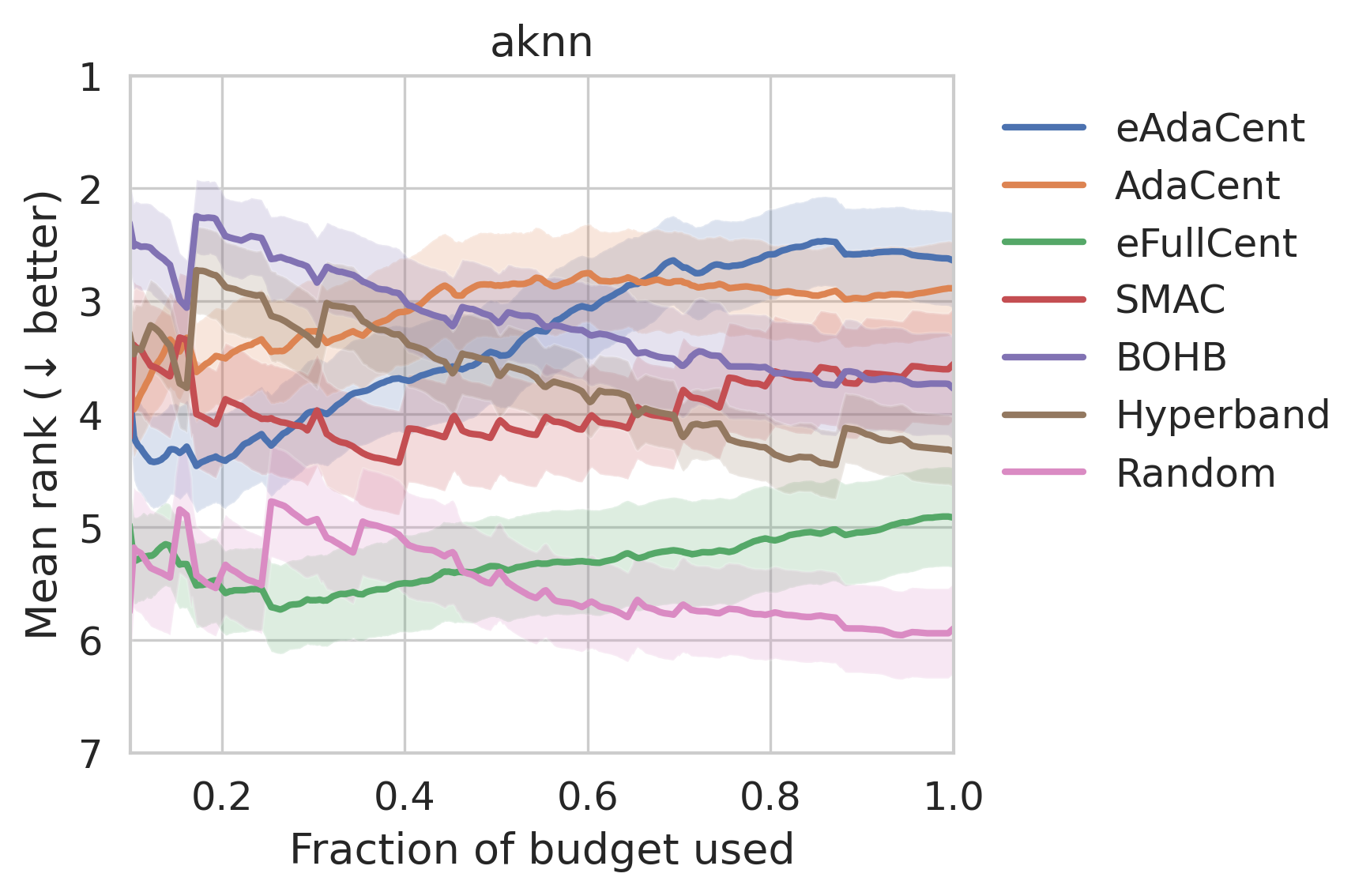}
    \label{fig:meanrank:aknn}
  \end{subfigure}
  \caption{Mean rank of each algorithm over all datasets and 30 runs for \texttt{lcbench} (35 datasets), \texttt{rbv2\_rpart} (107), and \texttt{rbv2\_aknn} (118). Lower is better. The $x$-axis shows the fraction of the total budget, starting at 0.1 to highlight early-budget differences.}
  \label{fig:meanrank:all}
\end{figure}

\section{Experiments}
\label{sec:experiments}


In this section, we evaluate \AdaCent{} and \eAdaCent{} on three multi-fidelity HPO scenarios from YAHPO Gym~\cite{pfisterer2022yahpo}: \texttt{lcbench}~\cite{zimmer2021autopytorch,zimmer2020lcbenchdata}, \texttt{rbv2\_rpart}, and \texttt{rbv2\_aknn}~\cite{binder2020rbv2}. Together, these scenarios span more than \textbf{250} experimental settings and cover both deep learning (\texttt{lcbench}) and classical ML (\texttt{rbv2}) pipelines.

\paragraph{\textbf{Evaluation Protocol and Benchmark.}}
Each method interacts with the benchmark through the same discrete fidelity levels $b\in[T]$ provided by YAHPO Gym (number of epochs for \texttt{lcbench} and training-set fraction for \texttt{rbv2}) and all methods receive an identical total budget. In all experiments, we use the Euclidean distance $d(\x,\x')=\lVert \x-\x' \rVert_2$. We evaluate performance by tracking the anytime best validation accuracy trajectory and the final incumbent accuracy. Results are averaged over 30 independent runs per task with different random seeds; we report mean $\pm$ standard deviation and summarize robustness across tasks via mean ranks (lower is better). Hyperparameter domains and dataset coverage follow the official YAHPO Gym specifications~\cite{pfisterer2022yahpo}. For completeness, the hyperparameter search spaces used for all three scenarios are listed in Appendix~\ref{sec:figs}. Finally, we present empirical evidence supporting the validity of Assumption~\ref{assump:assump2} in these practical search spaces.

\paragraph{\textbf{Baselines.}}
We compare \AdaCent{} and \eAdaCent{} against four competitive HPO baselines spanning complementary paradigms. Hyperband~\citep{li2017hyperband} is a canonical bandit-based multi-fidelity method based on successive halving, and BOHB~\citep{falkner2018bohb} combines this allocation scheme with Bayesian optimization. We also include SMAC~\citep{hutter2011smac}, a sequential model-based optimizer using random-forest surrogates for mixed and conditional spaces, and uniform random search as a nonparametric sanity-check baseline. These methods are standard points of comparison in practical HPO studies and best-practice guidelines~\cite{bischl2023hpo,falkner2018bohb,li2017hyperband}.
Additionally, in order to assess the contribution of adaptation, we include \eFullCent{}, our strongest non-adaptive baseline. Because \eFullCent{} selects centers using terminal values $A(\bc,T)$ while \eAdaCent{} selects using early values $A(\bc,b_0)$, their comparison is an end-to-end assessment of the full adaptive pipeline (early-value-informed selection together with pruning) rather than an isolation of a single component.

All methods were run under the same budget of 20 full-fidelity evaluations (\(B = 20\,T\)). We ran SMAC with 24 trials per task, Hyperband and BOHB with 6 iterations and \(\eta = 3\), and \eFullCent{} with \(k=\lfloor B/T\rfloor=20\) fully evaluated centers. We ran \AdaCent{} with \(m = 25\), and \eAdaCent{} with \(m = 25\) and exploration factor \(\delta = 0.1\). We set the smoothness parameter $\epsilon$ per task as the $95$th percentile of the pairwise smoothness ratios $\epsilon_{ij}$ defined in Section~\ref{sec:val_assump}, computed over $10{,}000$ Sobol configurations evaluated at fidelity $\lfloor0.1T\rfloor$. These estimation queries are charged neither to the global budget $B$ nor to the baselines, and the resulting $\epsilon$ is a practical selection rule rather than a certificate of the global smoothness constant required by Assumption~\ref{assump:assump2}; the sensitivity analysis in Appendix~\ref{app:sensitivity} shows that the results are stable under this choice. We used standard configurations for all methods and ran them using YAHPO Gym's official public implementation for reproducibility. 

Additional synthetic comparisons are reported in Appendix~\ref{app:synthetic-baselines}, where we evaluate \eFullCent{} against Random Search, SMAC, and BOHB on three challenging two-dimensional test functions under a common evaluation budget. These experiments complement the benchmark results by illustrating how geometry-aware selection behaves across controlled optimization landscapes. Appendix~\ref{app:mfbo-comparison} further compares \AdaCent{} and \eAdaCent{} with representative Multi-Fidelity Bayesian Optimization (MFBO) methods, including rMFBO, multi-fidelity Max-value Entropy Search (MES), and Knowledge Gradient (KG).

\begin{figure}[t]
  \centering

  \begin{subfigure}[t]{0.325\textwidth}
    \includegraphics[width=\textwidth,trim={21.600bp 12.000bp 10.800bp 10.320bp},clip]{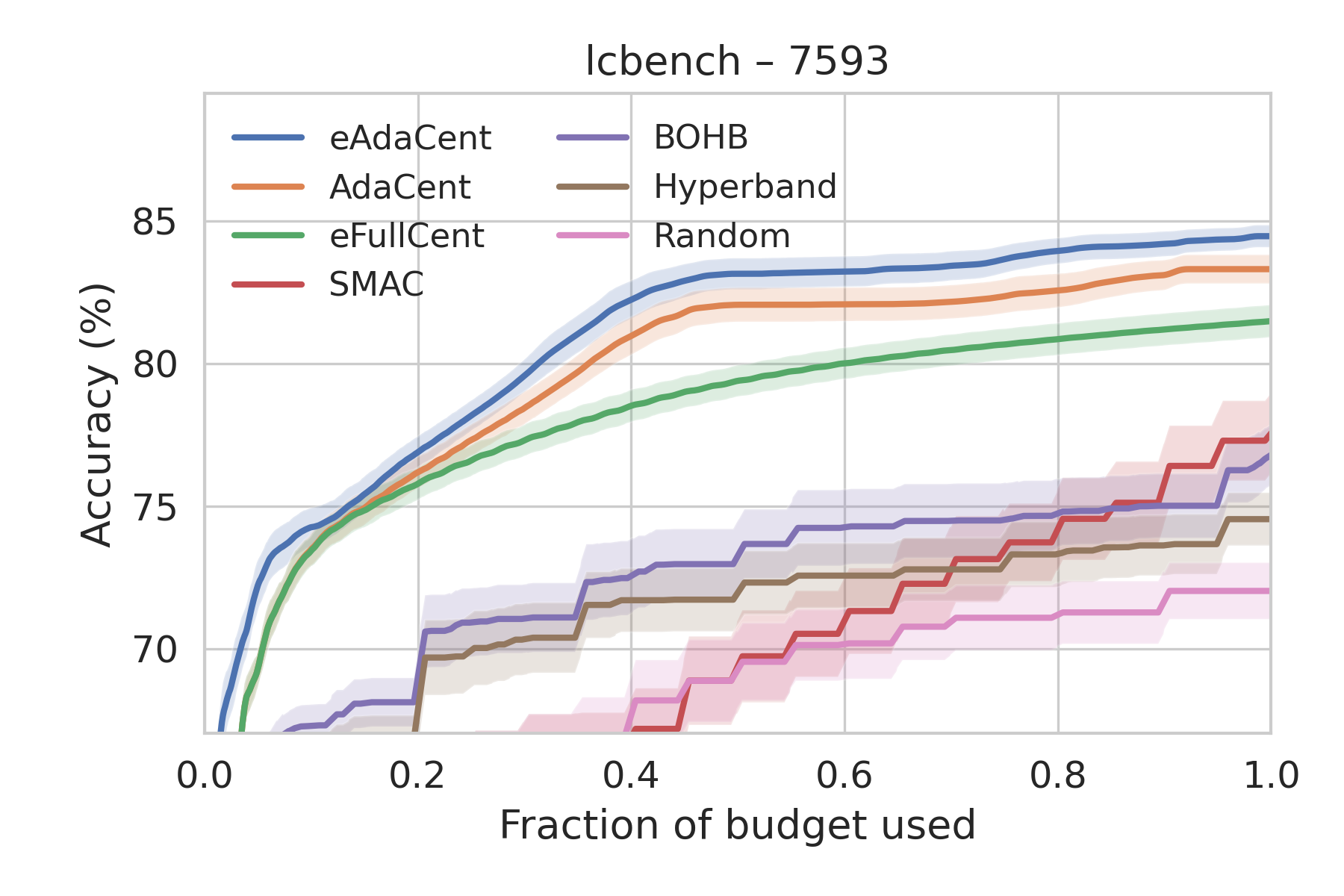}
  \end{subfigure}\hfill
  \begin{subfigure}[t]{0.325\textwidth}
    \includegraphics[width=\textwidth,trim={0.000bp 12.000bp 10.800bp 10.320bp},clip]{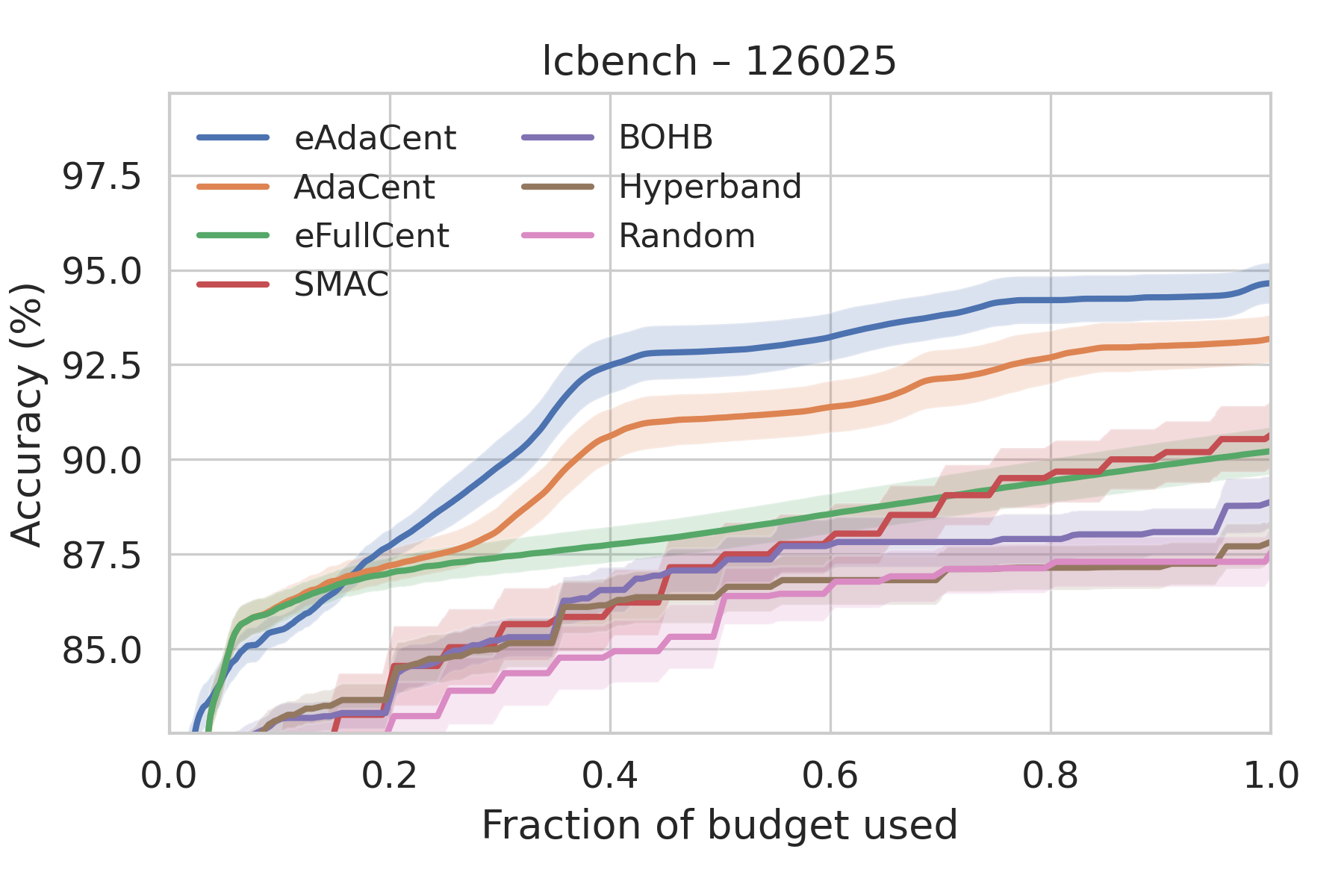}
  \end{subfigure}\hfill
  \begin{subfigure}[t]{0.325\textwidth}
    \includegraphics[width=\textwidth,trim={10.320bp 12.000bp 10.800bp 10.320bp},clip]{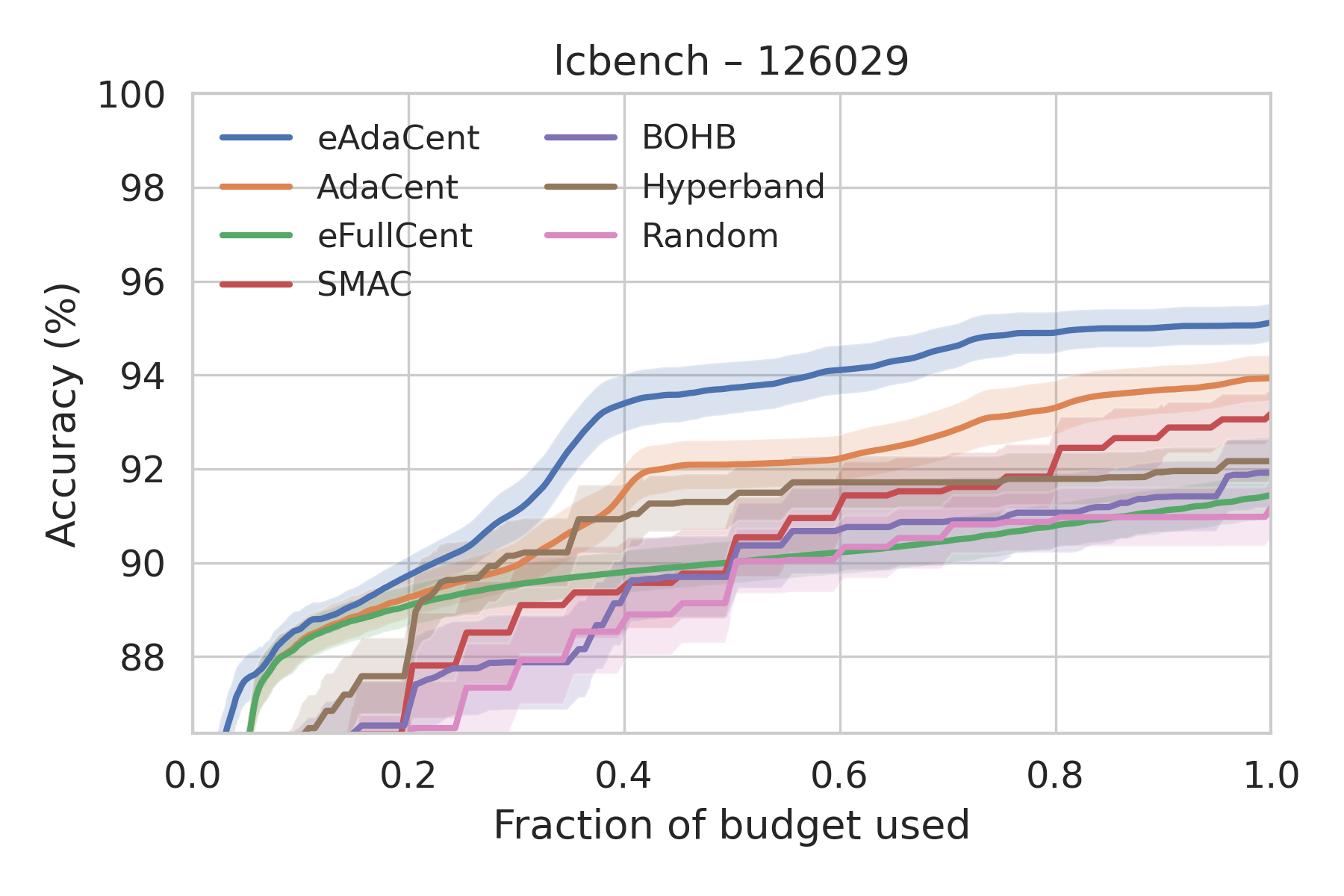}
  \end{subfigure}
\par
  \begin{subfigure}[t]{0.325\textwidth}
    \includegraphics[width=\textwidth,trim={10.320bp 12.000bp 10.800bp 10.320bp},clip]{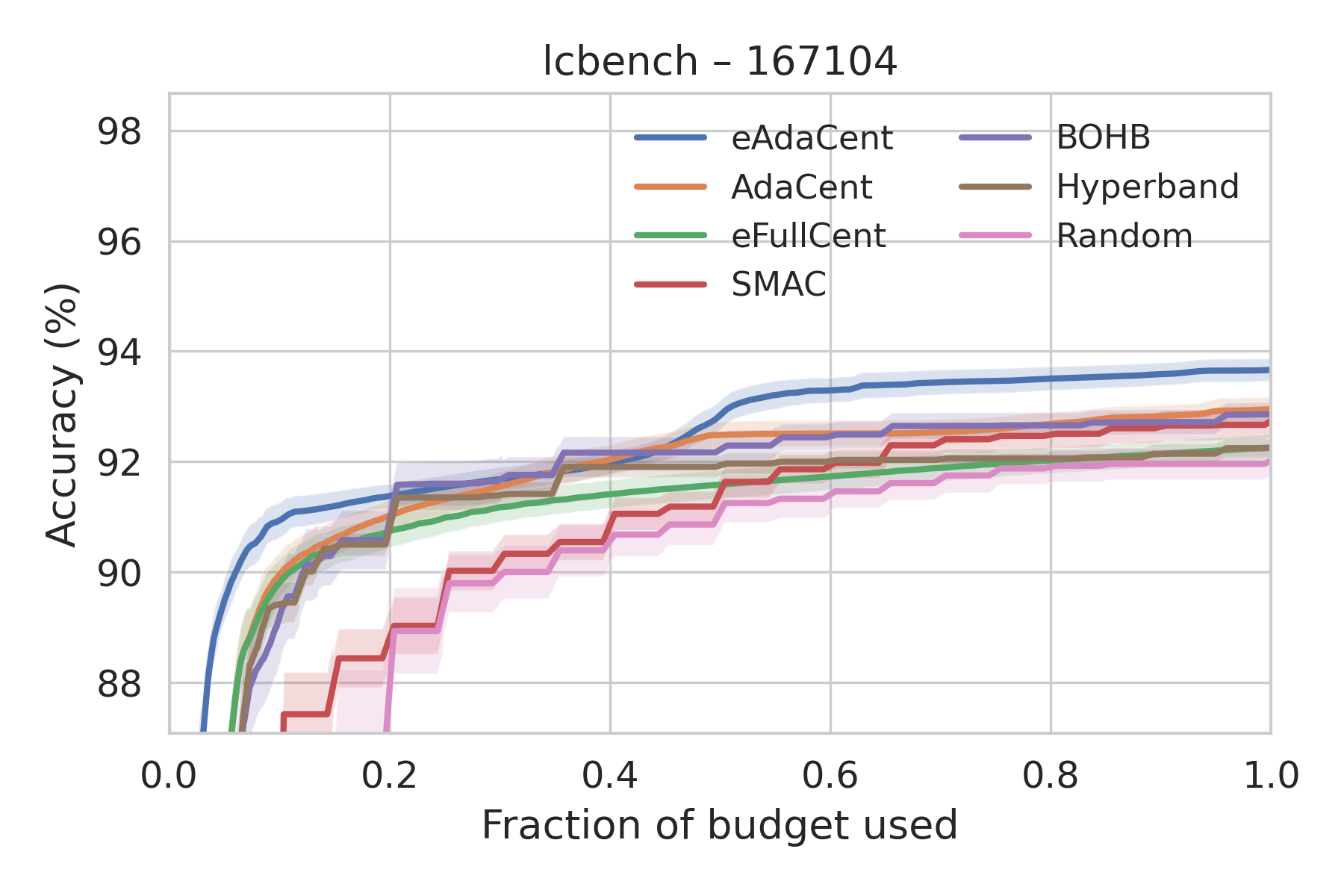}
  \end{subfigure}\hfill
  \begin{subfigure}[t]{0.325\textwidth}
    \includegraphics[width=\textwidth,trim={10.320bp 12.000bp 10.800bp 10.320bp},clip]{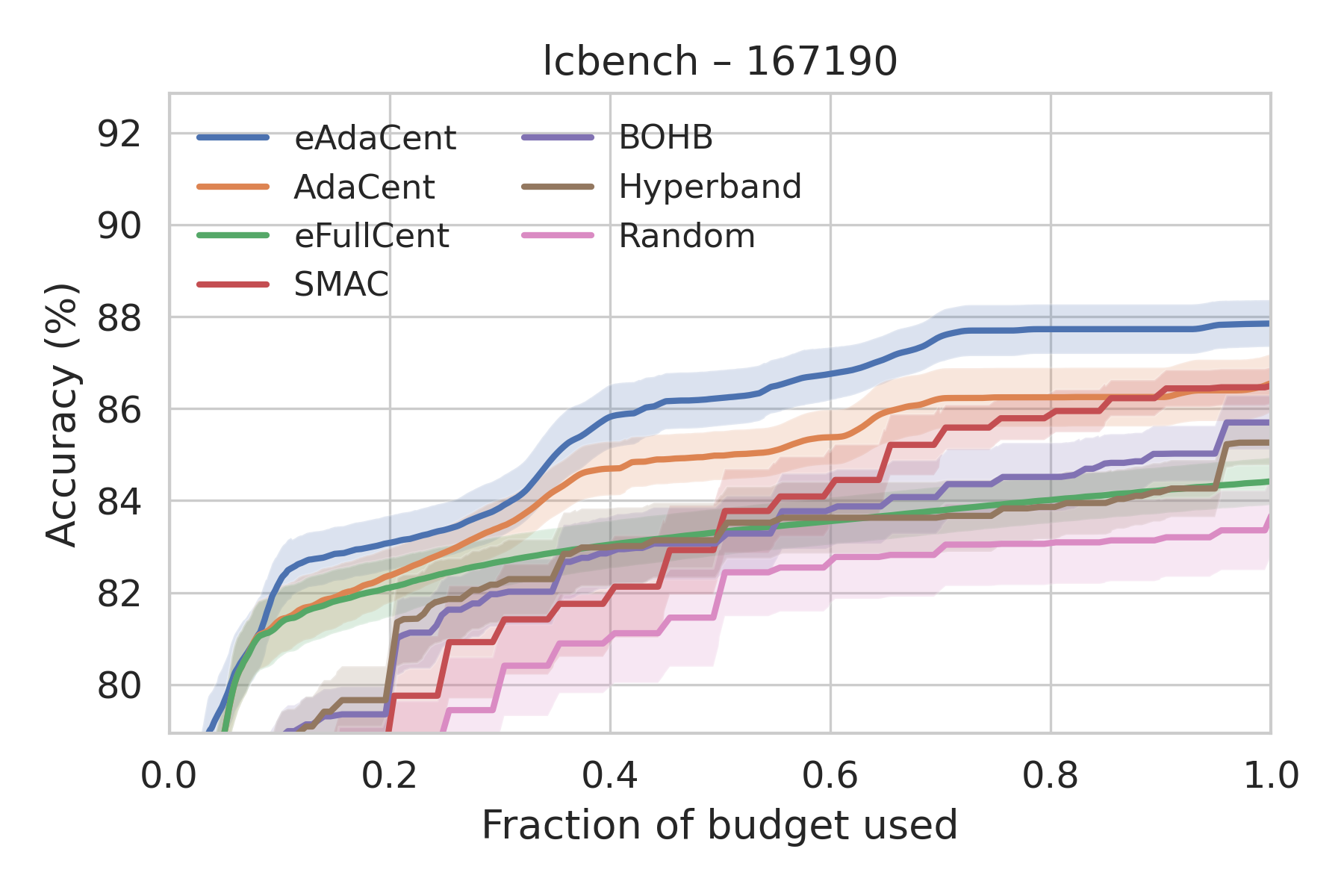}
  \end{subfigure}\hfill
  \begin{subfigure}[t]{0.325\textwidth}
    \includegraphics[width=\textwidth,trim={10.320bp 12.000bp 10.800bp 10.320bp},clip]{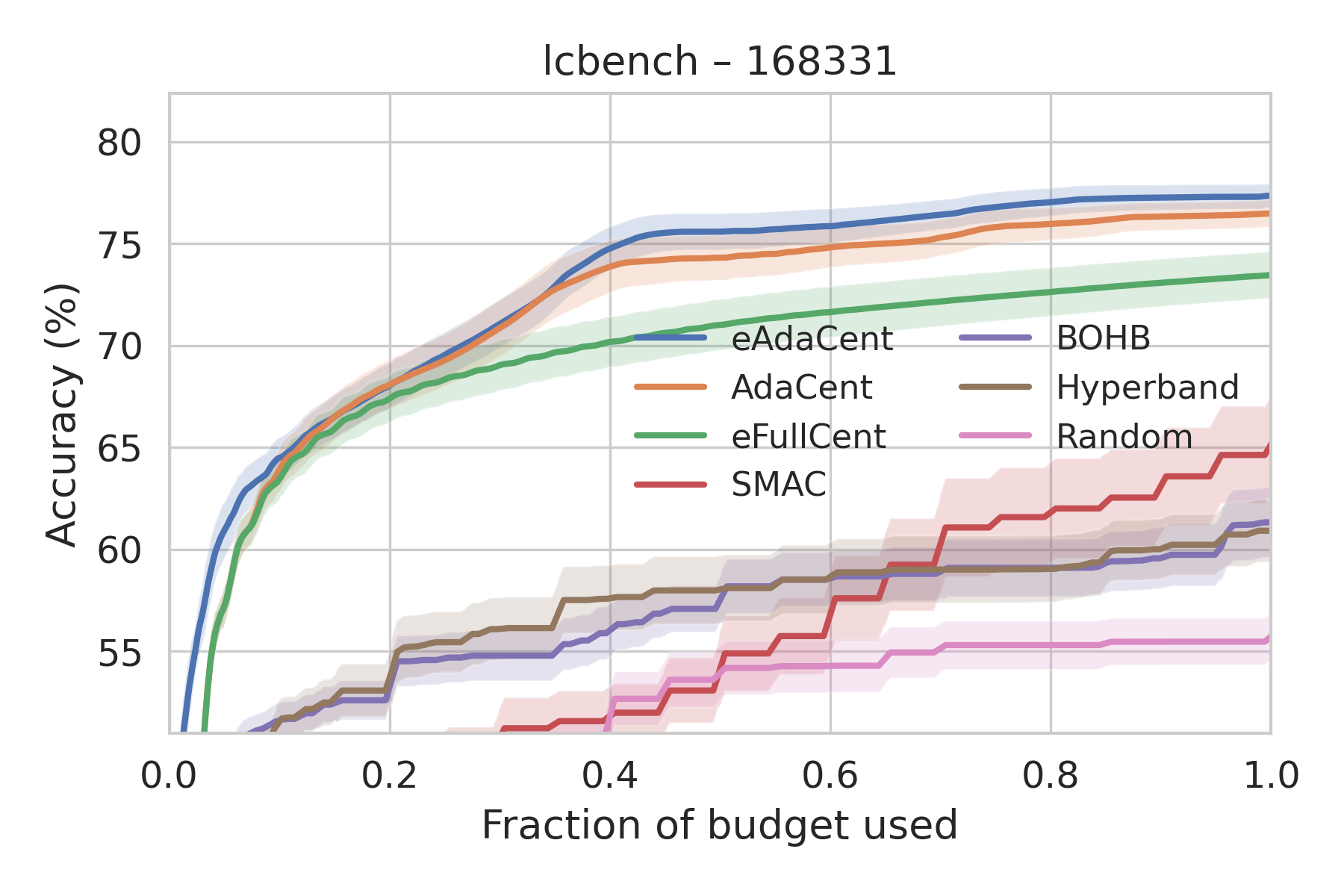}
  \end{subfigure}
  
  \caption{Validation accuracy curves for six randomly selected tasks from \texttt{lcbench}. \AdaCent{} and \eAdaCent{} consistently converge earlier and achieve higher accuracy. Curves for all 35 tasks are provided in Appendix~\ref{sec:figs}.}
  \label{fig:acc:lcbench}
\end{figure}

\subsection{Experiments on \texttt{lcbench}}
\label{subsec:experiments-lcbench}
We begin with \texttt{lcbench}, a multi-fidelity benchmark consisting of 35 OpenML classification tasks and a 7-dimensional hyperparameter search space, where fidelity is given by the number of training epochs (\texttt{epoch})~\citep{zimmer2021autopytorch,zimmer2020lcbenchdata,pfisterer2022yahpo}. Following the benchmark specification, we normalize each hyperparameter to $[0,1]$ using either linear or logarithmic scaling, depending on the parameter type, and we restrict evaluations to the standardized discrete epoch grid $b\in[T]$ provided by YAHPO Gym. This ensures that every method is compared under the same fidelity schedule and noise model. For each task, we report anytime validation accuracy as a function of consumed budget together with the final incumbent at the end of the run. Beyond our study, \texttt{lcbench} is also used as a standard empirical testbed for multi-fidelity HPO in recent work~\citep{liu2024uqguided,brandt2024bestarm,fehring2025dynabo}. Figure~\ref{fig:acc:lcbench} shows six representative learning curves, while Appendix~\ref{sec:figs} contains the complete set of curves for all tasks and the corresponding \texttt{lcbench} hyperparameter domain.

\subsection{Experiments on \texttt{rbv2}}
\label{subsec:experiments-rbv2}
Our second evaluation examines two \texttt{rbv2} scenarios from YAHPO Gym that model classical machine learning pipelines via surrogate benchmarks with fidelity determined by the fraction of training data (\texttt{trainsize})~\cite{pfisterer2022yahpo}. Specifically, \texttt{rbv2\_rpart} simulates decision tree induction across 117 classification datasets, while \texttt{rbv2\_aknn} tunes approximate $k$-nearest neighbor classifiers on 118 classification datasets; both scenarios expose 4-dimensional hyperparameter spaces. We evaluate \AdaCent{} and \eAdaCent{} by querying the surrogate models at discrete training fractions $b \in [T]$, corresponding to the default \texttt{trainsize} grid. This shared protocol emulates realistic multi-fidelity evaluations under a fixed budget while avoiding retraining models from scratch. The \texttt{rbv2} family (including \texttt{rbv2\_rpart} and \texttt{rbv2\_aknn}) is also used as a benchmarking substrate in recent YAHPO-based HPO studies~\citep{fehring2025dynabo,brandt2024bestarm}. Figure~\ref{fig:acc:rbv2} shows accuracy curves for three randomly selected datasets, and Appendix~\ref{sec:figs} provides the full collection of curves for all datasets.

\begin{figure}[t]
  \centering

  \begin{subfigure}[t]{0.325\textwidth}
    \includegraphics[width=\textwidth,trim={10.320bp 12.000bp 10.800bp 10.560bp},clip]{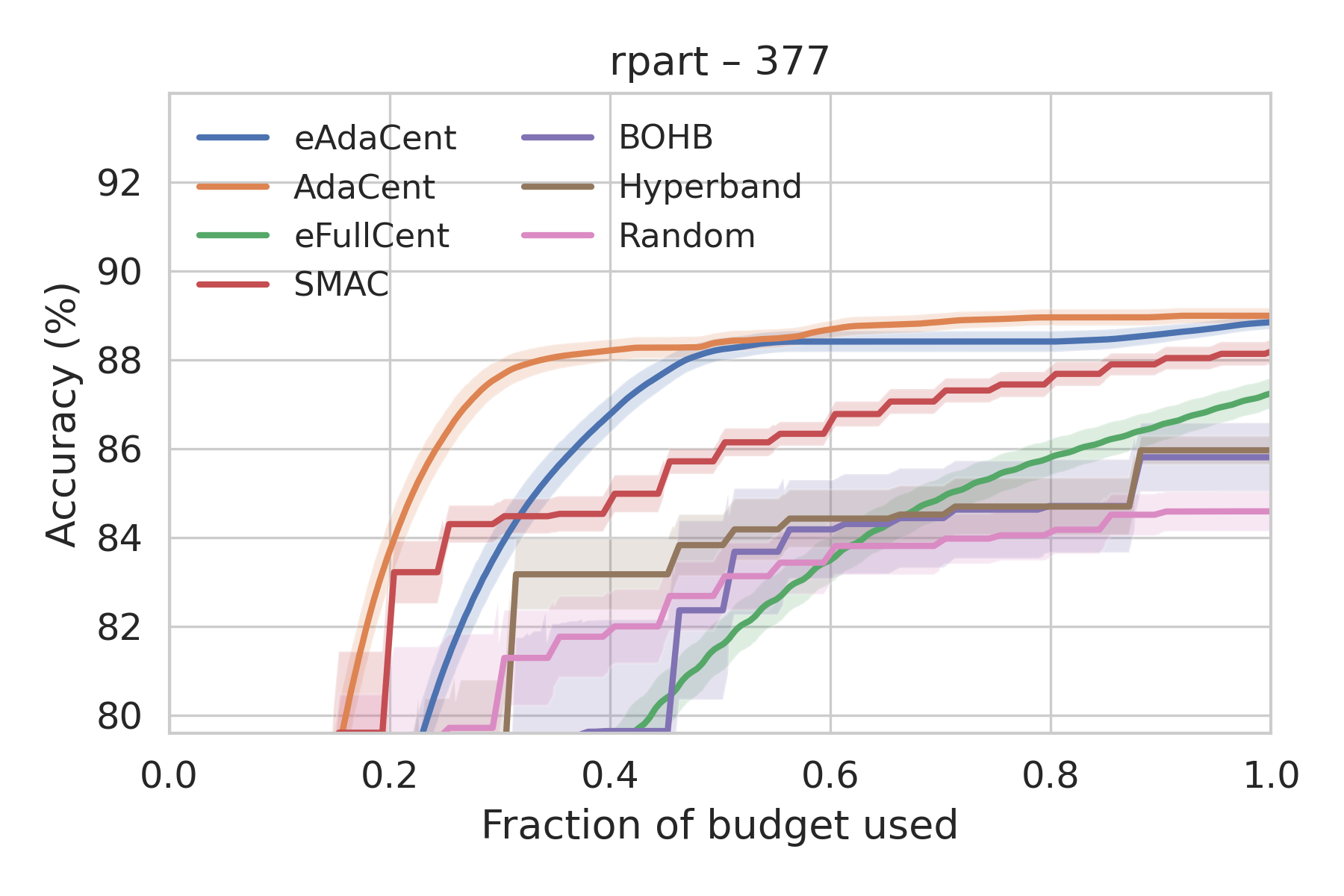}%
    \label{fig:acc:rp1}
  \end{subfigure}\hfill
  \begin{subfigure}[t]{0.325\textwidth}
    \includegraphics[width=\textwidth,trim={10.320bp 12.000bp 10.800bp 10.560bp},clip]{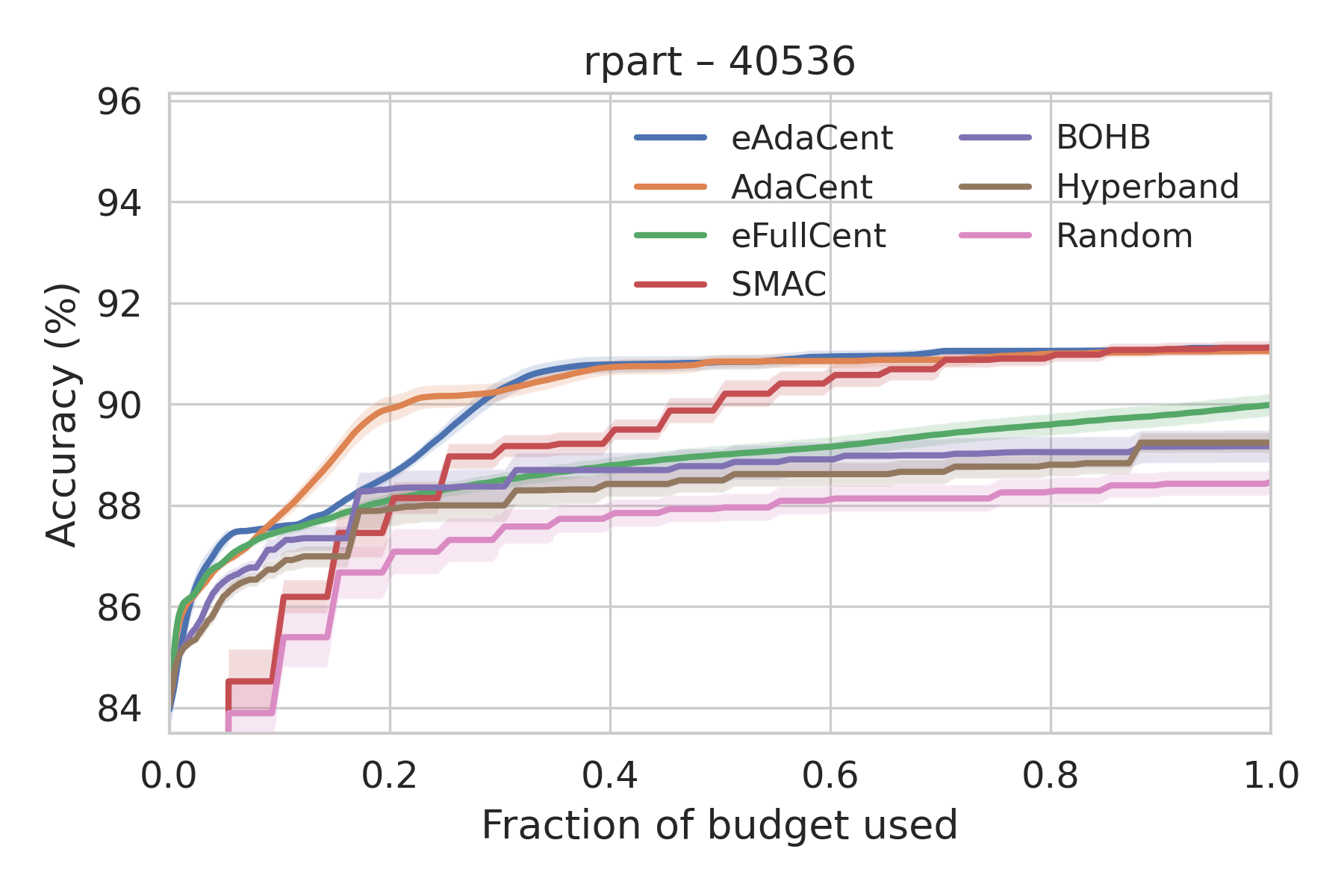}%
    \label{fig:acc:rp2}
  \end{subfigure}\hfill
  \begin{subfigure}[t]{0.325\textwidth}
    \includegraphics[width=\textwidth,trim={10.320bp 12.000bp 10.800bp 10.560bp},clip]{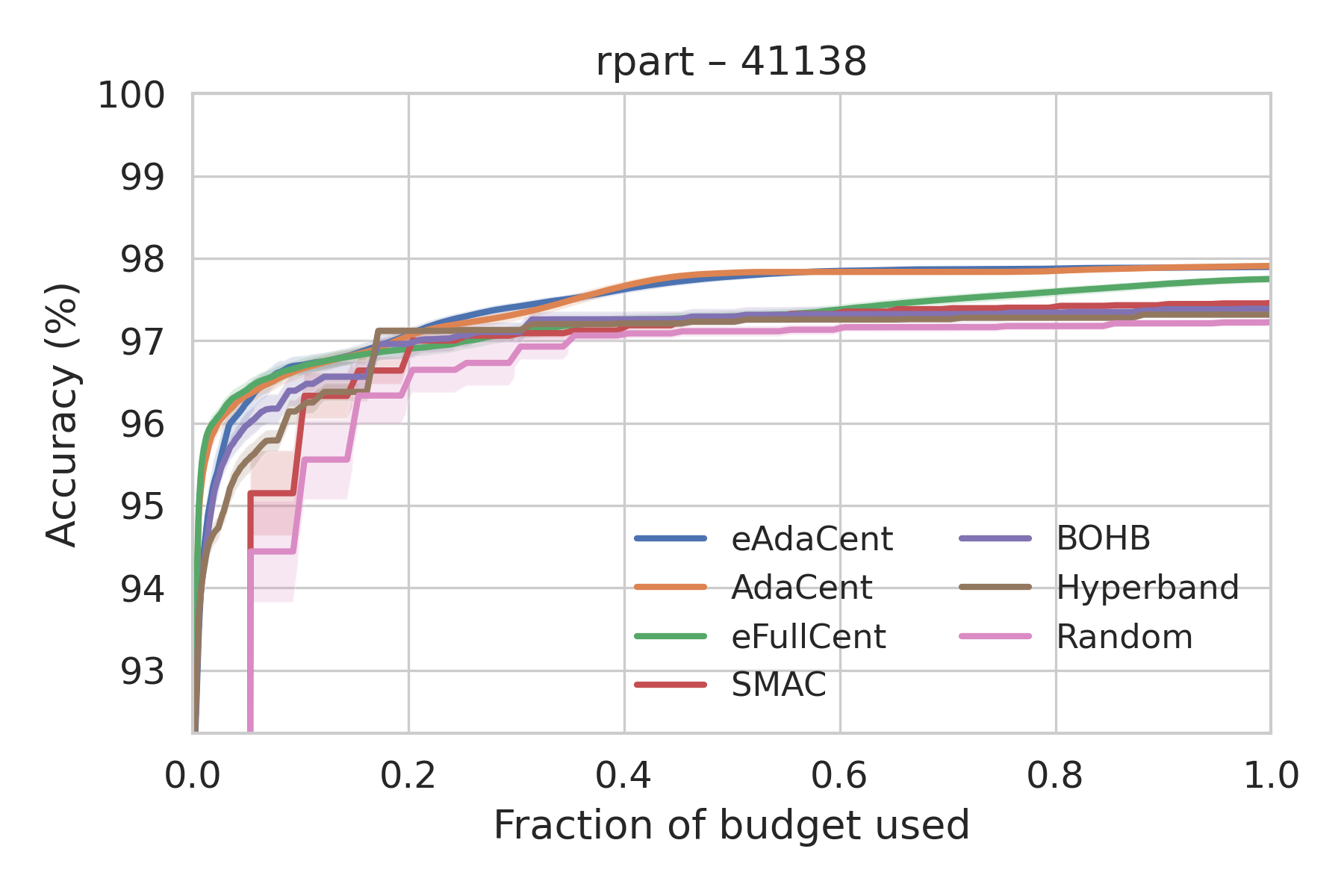}%
    \label{fig:acc:rp3}
  \end{subfigure}
  
  \begin{subfigure}[t]{0.325\textwidth}
    \includegraphics[width=\textwidth,trim={10.320bp 12.000bp 10.800bp 10.320bp},clip]{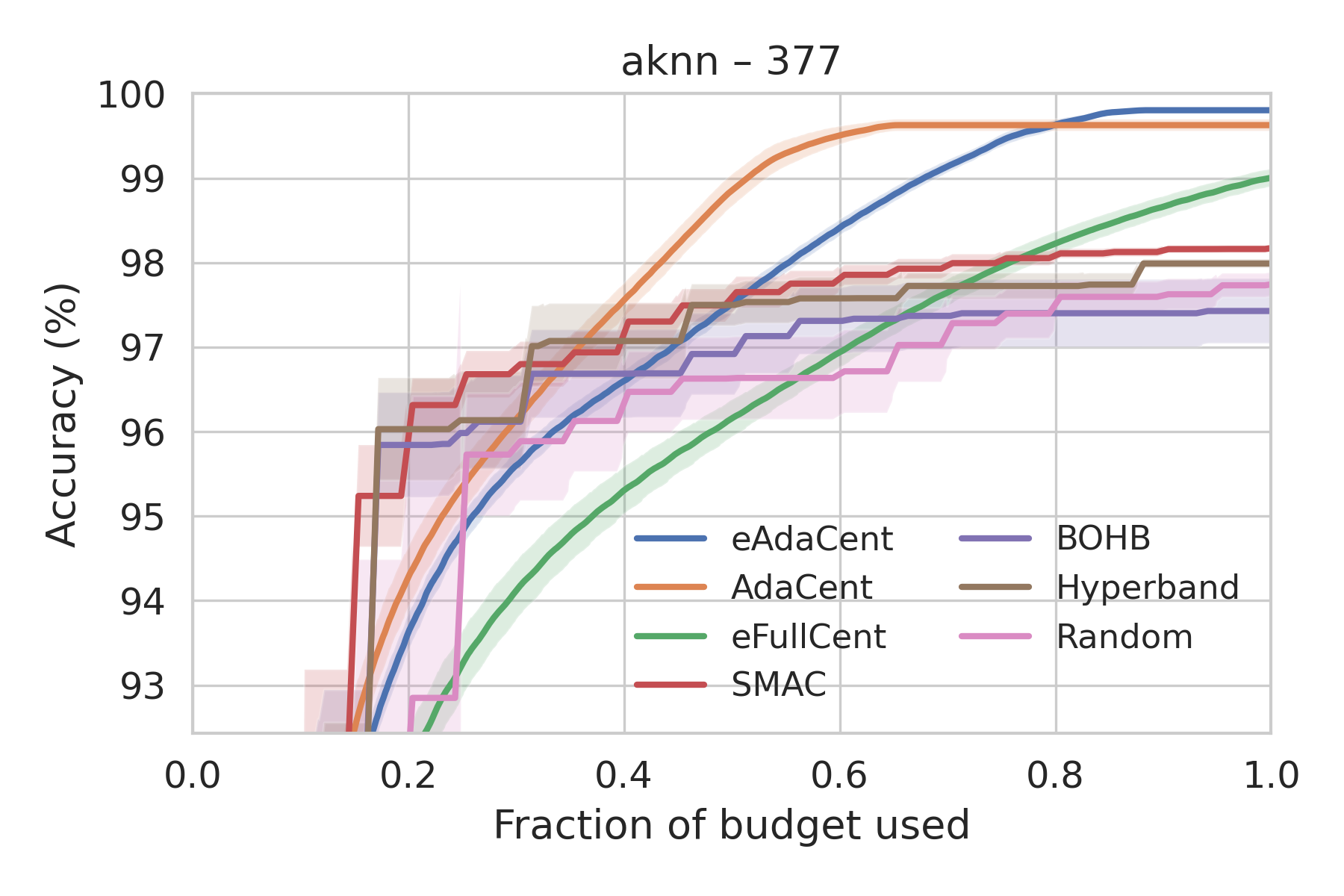}%
    \label{fig:acc:ak1}
  \end{subfigure}\hfill
  \begin{subfigure}[t]{0.325\textwidth}
    \includegraphics[width=\textwidth,trim={10.320bp 12.000bp 10.800bp 10.320bp},clip]{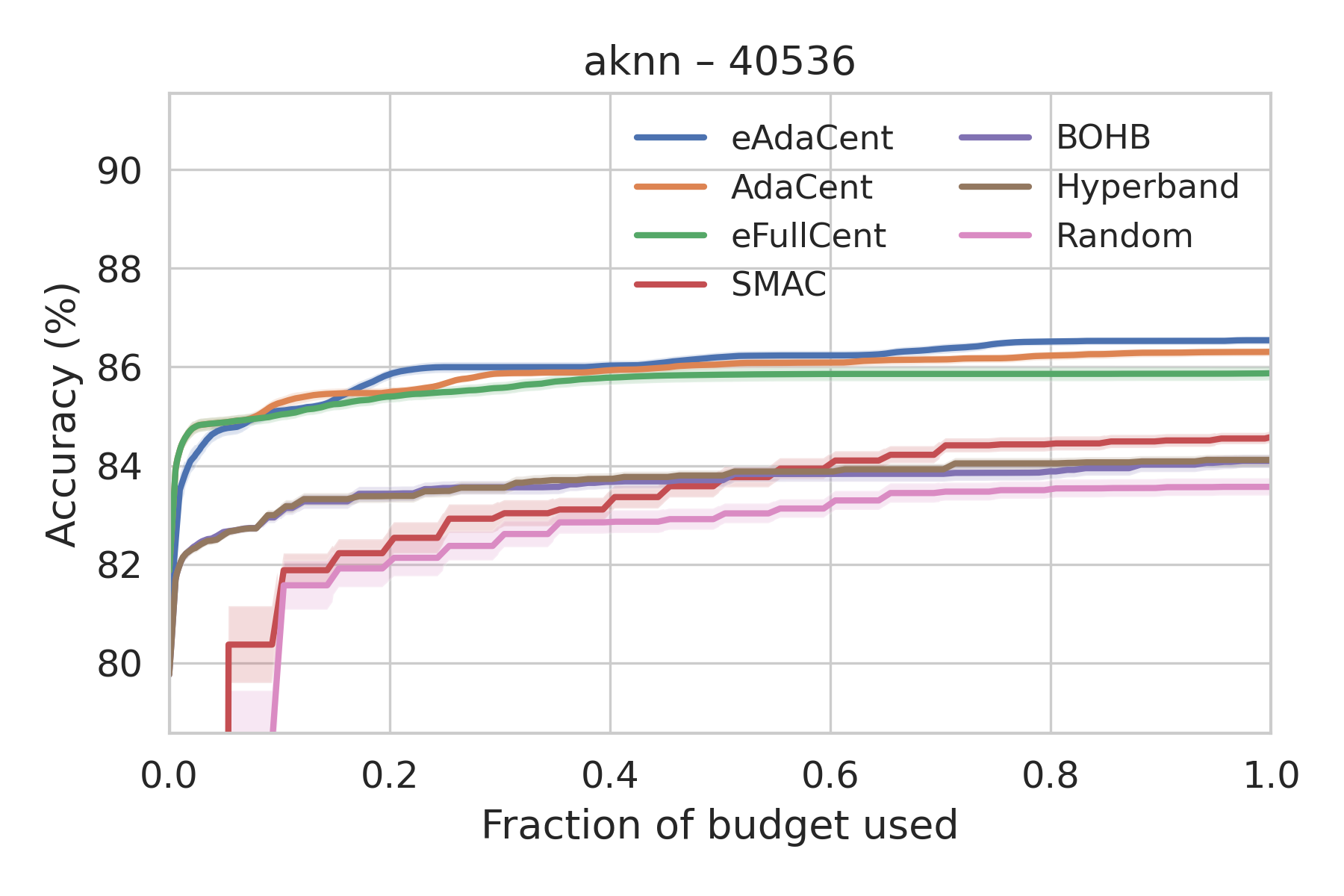}
    \label{fig:acc:ak2}%
  \end{subfigure}\hfill
  \begin{subfigure}[t]{0.325\textwidth}
    \includegraphics[width=\textwidth,trim={10.320bp 12.000bp 10.800bp 10.320bp},clip]{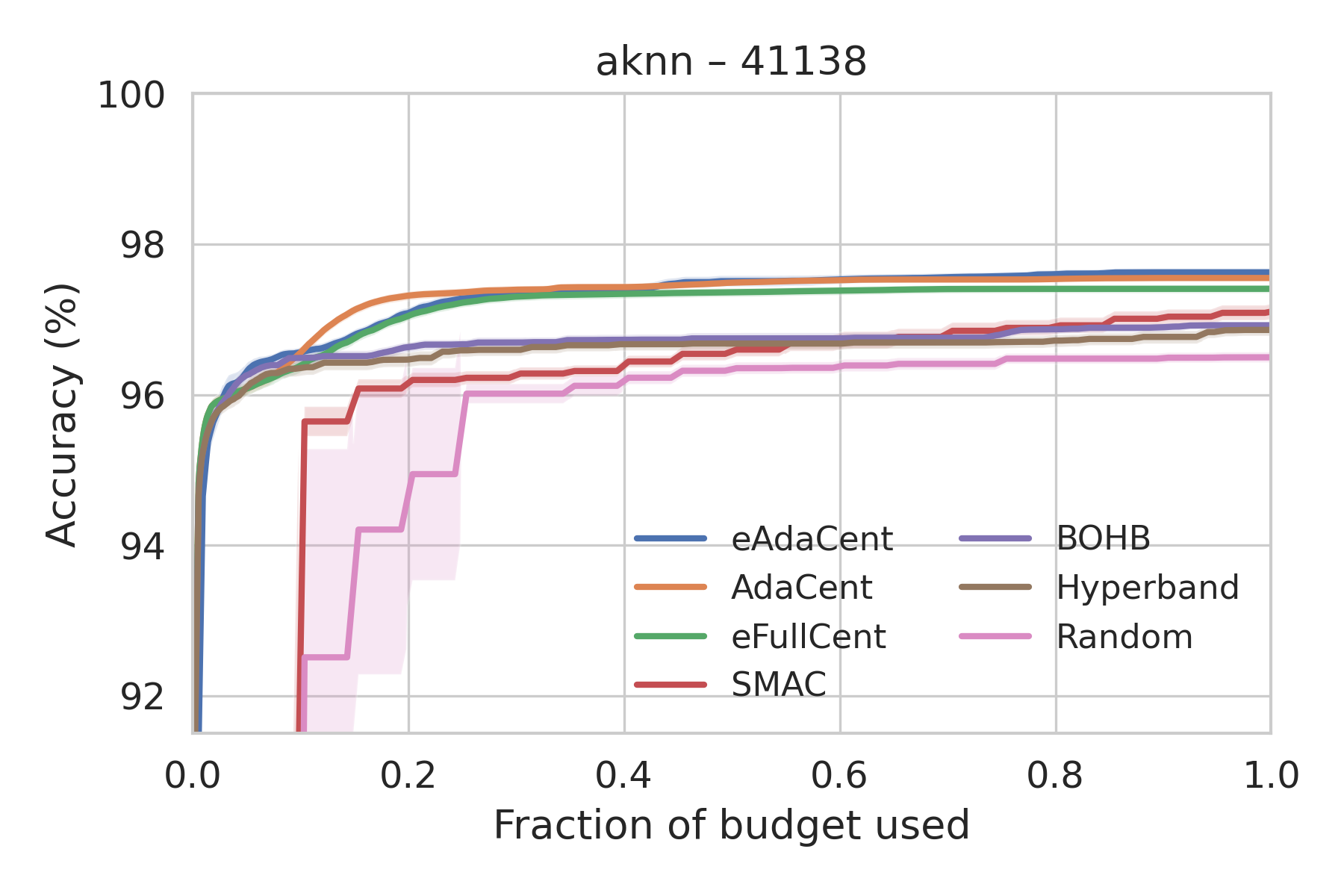}%
    \label{fig:acc:ak3}
  \end{subfigure}

\caption{Validation accuracy curves for three randomly selected datasets in both \texttt{rbv2\_rpart} and \texttt{rbv2\_aknn}. \AdaCent{} and \eAdaCent{} consistently outperform the baselines. Curves for all $100{+}$ datasets are provided in Appendix~\ref{sec:figs}.}%
\label{fig:acc:rbv2}
\end{figure}

\subsection{Empirical Validation of Assumption~\ref{assump:assump2}}
\label{sec:val_assump}
In the HPO setting, Assumption~\ref{assump:assump2} states that for any hyperparameter configurations $\x_i,\x_j \in \X$ and fidelity $b \in [T]$, we have:
\[
{A(\x_i,b)}  \ge  \left(1 - \epsilon \,\|\x_i - \x_j\|_2\right) A(\x_j,b),
\]
where here $A(\x,b)\in[0,1]$ denotes the validation accuracy of configuration $\x$ evaluated at fidelity $b$. We evaluate this assumption using the original tabular data on which the YAHPO Gym surrogates were trained.
For each configuration pair \((i, j)\), we define
\[
\epsilon_{ij} = \max \left\{ 0,\, \frac{1 - \min_{b \in [T]} A(\x_i,b) / A(\x_j,b)}{\|\x_i - \x_j\|_2} \right\},
\]
as the smallest value that makes the inequality tight for that pair. Because the bound depends on inter-configuration distance, we report the \(\alpha\)-percentiles of \(\epsilon_{ij}\, r\), where \(r\) is the clustering radius returned by the greedy \(k\)-center algorithm, for \(\alpha \in \{90,95,98,99\}\). Percentiles are computed over all experimental settings. Table~\ref{tab:summary-epsilon} shows a representative subset, randomly sampled from each scenario; full results appear in Appendix~\ref{sec:figs}.


\begin{table}[H]
\centering
\normalsize
\caption{Selected $\alpha$-percentiles of $\epsilon \cdot r$ from original tabular data, providing evidence for Assumption~\ref{assump:assump2}.}

\label{tab:summary-epsilon}
\fontsize{9}{10.5}\selectfont
\setlength{\tabcolsep}{4.5pt}
\renewcommand{\arraystretch}{1.08}

\begin{tabular}{l r
                S[table-format=1.4]
                S[table-format=1.4]
                S[table-format=1.4]
                S[table-format=1.4]}
\toprule
\textbf{Suite} & \textbf{Task ID} &
\multicolumn{1}{c}{$\boldsymbol{\alpha=90}$} &
\multicolumn{1}{c}{$\boldsymbol{\alpha=95}$} &
\multicolumn{1}{c}{$\boldsymbol{\alpha=98}$} &
\multicolumn{1}{c}{$\boldsymbol{\alpha=99}$} \\
\midrule
\multirow{4}{*}{\texttt{lcbench}}   & 167083 & 0.0374 & 0.0438 & 0.0522 & 0.0590 \\
                                   & 167190 & 0.3707 & 0.4760 & 0.5949 & 0.6853 \\
                                   & 189909 & 0.3621 & 0.5278 & 0.6848 & 0.8003 \\
                                   & 168329 & 0.8584 & 0.9864 & 1.1633 & 1.3061 \\
\midrule
\multirow{3}{*}{\texttt{rbv2\_rpart}}& 377    & 0.1568 & 0.2124 & 0.3321 & 0.3641 \\
                                   & 40536  & 0.0013 & 0.0022 & 0.0037 & 0.0045 \\
                                   & 41212  & 0.0211 & 0.0334 & 0.0417 & 0.0429 \\
\midrule
\multirow{3}{*}{\texttt{rbv2\_aknn}} & 377    & 0.4216 & 0.4686 & 0.5103 & 0.5503 \\
                                   & 40536  & 0.0052 & 0.0067 & 0.0105 & 0.0110 \\
                                   & 41212  & 0.0553 & 0.0587 & 0.0630 & 0.0674 \\
\bottomrule
\end{tabular}
\end{table}

\subsection{Results}

\paragraph{\textbf{Overall performance and stability.}}
Across all three YAHPO Gym scenarios, \eAdaCent{} delivers the strongest overall performance and the most stable behavior across tasks. We observe consistent gains in anytime performance, measured by best-so-far validation accuracy as the budget increases, and in robustness across heterogeneous tasks, summarized by mean ranks. As shown in Figure~\ref{fig:meanrank:all}, \eAdaCent{} reaches a top mean rank quickly and maintains it throughout the budget horizon. Its rank stabilizes early, around 20--25\% of the total budget, and remains strong on both \texttt{lcbench} and the two \texttt{rbv2} scenarios. This suggests that it identifies promising regions early and preserves this advantage as more budget is consumed. The early stabilization is consistent with using low-fidelity feedback to filter the search space, followed by focused refinement at higher fidelities. Our methods also exhibit lower run-to-run variability: over 30 independent runs, \AdaCent{} and \eAdaCent{} show smaller dispersion than the baselines in both the rank trajectories and the per-task curves in Figures~\ref{fig:acc:lcbench} and \ref{fig:acc:rbv2}. \AdaCent{} usually ranks second, indicating that adaptive pruning captures a substantial fraction of the gains, while the neighborhood-aware exploration of \eAdaCent{} provides an additional consistent improvement.

\paragraph{\textbf{Anytime behavior and search dynamics.}}
The per-task trajectories clarify the mechanism behind these improvements. Hyperband and BOHB often make rapid initial progress by exploiting low-fidelity signals, but can plateau once a candidate set is fixed and additional budget is spent on configurations that do not improve the incumbent. This is most apparent when low-fidelity rankings are noisy, promising regions are missed early, or good configurations occupy only a small portion of the space. In contrast, \eAdaCent{} continues improving under the same budget by combining aggressive pruning with targeted exploration around strong incumbents. This helps correct early uncertainty and concentrate evaluations where further gains remain plausible. Such structured reallocation supports sustained late-stage improvement rather than early saturation, particularly when strong solutions form coherent neighborhoods that can be discovered incrementally. Complete curves for all tasks and datasets are provided in Appendix~\ref{sec:figs}, confirming these trends across more than $250$ experimental settings.

\paragraph{\textbf{Effectiveness of adaptive allocation.}}
Across all three benchmark scenarios, \AdaCent{} and \eAdaCent{} consistently outperform the non-adaptive \eFullCent{} in both anytime accuracy and mean rank. At the same time, \eFullCent{} performs comparably to the standard baselines, demonstrating the effectiveness of enhanced distances even without adaptive allocation. Because \eAdaCent{} and \eFullCent{} differ both in their selection values (early versus terminal) and in whether pruning is applied, the additional improvement of \eAdaCent{} reflects the combined, end-to-end effect of adaptive pruning and early-value-informed selection rather than a single component in isolation. Together, these results show that enhanced distances provide a competitive non-adaptive strategy, while adaptive allocation yields further gains.

\paragraph{\textbf{Number of evaluated configurations.}}
To characterize allocation behavior, we measured the average number of distinct configurations evaluated on \texttt{lcbench}. Random Search, Hyperband, BOHB, SMAC, \eFullCent{}, \AdaCent{}, and \eAdaCent{} evaluated \(20.00\), \(82.00\), \(82.00\), \(24.00\), \(20.00\), \(100.62\), and \(67.06\) configurations, respectively. \AdaCent{} explores the most by pruning unpromising candidates early and reallocating the saved budget. However, more evaluations are not inherently better: \eAdaCent{} performs better with fewer, suggesting that the enhanced-distance rule avoids low-potential regions and focuses the budget on a smaller, more promising set.

\section{Conclusion}
We introduced UVP as a geometric model for budget-constrained multi-fidelity HPO and developed clustering-based methods that incorporate feedback while retaining worst-case guarantees. We integrated observed performance into the search geometry and introduced adaptive variants that reallocate budget through pruning and neighborhood-aware exploration, enabling better use of structural information. We proved approximation guarantees and showed they are essentially tight through a matching lower bound, yielding minimax optimality up to constant factors. Across $250{+}$ YAHPO Gym settings, our methods improve anytime and final performance over strong baselines, exhibit more stable behavior across runs, and provide evidence supporting the smoothness assumption on the tabular data. Overall, these results support geometry-aware, feedback-driven budget allocation for multi-fidelity HPO.

\bibliographystyle{plain}
\bibliography{main}

\clearpage
\appendix
\startcontents[appendix]

\clearpage
\section*{Appendix Contents}
\printcontents[appendix]{}{1}{\setcounter{tocdepth}{2}}

\clearpage

\section{Assumption~\ref{assump:assump2}}
\label{sec:A1}

Assumption~\ref{assump:assump2} imposes a \emph{relative} smoothness
condition on the performance landscape induced by the configuration
embedding. Specifically, for any fixed budget $b\in[T]$,
\[
A(\x_i,b)
\geq
\left(1-\epsilon d(\x_i,\x_j)\right)A(\x_j,b).
\]
Since the assumption holds for every ordered pair
$\x_i,\x_j\in\X$, applying it after exchanging $\x_i$ and $\x_j$ also
gives
\[
A(\x_j,b)
\geq
\left(1-\epsilon d(\x_i,\x_j)\right)A(\x_i,b).
\]
Thus, the performance values of two configurations cannot differ by more
than a multiplicative factor determined by their embedding distance. This
formalizes the intended role of the embedding: configurations that are
close in the embedding space should exhibit comparable performance at the
same budget.

Because $A(\x,b)\in[0,1]$, the assumption is informative only for pairs
satisfying
\(
1-\epsilon d(\x_i,\x_j)>0
\).
For larger distances, the multiplicative lower bound becomes non-positive
and therefore holds trivially. Accordingly,
Assumption~\ref{assump:assump2} should be interpreted as a \emph{local}
regularity condition: it meaningfully constrains pairs within radius
$1/\epsilon$, which is precisely the regime exploited by the covering and
clustering arguments in the main text.

Equivalently, the smallest valid smoothness parameter on the finite set
$\X$ can be written as
\[
\epsilon^\star
=
\sup_{\x_i\neq\x_j}
\frac{
\max_{b\in[T]}
\left(
1-
\min
\left\{
\frac{A(\x_i,b)}{A(\x_j,b)},
\frac{A(\x_j,b)}{A(\x_i,b)}
\right\}
\right)
}
{d(\x_i,\x_j)},
\]
where the ratio convention is $0/0=1$ and $a/0=+\infty$ for $a>0$.
Thus Assumption~\ref{assump:assump2} holds for every
$\epsilon\geq\epsilon^\star$.

\subsection{Additive Lipschitz continuity on $\X$}

Although Assumption~\ref{assump:assump2} is expressed multiplicatively,
it implies a standard additive Lipschitz condition on the finite metric
space $\X$, uniformly over budgets. Importantly, this statement does not
assert an extension to all of $\R^p$, and the converse need not hold:
additive smoothness does not prevent large relative differences when
performance values are small.

\begin{lemma}
\label{lem:budget-lip}
Under Assumption~\ref{assump:assump2}, every pair
$\x_i,\x_j\in\X$ satisfies
\[
\max_{b\in[T]}
\bigl|A(\x_i,b)-A(\x_j,b)\bigr|
\le
\epsilon d(\x_i,\x_j).
\]
\end{lemma}

\begin{proof}
Fix any $b\in[T]$. By Assumption~\ref{assump:assump2},
\[
A(\x_i,b)
\ge
\left(1-\epsilon d(\x_i,\x_j)\right)A(\x_j,b).
\]
Applying the assumption after exchanging $\x_i$ and $\x_j$, and using
the symmetry of the metric, gives
\[
A(\x_j,b)
\ge
\left(1-\epsilon d(\x_i,\x_j)\right)A(\x_i,b).
\]
Rearranging gives
\[
A(\x_j,b)-A(\x_i,b)
\le
\epsilon d(\x_i,\x_j)A(\x_j,b),
\]
and
\[
A(\x_i,b)-A(\x_j,b)
\le
\epsilon d(\x_i,\x_j)A(\x_i,b).
\]
Combining the two inequalities yields
\[
|A(\x_i,b)-A(\x_j,b)|
\le
\epsilon d(\x_i,\x_j)
\max\{A(\x_i,b),A(\x_j,b)\}.
\]
Since $A(\cdot,\cdot)\in[0,1]$, the maximum term is at most $1$, and hence
\[
|A(\x_i,b)-A(\x_j,b)|
\le
\epsilon d(\x_i,\x_j).
\]
Taking the maximum over all $b\in[T]$ completes the proof.
\end{proof}

\subsection{A functional-distance viewpoint}

It is useful to view each configuration $\x$ as inducing a discrete
performance curve
\[
\Phi(\x)=A(\x,\cdot)\in[0,1]^T.
\]
Assumption~\ref{assump:assump2} can then be phrased as Lipschitz
continuity of the map
\(
\Phi:\X\rightarrow[0,1]^T
\)
under a scale-invariant distance between curves.

For two curves $u,v\in[0,1]^T$, define
\begin{equation}
\label{eq:curve-metric}
d_{\mathrm{rel}}(u,v)
=
1-
\min_{b\in[T]}
\left\{
\frac{u(b)}{v(b)},
\frac{v(b)}{u(b)}
\right\},
\end{equation}
with the conventions
\(0/0=1\), and
\(a/0=+\infty\) for \(a>0\).

The quantity $d_{\mathrm{rel}}(u,v)$ measures the largest relative
discrepancy between two performance curves across all budgets. It is
invariant to common positive scaling:
\[
d_{\mathrm{rel}}(cu,cv)=d_{\mathrm{rel}}(u,v),
\qquad c>0,
\]
and therefore captures relative differences in curve values rather than
absolute scale.

\begin{lemma}[Equivalent formulation]
\label{lem:metric-equiv}
Assumption~\ref{assump:assump2} is equivalent to
\[
d_{\mathrm{rel}}
\left(
A(\x_i,\cdot),
A(\x_j,\cdot)
\right)
\le
\epsilon d(\x_i,\x_j)
\]
for every $\x_i,\x_j\in\X$.
\end{lemma}

\begin{proof}
Let
\(
\delta=\epsilon d(\x_i,\x_j)
\).
By definition, $d_{\mathrm{rel}}(u,v)\le\delta$ is equivalent to
\[
\min_{b\in[T]}
\left\{
\frac{u(b)}{v(b)},
\frac{v(b)}{u(b)}
\right\}
\ge
1-\delta.
\]
This holds if and only if, for every budget $b\in[T]$,
\[
u(b)\ge(1-\delta)v(b)
\qquad
\text{and}
\qquad
v(b)\ge(1-\delta)u(b),
\]
using the division conventions above. Substituting
$u(\cdot)=A(\x_i,\cdot)$,
$v(\cdot)=A(\x_j,\cdot)$, and
$\delta=\epsilon d(\x_i,\x_j)$ gives
Assumption~\ref{assump:assump2} and the inequality obtained by exchanging
$\x_i$ and $\x_j$. Since Assumption~\ref{assump:assump2} holds for every
ordered pair, the two formulations are equivalent.
\end{proof}

\section{Proofs}\label{sec:A3}

 \subsection{Theorem~\ref{thm:FC-approx}}
\label{subsec:approx}
\begin{proof}
    Let $\x^\star$ be the optimal solution to the UVP problem. From Assumption~\ref{assump:assump1}, we know that $A(\x, \cdot)$ is an increasing function, meaning the maximum value of $A(\x^\star, b)$ is attained when $b = T$. Thus, we have:
    \[
    A(\x^\star, T) = \max_{b_1, \ldots, b_n} \left\{\max_{\x_i \in \X}  A(\x_i, b_i)\right\}.
    \]

    Next, consider the output of \FullCent{}, denoted as $\hat\x$, which selects a set of centers $\C$ using the \KCenter{} procedure. Assume that $\x^\star$ lies within the cluster of some center $\bc \in \C$. Let $r_k$ represent the clustering radius obtained from \KCenter{}.

    By Assumption~\ref{assump:assump2}, we know that:
    \[
    A(\bc, T) \geq \left( 1 - \eps   d(\x^\star, \bc) \right) A(\x^\star, T).
    \]
    Using the fact that the clustering radius $r_k$ provides an upper bound on the distance between $\bc$ and $\x^\star$, i.e., $r_k \geq d(\x^\star, \bc)$, we obtain:
    \[
    A(\bc, T) \geq \left( 1 - \eps r_k\right) A(\x^\star, T).
    \]

    Now, since \FullCent{} selects $\hat\x$ as the configuration with the largest observed value $V_\bc$ from the centers in $\C$, and all configurations in $\C$ receive the full budget $T$, we know that:
    \[
    A(\hat\x, T) = \max_{\bc \in \C} V_\bc.
    \]
    Since $\bc$ is one of the centers in $\C$, it follows that:
    \[
    A(\hat\x, T) \geq A(\bc, T).
    \]

    Finally, by the approximation guarantee of the \KCenter{} algorithm, which in the case of \FullCent{} is the standard greedy $k$-center algorithm, we have $r_k \leq 2r_k^\star$, where $r_k^\star$ is the optimal clustering radius. Therefore, combining the above inequalities, we get:
    \[
    A(\hat\x, T) \geq \left( 1 - 2\eps r_k^\star \right) A(\x^\star, T).
    \]
    This completes the proof.
\end{proof}

\subsection{Theorem~\ref{thm:FC-hard}}
\label{subsec:hard}
\begin{proof} Construct an instance of the UVP problem with $\lceil \beta k \rceil$ symmetric clusters, each containing $N$ configurations. Inter-cluster distances are at least $\inveps$, and intra-cluster distances are $r$. A single cluster $\copt$ is chosen uniformly at random to contain one optimal configuration $\xopt$ and $N{-}1$ suboptimal ones $\xsub$ satisfying:
\[
A(\xopt,b) = A(\xsub,b) = 0 \quad \forall b < T,
\]
\[
A(\xopt,T) = 1,\quad A(\xsub,T) = 1 - \eps r.
\]
All configurations in other clusters yield zero accuracy at all budgets:
\[
\forall \cl \neq \copt,\ \forall \x \in \cl,\ \forall b \in [T]:\quad A(\x,b) = 0.
\]

By Yao's minimax principle~\cite{yao1977minimax}, we analyze the performance of any deterministic algorithm $\alg$ against a randomly chosen input instance. Without loss of generality, we assume:
(i) each evaluation costs $T$ (since all functions are $0$ for budgets $< T$, if $\alg$ does not spend $T$ it cannot distinguish the optimal cluster from an all-zero function);
(ii) upon finding a non-zero configuration, $\alg$ continues probing its cluster; and
(iii) $\alg$ uses the entire budget $B$ (which here is equivalent to using all its \(k\) evaluations).

Let $\x^{\alg}$ denote the output of $\alg$, and define:
\begin{itemize}
    \item $\pmiss^{(i)}$: probability that $\alg$ has not probed any point from $\copt$ before step $i$;
    \item $\pnew^{(i)}$: probability that the $i$-th probe is from a previously unseen cluster;
    \item $\nopt^{(i)}$: number of non-optimal clusters probed before step $i$.
\end{itemize}
The expected accuracy is bounded by the probability of discovering $\xopt$ or some $\xsub$:
\begin{align*}
    \mathbb{E}[A(\x^{\alg}, T)] &= \pr\left(\bigcup_{i=1}^{k} \alg 
    \text{ learns } \xopt \text{ at step } i \text{ for the first time} \right) \\
    &\quad + \left(1 - \eps r\right) \pr\left(\bigcup_{i=1}^{k} \alg 
    \text{ learns } \xsub \text{ at step } i \text{ but not } \xopt \right) \\
    &\leq \sum_{i=1}^{k} \pr\left( \alg \text{ learns } \xopt \text{ at step } i \text{ for the first time} \right) \\
    &\quad + \left(1 - \eps r\right) \sum_{i=1}^{k} \pr\left( \alg \text{ learns } \xsub \text{ at step } i \text{ for the first time} \right) \\
    &\leq \sum_{i=1}^{k} \frac{1}{N} \frac{\pmiss^{(i)} \pnew^{(i)}}{\lceil \beta k \rceil - \nopt^{(i)}} 
    + \left(1 - \eps r\right) \sum_{i=1}^{k} \frac{N - 1}{N} \frac{\pmiss^{(i)} \pnew^{(i)}}{\lceil \beta k \rceil - \nopt^{(i)}}.
\end{align*}

Since $\nopt^{(i)} \le k$, we have $\lceil \beta k \rceil - \nopt^{(i)} \ge (\beta - 1)k$, yielding:
\[
\mathbb{E}[A(\x^{\alg}, T)] \le \left[\frac{1}{N(\beta - 1)} + \frac{1 - \eps r}{\beta - 1}\right] \cdot A(\x^\star, T).
\]

Finally, for any $\delta > 0$, choosing \(N > \left\lceil \frac{1}{(1 - \eps r) \delta} \right\rceil\) yields
\[
\mathbb{E}[A(\x^{\alg}, T)] < (1 + \delta) \frac{1 - \eps r}{\beta - 1} \cdot A(\x^\star, T),
\]
contradicting any claimed $(1+\delta)$-approximation. This concludes the proof.
\end{proof}

\subsection{Theorem~\ref{thm:EFC-approx}}
\label{subsec:EFC-approx}

\begin{proof}
Let $\C^\star=\{\bc_1^\star,\dots,\bc_k^\star\}$ be an optimal $k$-center solution with radius
$r_k^\star$ and optimal clusters $\cl(\bc^\star)$.  
Let $\bc_1,\ldots,\bc_k$ denote the centers selected by
\eFullCent{}, and let
\[
\C^{(i)}=\{\bc_1,\ldots,\bc_i\},
\qquad
\C=\C^{(k)}.
\]
For \eFullCent{}, every selected center is evaluated to fidelity $T$
before it is used for subsequent enhanced distance computations, so its
selection value satisfies
\[
S_\bc = A(\bc,T).
\]
For each iteration $i$, define
\[
S_{\max}^{(i)}
=
\max_{\bc'\in\C^{(i)}} S_{\bc'},
\qquad
\eta_\bc^{(i)}
=
\frac{S_{\max}^{(i)}}{S_\bc}
\quad
(\bc\in\C^{(i)}),
\]
and let the iteration-$i$ enhanced distance be
\[
\dti^{(i)}(\x,\bc)
=
\min\left\{
d(\x,\bc),\;
\eta_\bc^{(i)}d(\x,\bc)
-\inveps\bigl(\eta_\bc^{(i)}-1\bigr)
\right\}.
\]
As usual, if $S_\bc=0$, we set
$\dti^{(i)}(\x,\bc)=d(\x,\bc)$.
Finally, let
\[
\x^{\text{E-FC}}
\in
\arg\max_{\bc\in\C} A(\bc,T)
\]
be the configuration returned by \eFullCent{}. Since
$S_\bc=A(\bc,T)$ for \eFullCent{}, we equivalently have
\[
A(\x^{\text{E-FC}},T)=\max_{\bc\in\C}S_\bc.
\]
When the iteration is clear from context, we write
$\eta_\bc=\eta_\bc^{(k)}$.

For any $\x_a,\x_b\in\cl(\bc^\star)$,
\[
\dti^{(i)}(\x_a,\x_b) \le d(\x_a, \x_b)
  \le d(\x_a,\bc^\star)+d(\x_b, \bc^\star)
  \le 2r_k^\star.
\]
Consider the first iteration \(i\geq 2\) in which the algorithm chooses
\(\bc_i\in\cl(\bc^\star)\) while some
\(\bc_j\in\cl(\bc^\star)\) was already chosen
(if no such iteration exists, then the first \(k\) selected centers lie
in distinct optimal \(k\)-center clusters, so one lies in
\(\cl(\bc^\star)\) and is within distance \(2r_k^\star\) of
\(\bc^\star\), yielding the claimed bound directly from
Assumption~\ref{assump:assump2}). Before step $i$, every center in $\C^{(i-1)}$ lies in a distinct
optimal cluster, hence for any unchosen $\x\in\X$
\[
\min_{\bc\in\C^{(i-1)}} \dti^{(i-1)}(\x,\bc)
      \le \dti^{(i-1)}(\bc_i,\bc_j)
      \le 2r_k^\star .
\]
The non-constant part of $\dti^{(i)}$ has slope
$d(\cdot, \cdot)-\inveps\le0$ and
$S_{\max}^{(i)}$ is non-decreasing, so $\dti^{(i)}$ is
non-increasing in $i$:
$\dti^{(k)}(\x,\bc)\le\dti^{(i-1)}(\x,\bc)$ for all $\x,\bc$.
Thus
\[
\min_{\bc\in\C}\dti^{(k)}(\x,\bc) \le 2r_k^\star
\qquad(\forall\x\in\X).
\]
Let $\x^\star$ be an optimal solution to the
UVP problem and choose
$\bc\in\C$ with $\x^\star\in\cl(\bc)$.
Then $\dti^{(k)}(\x^\star,\bc)\le2r_k^\star$, i.e.
\[
\min\,\Bigl\{d(\x^\star, \bc),\;
\eta_\bc^{(k)}d(\x^\star, \bc)-\inveps(\eta_\bc^{(k)}-1)\Bigr\}
\le2r_k^\star .
\]
We distinguish two cases:
\begin{enumerate}
    \item \textbf{Case 1: Metric bound is active.} If $d(\x^\star, \bc) \le 2r_k^\star$, the standard analysis applies immediately:
    \[
    \begin{aligned}
    A(\x^{\text{E-FC}}, T)
    &\ge  A(\bc, T) \\
    &\ge A(\x^\star, T)\left(1 - \epsilon d(\x^\star, \bc)\right) \\
    &\ge A(\x^\star, T)(1 - 2\epsilon r_k^\star).
    \end{aligned}
    \]
    
    \item \textbf{Case 2: Enhanced bound is active.} Here,
    $\eta_\bc d(\x^\star,  \bc) - \frac{1}{\epsilon}(\eta_\bc - 1) \le 2r_k^\star$.
    Rearranging for the metric distance:
    \[
    d(\x^\star, \bc) \le \frac{2r_k^\star}{\eta_\bc} + \frac{1}{\epsilon}\left(1 - \frac{1}{\eta_\bc}\right).
    \]
    Substituting this into the Lipschitz lower bound
    $A(\bc, T) \ge A(\x^\star, T)\left(1 - \epsilon d(\x^\star, \bc)\right)$
    and recalling that, for Enhanced-FullCent, $S_\bc=A(\bc,T)$ and hence
    $A(\x^{\text{E-FC}}, T)=\eta_\bc S_\bc=\eta_\bc A(\bc,T)$, we derive:
    \[
    \begin{aligned}
    \frac{A(\x^{\text{E-FC}}, T)}{A(\x^\star, T)}
    &= \eta_\bc \frac{A(\bc, T)}{A(\x^\star, T)} \\
    &\ge \eta_\bc \left( 1 - \epsilon \left[ \frac{2r_k^\star}{\eta_\bc} + \frac{1}{\epsilon}\left(1 - \frac{1}{\eta_\bc}\right) \right] \right) \\
    &= \eta_\bc - 2\epsilon r_k^\star - (\eta_\bc - 1) \\
    &= 1 - 2\epsilon r_k^\star.
    \end{aligned}
    \]
\end{enumerate}
In both cases, the approximation ratio is at least $(1 - 2\epsilon r_k^\star)$.
\end{proof}

\subsection{Theorem~\ref{thm:AC-approx}}
\label{subsec:approx-adacent}
\begin{proof}
 Let $\x^\star$ be a configuration attaining the UVP optimum $\mathrm{OPT}=A(\x^\star,T)$ under Assumption~\ref{assump:concavity}, and let \(k=\min\{|\X|,\lfloor B/T\rfloor\}\). The first \(k\) centers selected by \AdaCent{} form the protected set \(\mathcal P\) of Algorithm~\ref{alg:adacent}. Since greedy \(k\)-center is a \(2\)-approximation, these \(k\) centers cover every configuration within radius \(2r_k^\star\), so there is a protected center
\(\bc\in\mathcal P\) with
\[
  d(\x^\star, \bc) \le 2r_k^\star.
\]
Using Assumption~\ref{assump:assump2},
\begin{equation}
  A(\bc,T)
      \ge
      \left(1-\eps d(\x^\star,\bc)\right)\,A(\x^\star,T)
      \ge
      \bigl(1-2\eps r_k^\star\bigr)\,A(\x^\star,T).
  \label{eq:bc-lower}
\end{equation}
After each partial evaluation \(b<T\) of \(\bc\), it is guaranteed that
\(UB\bigl(\bc)\ge A(\bc,T)\) (Lemma~\ref{lem:adacent-ub}).
Therefore \(\bc\) can only be removed from the active pool \(\A\) if
some center configuration \(\bc'\) already attains
\(V_{\bc'} > UB\bigl(\bc)\ge A(\bc,T)\).
Consequently,
at the moment \(\bc\) would be pruned,
the algorithm has already observed a value
strictly larger than \eqref{eq:bc-lower}.
Otherwise, since \(\bc\) is protected, the completion-budget restriction of Algorithm~\ref{alg:adacent} guarantees that \(\bc\) is fully evaluated at budget \(T\). Let $M=\max_{\bc'\in\C}V_{\bc'}$ be the attained value returned by \AdaCent{}. Combining the arguments above gives
\[
     M
     \ge
     A(\bc,T)
     \ge
     \bigl(1-2\eps r_k^\star\bigr)
     A(\x^\star,T),
\]
which matches the claimed approximation factor.
\end{proof}

\subsection{Theorem~\ref{thm:AC-hard}}
\label{subsec:hard-adacent}
\begin{proof}
We work in the same adversarial setting as in the proof of Theorem~\ref{thm:FC-hard}, 
with the same clustering construction, assumptions (ii)--(iii), and the same variables 
$\pmiss^{(i)}$, $\pnew^{(i)}$ and $\nopt^{(i)}$ defined there.
The only differences are the definition of the accuracy functions and a revised assumption~(i).
Specifically, the accuracy functions for the configurations in \(\copt\) are now defined as
\[
\forall b \le T, \quad \begin{cases}A(\xopt,b) = \frac{b}{T}, \\ A(\xsub,b) = \frac{(1 - \eps r)b}{T},
\end{cases}
\]
Also for all other clusters,
\[
\forall \cl \neq \copt,\ \forall \x \in \cl,\ A(\x, b) = \begin{cases}
    \frac{(1 - \eps r)b}{T} & b < \lfloor \theta T \rfloor,
    \\
    \frac{(1 - \eps r)\lfloor \theta T\rfloor}{T} & b \ge \lfloor \theta T \rfloor.
\end{cases}
\]
We also replace assumption~(i) by: each evaluation uses at least $\lfloor \theta T + 1\rfloor$ budget, since otherwise $\alg$ cannot distinguish $\xsub$ from configurations in suboptimal clusters and thus cannot determine whether it is in the optimal cluster. Therefore, we can assume it performs at most $e = \left\lfloor B / \lfloor \theta T + 1\rfloor \right\rfloor$ evaluations, as any remaining budget will be spent but will not affect the maximum accuracy.

As before, we apply Yao's minimax principle~\cite{yao1977minimax} and analyze the performance of any deterministic 
algorithm $\alg$ on a randomly chosen instance. Note that since \(\lfloor \theta T\rfloor \le \theta T\), and therefore \(\frac{(1 - \eps r)\lfloor \theta T\rfloor}{T} \le \theta (1 - \epsilon r)\), we can bound the expected accuracy as
\[
\begin{aligned}
\mathbb{E}[A(\x^{\alg}, T)] &\leq \theta (1 - \epsilon r) + \sum_{i=1}^{e} \pr\left( \alg \text{ learns } \xopt \text{ at step } i \text{ for the first time} \right) \\
    &\quad + \left(1 - \eps r\right) \sum_{i=1}^{e} \pr\left( \alg \text{ learns } \xsub \text{ at step } i \text{ for the first time} \right) \\
    &\leq \theta (1 - \epsilon r) + \sum_{i=1}^{e} \frac{1}{N} \frac{\pmiss^{(i)} \pnew^{(i)}}{\lceil \beta k \rceil - \nopt^{(i)}} 
     + \left(1 - \eps r\right) \sum_{i=1}^{e} \frac{N - 1}{N} \frac{\pmiss^{(i)} \pnew^{(i)}}{\lceil \beta k \rceil - \nopt^{(i)}} \\&\le 
    \left[\left(\theta + \frac{1}{\theta (\beta - 1)}\right)(1 - \eps r) + \frac{1}{N\theta (\beta - 1)} \right] \cdot A(\x^\star, T), 
\end{aligned}
\]
where the last inequality comes from these two inequalities
\[
\begin{aligned}
 \lceil \beta k \rceil - \nopt^{(i)} &\ge (\beta - 1)k, \\
 \forall \theta \in (0,1), \quad\frac{e}{k} &\le \frac{1}{\theta}.
\end{aligned}
\]
Finally, for any $\delta > 0$, choosing \(N > \left\lceil \frac{1}{\left(1 + \theta^2(\beta - 1)\right)(1-\eps r)\delta} \right\rceil\) yields
\[
\mathbb{E}[A(\x^{\alg}, T)] < (1 + \delta) \left(\theta + \frac{1}{\theta (\beta - 1)} \right) (1 - \epsilon r) \cdot A(\x^\star, T),
\]
contradicting any claimed $(1+\delta)$-approximation. This concludes the proof.
\end{proof}

\section{Upper bounds for \textsc{AdaCent}}\label{sec:adacent_UB}

Because \textsc{AdaCent} observes $A(\x,\cdot)$ sequentially, it requires an online certificate that a partially evaluated configuration cannot overtake the current incumbent. Under Assumption~\ref{assump:concavity} (diminishing returns), a simple \emph{linear extrapolation from the latest increment} yields a valid upper bound on the terminal value $A(\x,T)$. Fix a configuration $\x$ and an intermediate budget $b\in\{2,\ldots,T-1\}$. Define
\begin{equation}\label{eq:adacent-ub}
UB_b(\x)
\;:=\;
A(\x,b) + (T-b)\bigl(A(\x,b)-A(\x,b-1)\bigr).
\end{equation}

\begin{lemma}[Validity of the upper bound]\label{lem:adacent-ub}
Under Assumption~\ref{assump:concavity}, for every $\x\in\X$ and every $b\in\{2,\ldots,T-1\}$,
\[
UB_b(\x)\;\ge\;A(\x,T).
\]
\end{lemma}

\begin{proof}
By Assumption~\ref{assump:concavity}, the marginal gains are non-increasing, hence for all $s\in\{b,\ldots,T-1\}$,
\[
A(\x,s+1)-A(\x,s)\;\le\;A(\x,b)-A(\x,b-1).
\]
Summing these inequalities over $s=b,\ldots,T-1$ yields
\[
\begin{aligned}
A(\x,T)-A(\x,b)
&=
\sum_{s=b}^{T-1}\bigl(A(\x,s+1)-A(\x,s)\bigr)
\\&\le
(T-b)\bigl(A(\x,b)-A(\x,b-1)\bigr).
\end{aligned}
\]
Rearranging gives $A(\x,T)\le UB_b(\x)$, as claimed.
\end{proof}

\paragraph{Implication for pruning.}
Therefore, if at some time $b$ a configuration has an upper bound $UB_b(\x)$ below the best fully (or currently) attainable value, it cannot become optimal at budget $T$ and may be safely pruned without risking elimination of the global maximizer. This ``safe pruning'' statement is conditional on Assumption~\ref{assump:concavity}; Appendix~\ref{app:pruning-diagnostics} reports how often the upper bound actually covers the terminal value and how often pruning removes the eventual best configuration on the benchmark suites.

\section{Conditional guarantee for \textsc{Enhanced-AdaCent}}\label{sec:eadacent-cond}

\eAdaCent{} couples enhanced distances with pruning, so its analysis requires an additional calibration condition linking early observations to terminal values. Write $f(\x)=A(\x,T)$ for the terminal value, $\mathrm{OPT}=\max_{\x\in\X}f(\x)$, and let $M$ be the best value observed within the global budget $B$.

\begin{assumption}[Early-to-terminal calibration]\label{assump:calibration}
There exist $a>0$ and $\kappa\ge1$ such that, for every $\x\in\X$,
\[
a\,f(\x)\;\le\;A(\x,b_0)\;\le\;a\kappa\,f(\x).
\]
\end{assumption}

\begin{theorem}\label{thm:eadacent-cond}
Let $k=\min\{|\X|,\lfloor B/T\rfloor\}$, and let the selection values be frozen at the common fidelity $b_0=\max\{1,\lfloor\delta T\rfloor\}$. Under Assumptions~\ref{assump:assump1}--\ref{assump:assump2}, \ref{assump:concavity}, and \ref{assump:calibration}, and with the completion-budget restriction of Algorithm~\ref{alg:adacent}, \eAdaCent{} attains
\[
M \;\ge\; \bigl(1-2\epsilon r_k^\star\bigr)\,\frac{\mathrm{OPT}}{\kappa}.
\]
\end{theorem}

\begin{proof}
Let $\mathcal P=\{\bc_1,\ldots,\bc_k\}$ be the first $k$ centers selected by Algorithm~\ref{alg:adacent}. The protected batches cost at most $kT\le B$, and no later center is admitted before they are resolved. Hence every center in $\mathcal P$ is completed or pruned under a valid terminal upper bound. Write
\[
S(\x)=A(\x,b_0),\qquad
S_{\max}=\max_{\bc\in\mathcal P}S(\bc),\qquad
\alpha=1-2\epsilon r_k^\star.
\]
If $\alpha\le0$ or $\mathrm{OPT}=0$, the claim is immediate, so assume otherwise. Smoothness at fidelity $b_0$ gives
\begin{equation}\label{eq:early-smoothness}
S(\bc)\ge\bigl(1-\epsilon d(\x,\bc)\bigr)S(\x).
\end{equation}
We first show $S_{\max}\ge\alpha S(\x^\star)$, where $f(\x^\star)=\mathrm{OPT}$.
If $\x^\star\in\mathcal P$, this is immediate. Otherwise, partition $\X$ into $k$ optimal metric $k$-center clusters of radius $r_k^\star$.
If the selected centers occupy distinct clusters, one shares a cluster with $\x^\star$, so its distance to $\x^\star$ is at most $2r_k^\star$ and \eqref{eq:early-smoothness} proves the claim.

If a cluster is repeated, consider its first repetition, when $\bc_i$ is selected and an earlier $\bc_j$ lies in the same cluster. Just before this selection, write $H=\max_{\ell<i}S(\bc_\ell)$. Since $\tilde d\le d$, the max--min rule gives
\[
\min_{\ell<i}\tilde d(\x^\star,\bc_\ell)
\le \min_{\ell<i}\tilde d(\bc_i,\bc_\ell)
\le d(\bc_i,\bc_j)\le2r_k^\star.
\]
Choose a center $\bc$ attaining the minimum on the left. If $d(\x^\star,\bc)\le2r_k^\star$, \eqref{eq:early-smoothness} again gives $S_{\max}\ge\alpha S(\x^\star)$. Otherwise the enhanced branch must be active, with $S(\bc)>0$ and $\eta=H/S(\bc)$, and
\[
\eta d(\x^\star,\bc)-\frac{\eta-1}{\epsilon}\le2r_k^\star
\quad\Longrightarrow\quad
1-\epsilon d(\x^\star,\bc)\ge\frac{\alpha}{\eta}.
\]
Multiplying \eqref{eq:early-smoothness} by $\eta$ yields the required certificate
\[
S_{\max}\ge H=\eta S(\bc)\ge\alpha S(\x^\star).
\]
This certifies the maximum selection value, without requiring the center of small enhanced distance itself to have a large terminal value.

Now choose $\bc_{\max}\in\mathcal P$ attaining $S_{\max}$. Calibration implies
\[
f(\bc_{\max})
\ge\frac{S_{\max}}{a\kappa}
\ge\frac{\alpha S(\x^\star)}{a\kappa}
\ge\frac{\alpha}{\kappa}\,\mathrm{OPT}.
\]
If $\bc_{\max}$ completes, its terminal value is observed. If it is pruned, the incumbent at that time exceeds its valid terminal upper bound and hence $f(\bc_{\max})$. In either case $M\ge f(\bc_{\max})$, proving the result.
\end{proof}

If the normalized progress $A(\x,b_0)/f(\x)$ is a common positive constant across configurations with $f(\x)>0$, the calibration gap satisfies $\kappa=1$. The attained-value ratios reported in Appendix~\ref{app:attained-ratios} confirm that the observed value stays close to $\mathrm{OPT}$ in practice.

\suppressfloats[t]
\section{Illustrative Example}
\label{app:illustrative-example}
Consider the setting in Figure~\ref{fig:example}, where $T = 1$ and each configuration lies in 2D, with \(d(\x, \x') = \|\x-\x'\|_2\) being the standard Euclidean distance. Three regions are centered along the $x$-axis, equally spaced by $d$. Two of them (red and green) are surrounded by rings of radius $r$ with constant function values $v_1 = 1 - \epsilon d$ and $v_2 = (1 - \epsilon d)^2$. A single blue point at $x = 0$ attains the optimal value $v_0 = 1$.

\begin{figure}[t]
\centering
\includegraphics[width=0.7\linewidth]{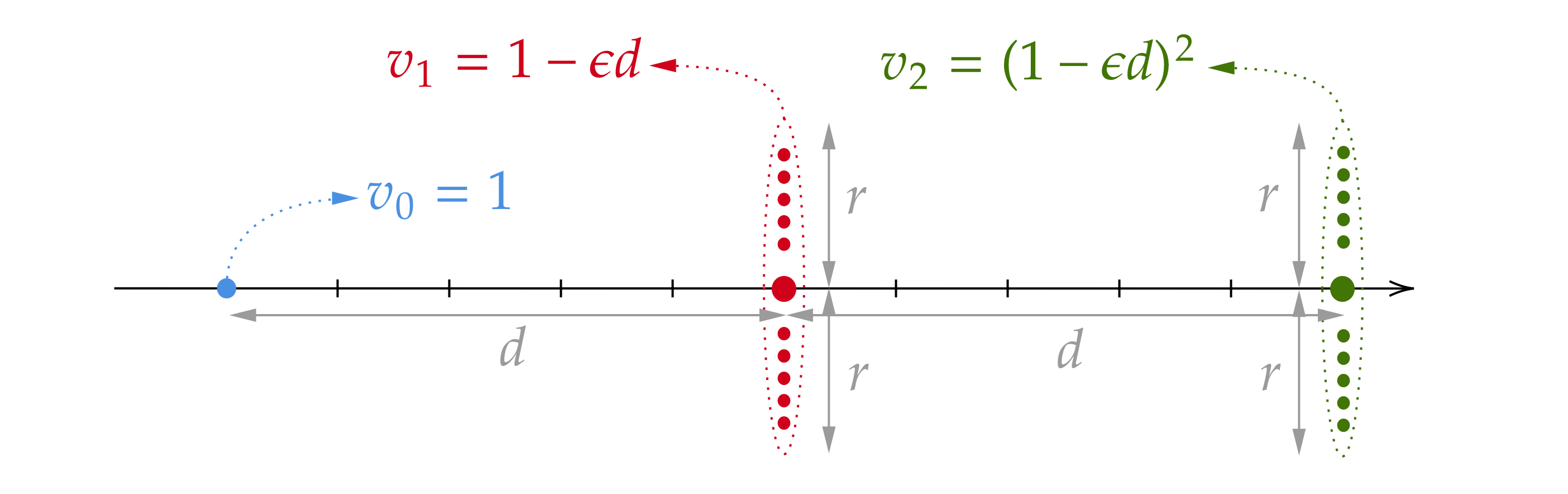}
\caption{Illustrative example where adjusted distances yield better clustering, resulting in probing the highest value point.}
\label{fig:example}
\end{figure}

\begin{claim}
\label{clm:example}
For $k = 2$, there exist values of $\epsilon$, $r$, and $d$ such that clustering with enhanced distances selects the optimal (blue) configuration as a center, while standard $k$-center clustering does not.
\end{claim}

\begin{proof}[Proof of Claim~\ref{clm:example}]
    When clustering is done via standard distances, the optimal blue point is never picked. More specifically, using the two centers of the red and green rings as the clustering centers would result in the clustering radius being $d$; however, picking the blue point would result in the clustering radius being at least $\sqrt{d^2+r^2}$. We now analyze the optimal clustering using \emph{enhanced} distances.
    
    Case 1: Suppose the two centers of the rings are picked as cluster centers. Then $v_\mathrm{max}$ would be $1 - \eps d$ and the furthest point from its center would be the blue point. The \emph{enhanced} distance of the blue point ($\x$) from the center of the red ring ($\bc$) could be calculated as follows:
\begin{align*}
\tilde{d} (\x, \bc) &= \frac{v_\mathrm{max}}{v_\bc}  \| \x - \bc \|_2 -\frac{1}{\eps}\left(\frac{v_\mathrm{max}}{v_\bc} - 1\right)= \| \x - \bc \|_2 = d
\end{align*}
In this case, the clustering radius is $d$.

    Case 2: Suppose the blue point and the center of the red ring are picked as cluster centers. The $v_\mathrm{max}$ would be $1$ and the points farthest from their centers would be the outer green points from the green ring. The \emph{enhanced} distance of an outer green point ($\x$) from the center of the red ring ($\bc$) could be calculated as follows:
    \begin{align*}
        \tilde{d} (\x, \bc) &= \frac{v_\mathrm{max}}{v_\bc}  \| \x - \bc \|_2 -\frac{1}{\eps}\left(\frac{v_\mathrm{max}}{v_\bc} - 1\right) \\
        &=\frac{\sqrt{d^2 + r^2}}{1-\eps d}  -\frac{1}{\eps}\left(\frac{1}{1-\eps d} - 1\right) \\
        &= \frac{\sqrt{d^2+r^2} - d}{1-\eps d}
    \end{align*}
    In this case, the clustering radius is $\frac{\sqrt{d^2+r^2} - d}{1-\eps d}$.

    Case 3: Suppose the blue point and the center of the green ring are picked as cluster centers. The $v_\mathrm{max}$ would be $1$ and the points farthest from their centers would be the outer red points from the red ring. It is not immediately clear which center these points select, so we check both cases. The \emph{enhanced} distance of an outer red point ($\x$) from the center of the green ring ($\bc$) could be calculated as follows:
    \begin{align*}
        \tilde{d} (\x, \bc) &= \frac{v_\mathrm{max}}{v_\bc}  \| \x - \bc \|_2 -\frac{1}{\eps}\left(\frac{v_\mathrm{max}}{v_\bc} - 1\right) \\
        &=\frac{\sqrt{d^2 + r^2}}{(1-\eps d)^2}  -\frac{1}{\eps}\left(\frac{1}{(1-\eps d)^2} - 1\right) \\
        &= \frac{\sqrt{d^2+r^2} - d(2-\eps d)}{(1-\eps d)^2}
    \end{align*}
     The \emph{enhanced} distance of an outer red point ($\x$) from the blue point ($\bc'$) could be calculated as follows:
     \begin{align*}
        \tilde{d} (\x, \bc') &= \frac{v_\mathrm{max}}{v_{\bc'}}  \| \x - \bc' \|_2 -\frac{1}{\eps}\left(\frac{v_\mathrm{max}}{v_\bc} - 1\right) =\| \x - \bc' \|_2 = \sqrt{d^2 + r^2}
    \end{align*}
    Now it is clear that the outer red points pick the center of the green ring as their center. So in this case, the clustering radius is
    \[
    \frac{\sqrt{d^2+r^2} - d(2-\eps d)}{(1-\eps d)^2}.
    \]

    Note that in order for case 1 not to happen, the clustering radius of cases 2 and 3 should be less than the clustering radius of case 1. It is easy to see that when $1 - \eps d > 0$ (which is the case considering the values should be positive), the clustering radius of case 3 is less than case 2. So it suffices that:
    \begin{align*}
        d > \frac{\sqrt{d^2+r^2} - d}{1-\eps d} \iff 2 > \sqrt{1+\left(\frac{r}{d}\right)^2} + \eps d
    \end{align*}
    Picking $r = \gamma d$ and $\eps = \kappa /d$ would result in the inequality being met for $\gamma, \kappa< 3/4$.
\end{proof}

\section[Synthetic-Landscape Comparison of \FullCent{} and \eFullCent{}]{Synthetic-Landscape Comparison:\\\FullCent{} and \eFullCent{}}
\label{app:synthetic-fullcent}

To visualize how \FullCent{} and
\eFullCent{} behave compared to each other, we construct
three analytic landscapes that provably satisfy Assumption \ref{assump:assump2}.  All
functions are defined on the square configuration domain
\(\mathcal D=[-8,8]^2\subset\mathbb R^2\).
For every experiment we draw \(n=10\,000\) configurations
\(\x_i\sim\mathcal U(\mathcal D)\) and evaluate a single-budget
performance metric \(f(\x) = A(\x, 1)\). Note that we evaluate each configuration at a single final budget (\(T = 1\)),
    hence Assumption \ref{assump:assump1} is trivially satisfied.
We set the smoothness parameter
\(\eps=0.2\) and the total budget to \(B = 10\).  Resulting plots are shown in Figure~\ref{fig:synthetic}.

\begin{itemize}
  \item[(a)] \textbf{Radial Decay}\\[2pt]
  The value function is an isotropic exponential
  \[
    f_{\text{rad}}(\x)=
    \exp\bigl(-\lambda\lVert\x\rVert_2\bigr),
    \qquad \lambda=0.18.
  \]

  The gradient magnitude is bounded by
  \(\lVert\nabla f_{\text{rad}}\rVert
  =\lambda f_{\text{rad}}\le\lambda\).  
  Setting \(\eps\ge\lambda=0.18\) gives the
  guarantee required by Assumption \ref{assump:assump2}.
  Using the mean value inequality, we have
  \[
    |f_{\text{rad}}(\x) - f_{\text{rad}}(\mathbf y)|
    \le \lambda \lVert\x - \mathbf y\rVert_2.
  \]
  Assuming \(f_{\text{rad}}(\x) \le f_{\text{rad}}(\mathbf y)\), we obtain
  \[
    \frac{f_{\text{rad}}(\x)}{f_{\text{rad}}(\mathbf y)}
    \ge 1 - \lambda \lVert\x - \mathbf y\rVert_2
    \ge  1 - \eps\lVert\x - \mathbf y\rVert_2.
  \]

  \item[(b)] \textbf{Off-Center Peak}\\[2pt]
  We superimpose a broad base and a narrow displaced peak
  \[
  \begin{aligned}
    f_{\text{off}}(\x)
    &=
    \underbrace{b
    \exp\bigl(-\lVert\x\rVert_2^2/(2\sigma_b^2)\bigr)}_{%
      \text{base, }\sigma_b=10}
    \\&\quad+
  \underbrace{\exp\!\bigl(-\lVert\x-\mathbf c\rVert_2^2/(2\sigma_p^2)\bigr)}_{%
      \text{peak, }\mathbf c=(0.2,-0.1), \sigma_p=5}.
  \end{aligned}
  \]
  We also set \(b = 0.6\) in our simulations.

  Both Gaussian components are infinitely differentiable; their sum is
  therefore smooth.  The gradient of each term is bounded by
  \(\lVert\nabla f_{\text{off}}\rVert_2\le \sigma_{\min}^{-1}f_{\text{off}}\), where
  \(\sigma_{\min}=\min\{\sigma_b,\sigma_p\}=5\).
  With a similar argument as before, choosing \(\eps=0.2\) again suffices for
  Assumption \ref{assump:assump2}.  Monotonicity in budget holds as before.

  \item[(c)] \textbf{Cosine Ring}\\[2pt]
  We define a smooth ring-shaped landscape by
  \[
  f_{\mathrm{ring}}(\x)
  =\begin{cases}
  b + \frac{h + h\cos\bigl(\tfrac{\pi}{w}(\|\x\|_2 - R)\bigr)}{2}
  & \bigl|\|\x\|_2 - R\bigr| \le w,\\
  b& \text{otherwise},
  \end{cases}
  \]
  with parameters
  \[
  R = 3, 
  \quad w = 3, 
  \quad h = 0.06,
  \quad b = 0.2.
  \]
  Outside the band $\bigl|\|\x\|_2 - R\bigr|\le w$, the function is constant
  $f_{\mathrm{ring}}=b$, and inside it has a single smooth cosine bump of
  height~$h$ above base~$b$.

  One checks
  \[
  \|\nabla f_{\mathrm{ring}}(\x)\|_2\le L=\frac{h\pi}{2w}.
  \]
  Since \(|f_{\mathrm{ring}}(\x)-f_{\mathrm{ring}}(\mathbf y)|
    \le L\|\x-\mathbf y\|_2\) (by the mean value theorem) and
  \(f_{\mathrm{ring}}\ge b\),  
  \[
  \frac{f_{\mathrm{ring}}(\x)}{f_{\mathrm{ring}}(\mathbf y)}
    \ge1-\frac{L}{b}\|\x-\mathbf y\|_2
    \ge 1 - \eps \|\x-\mathbf y\|_2.
  \]
\end{itemize}

\begin{figure}[t]
    \centering
    \begin{subfigure}[b]{0.32\textwidth}
        \centering
        \includegraphics[width=\linewidth,trim={40.080bp 18.000bp 92.640bp 23.760bp},clip]{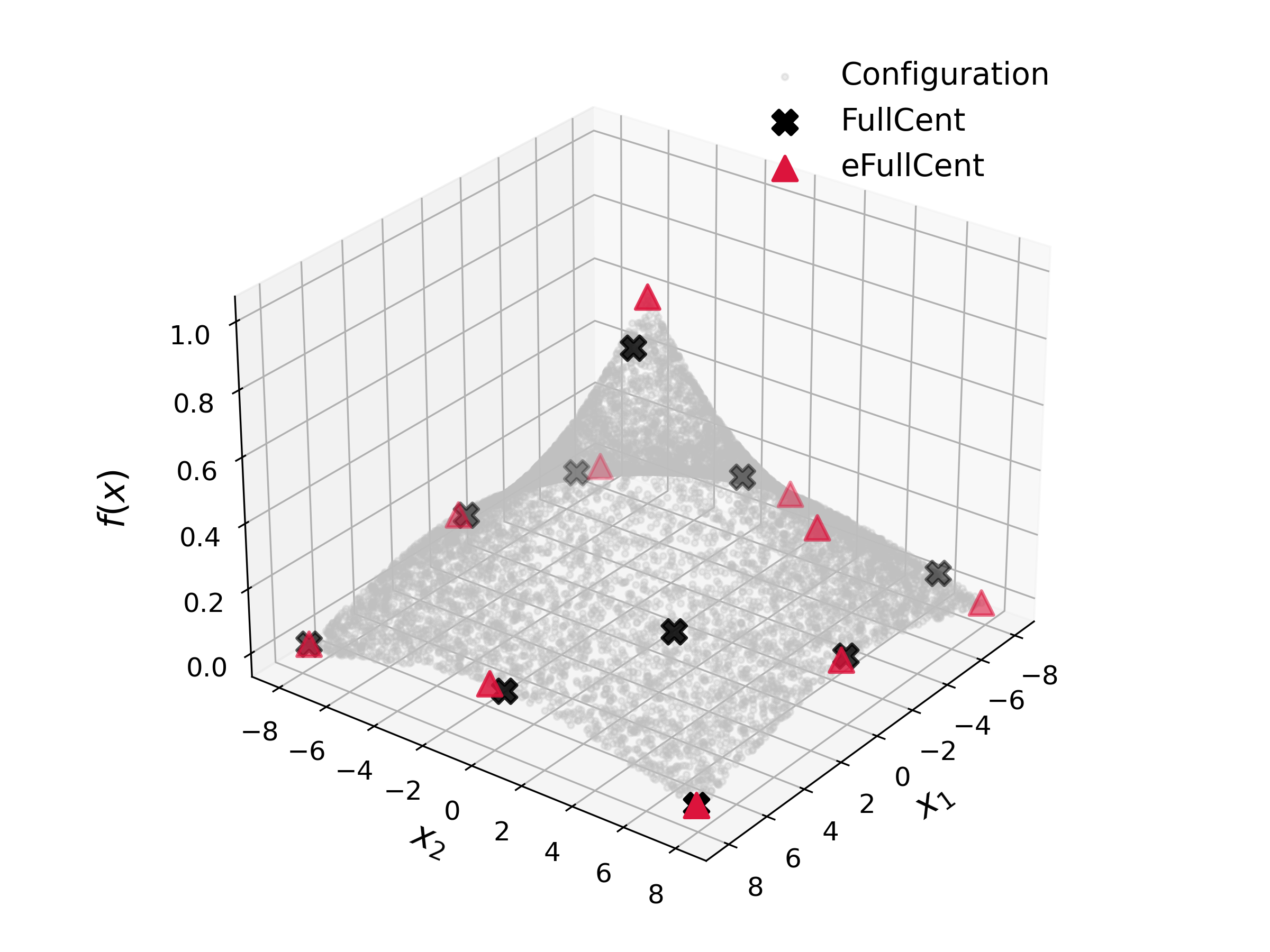}
        \caption{Radial decay}
        \label{fig:radial}
    \end{subfigure}\hfill
    \begin{subfigure}[b]{0.32\textwidth}
        \centering
        \includegraphics[width=\linewidth,trim={40.080bp 18.000bp 92.640bp 23.760bp},clip]{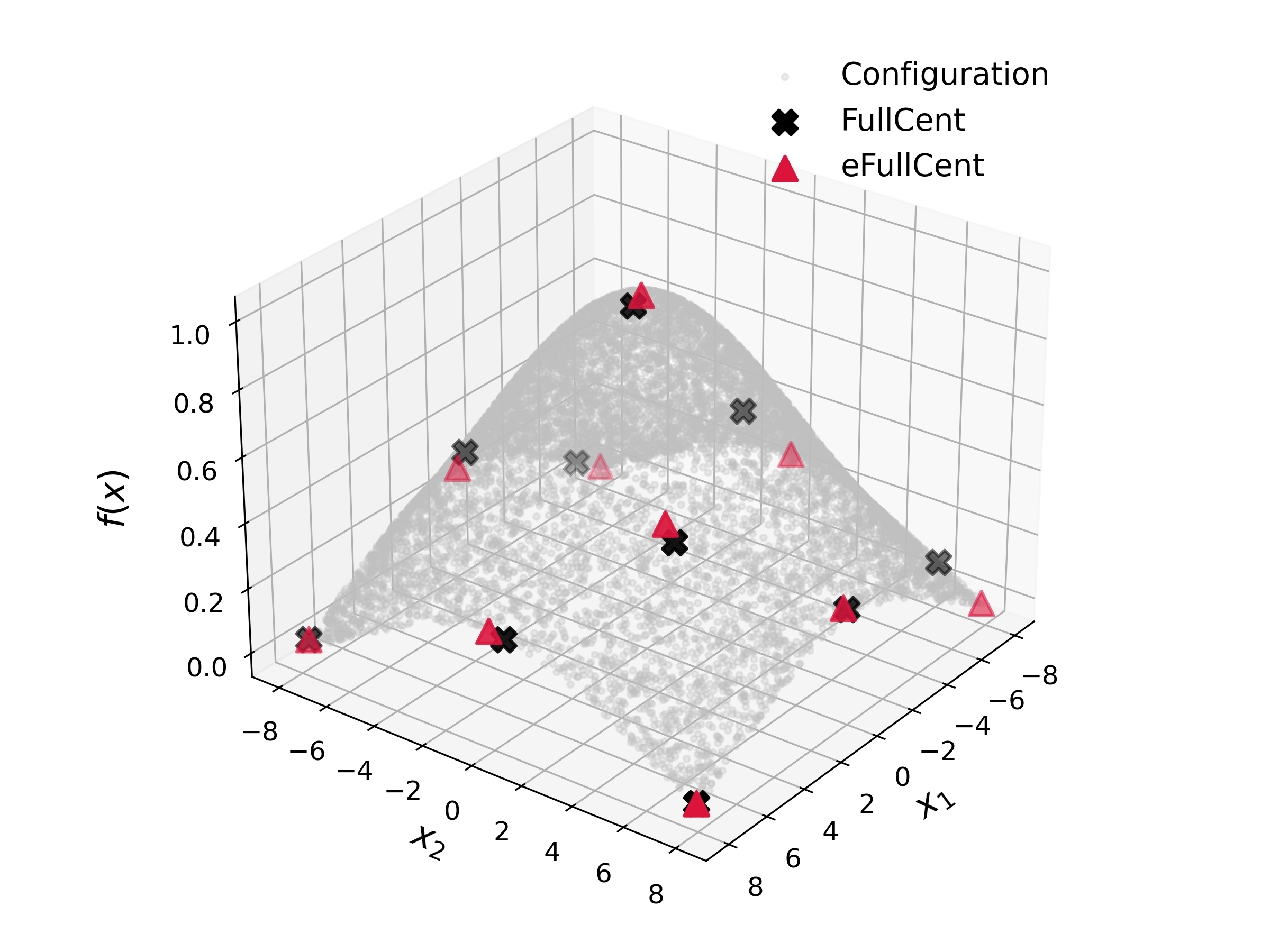}
        \caption{Off-center peak}
        \label{fig:offcentre}
    \end{subfigure}\hfill
    \begin{subfigure}[b]{0.32\textwidth}
        \centering
        \includegraphics[width=\linewidth,trim={40.080bp 18.000bp 92.640bp 23.760bp},clip]{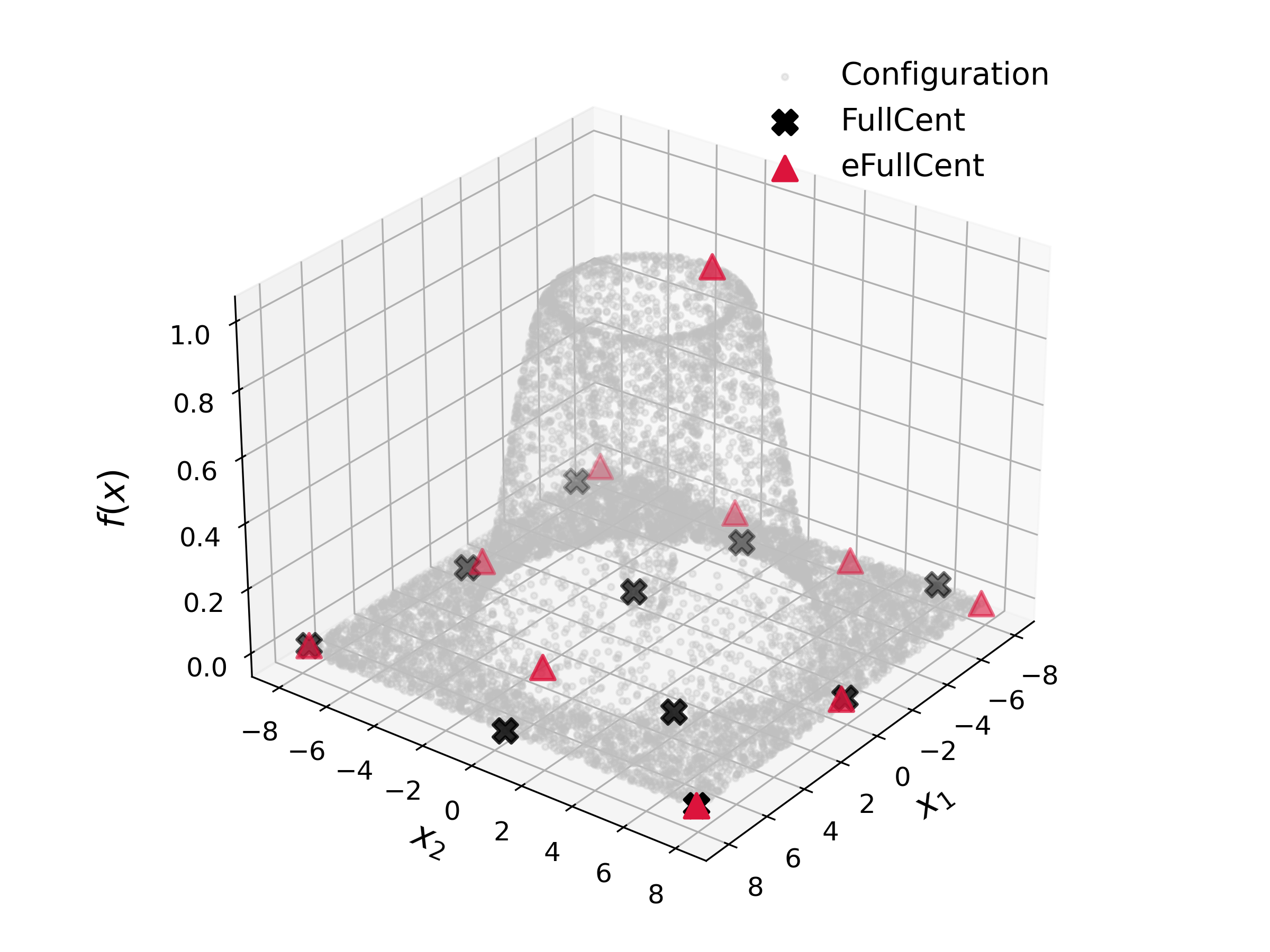}
        \caption{Ring cosine}
        \label{fig:ring}
    \end{subfigure}
    \caption{Qualitative comparison on three synthetic landscapes. Triangles (\eFullCent{}) and crosses (\FullCent{}) indicate the centers selected by the respective algorithms. The best center selected by \eFullCent{} consistently attains a higher value than the best center selected by \FullCent{}.}
    \label{fig:synthetic}
\end{figure}

\section[Synthetic-Landscape Comparison of \eFullCent{} and Baselines]{Synthetic-Landscape Comparison:\\\eFullCent{} and Baselines}
\label{app:synthetic-baselines}

We benchmarked \eFullCent{}
against three established baselines: Random Search, SMAC, and BOHB, 
on three complex and difficult two-dimensional test functions.
All functions are defined on the square configuration domain
\(\mathcal D=[-10,10]^2\subset\mathbb R^2\).
For every experiment we drew \(n=10000\) configurations
\(\mathbf x_i\sim\mathcal U(\mathcal D)\) once and reused them across
methods to ensure strict comparability.

\begin{figure}[t]
  \centering
    \begin{subfigure}[b]{0.32\textwidth}
        \centering
        \includegraphics[width=\linewidth,trim={42.720bp 64.320bp 56.160bp 75.120bp},clip]{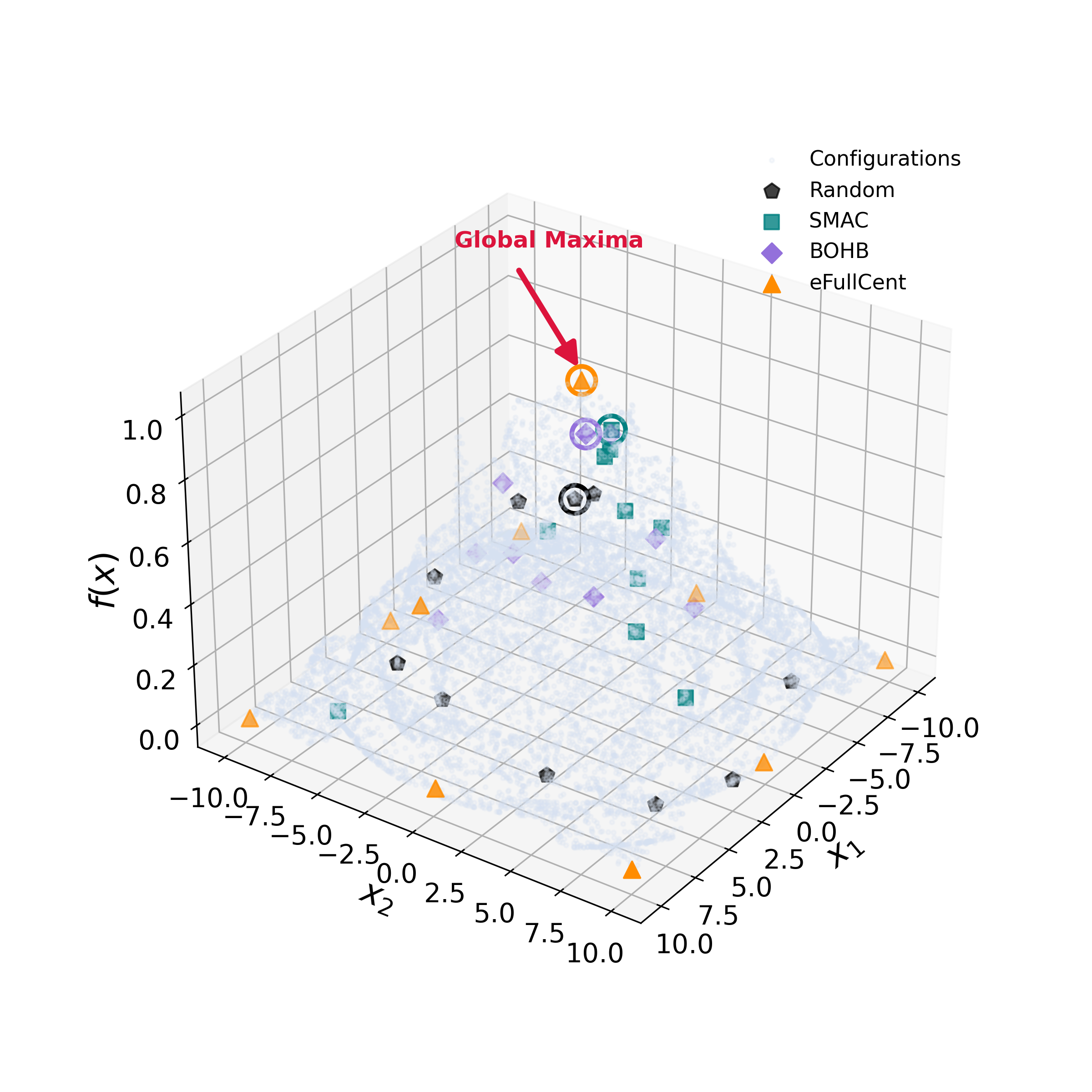}
        \caption{Radial ripples}
        \label{fig:bl-radial}
    \end{subfigure}\hfill
    \begin{subfigure}[b]{0.32\textwidth}
        \centering
        \includegraphics[width=\linewidth,trim={42.720bp 64.320bp 56.160bp 75.120bp},clip]{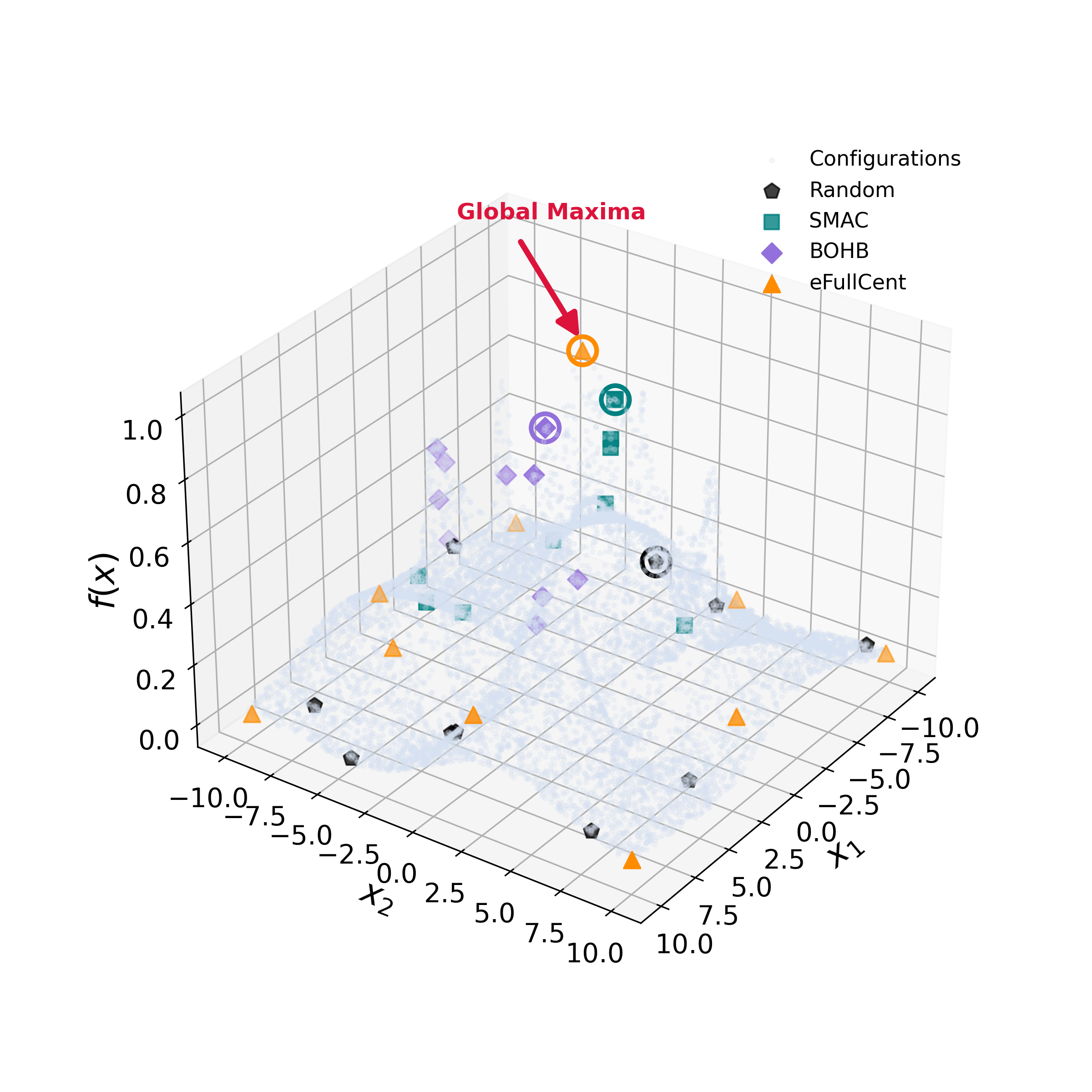}
        \caption{Double rings}
        \label{fig:bl-offcentre}
    \end{subfigure}\hfill
    \begin{subfigure}[b]{0.32\textwidth}
        \centering
        \includegraphics[width=\linewidth,trim={42.720bp 64.320bp 56.160bp 75.120bp},clip]{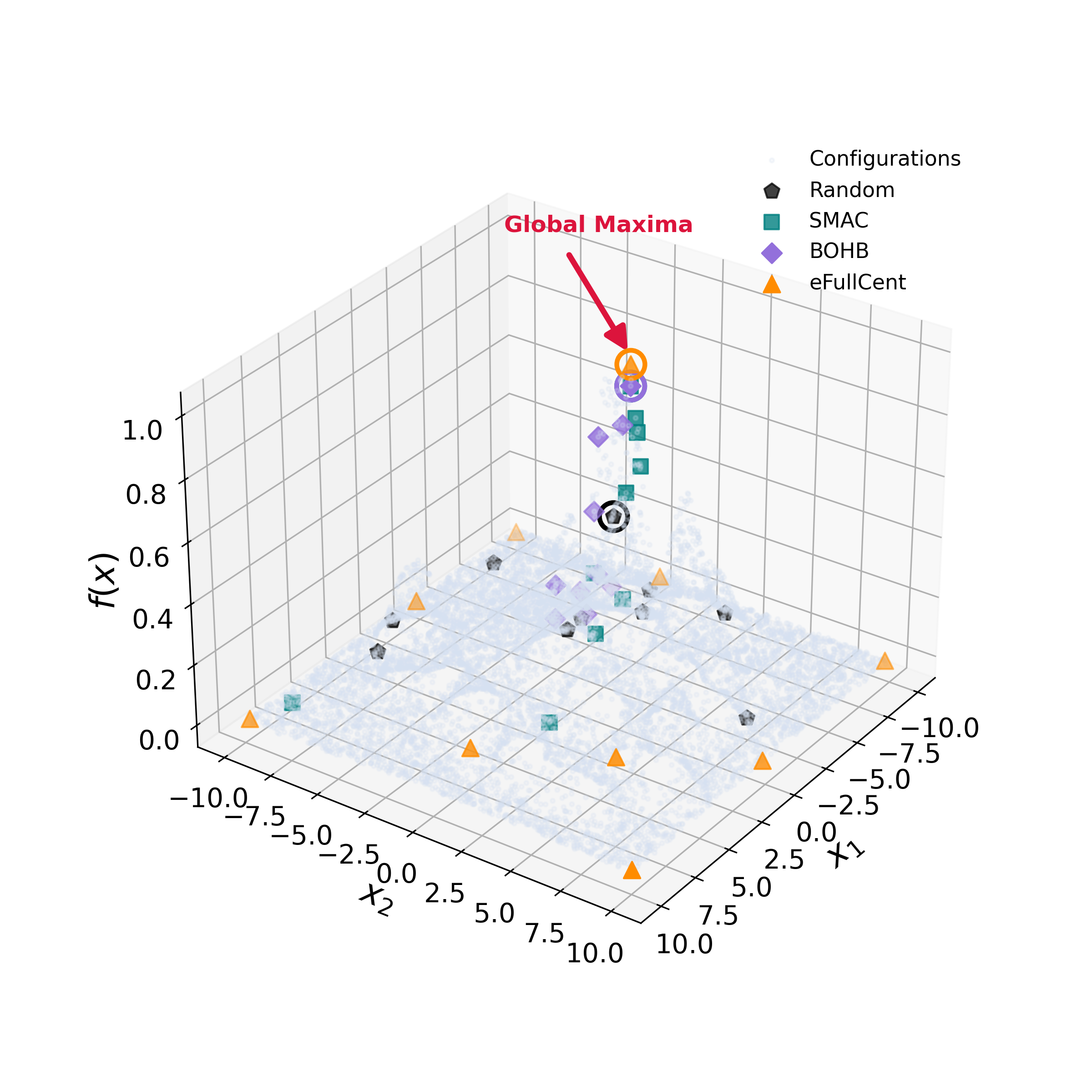}
        \caption{Multi-modal bumps}
        \label{fig:bl-ring}
    \end{subfigure}
  \caption{Visualization of the centers selected by the four algorithms on the benchmark landscapes. Each marker shape represents one algorithm: pentagons (Random), squares (SMAC), diamonds (BOHB), and triangles (\eFullCent{}). Hollow markers indicate each algorithm’s best-found solution, and the red arrow highlights the true global maximum. Notably, \eFullCent{} consistently identifies the global maximum, demonstrating superior performance over the baselines.}
  \label{fig:synth_landscape}
\end{figure}

\paragraph{Algorithm Configurations}
In our simulations
we used \(T=1\), under which \eAdaCent{} and \eFullCent{} algorithms coincide and are therefore identical, and \(\eps = 0.2\).
We also use a total budget of \(B=10\) for value probing
so that \(f(\mathbf x)=A(\mathbf x,1)\) is observed exactly ten times per
method and surface.
BOHB was run with
\(B_{\max}=27\) and  \(\eta=3\);
SMAC used its default intensification loop and a
100-tree random-forest surrogate.
Both BOHB and SMAC were executed five times with different random
seeds; for these synthetic illustrations we report the best run per
method, so the traces depict each method's search behavior rather than a
mean-and-variance comparison. Quantitative claims rest on the main
$250{+}$-setting study with $30$ seeds per task.
Although the candidate points were randomly scattered and not arranged
in clusters, \eFullCent{} still achieved the highest maxima on every surface,
demonstrating robustness to unstructured search spaces.
The resulting plots are visualized in
Figure~\ref{fig:synth_landscape}.

\paragraph{Analytical Surfaces}
We reproduce the closed-form definitions for completeness.
Let \(r=\sqrt{x^{2}+y^{2}}\) and
\(\theta=\operatorname{\arctan}\left(\frac{y}{x}\right)\).

\begin{itemize}
    \item [(a)] \textbf{Radial ripples} \[\begin{aligned}
f_{a}(x,y) &=
  \frac{1}{2}\bigl(\sin 3r + 1\bigr)
  e^{-\tfrac{r^{2}}{50}}
+ \sum_{i=1}^{150}
    h_{i}\exp\Bigl(
      -\tfrac{(x-c_{x,i})^{2} + (y-c_{y,i})^{2}}{2s_{i}^{2}}
    \Bigr).
\end{aligned}\]
\item [(b)] \textbf{Double rings}\[\begin{aligned}
f_{b}(x,y) &=
  0.5e^{-(r-3)^{2}/(2\cdot0.18^{2})}
  + 0.4e^{-(r-6)^{2}/(2\cdot0.25^{2})}
  \nonumber\\
  &\quad + 0.3\bigl(\sin 4\theta + 1\bigr)e^{-r^{2}/90} \quad \text{with } r \in \{3, 6\}.
\end{aligned}\]
\item [(c)] \textbf{Multi-modal bumps} \[\begin{aligned}
f_{c}(x,y) &=
  0.4e^{-(x^{2}+y^{2})/(2\cdot4.5^{2})}
+ \sum_{j=1}^{30}
    h_{j}\exp\Bigl(
      -\tfrac{(x-c_{x,j})^{2} + (y-c_{y,j})^{2}}{2s_{j}^{2}}
    \Bigr).
\end{aligned}\]
\end{itemize}
All heights \(h_{\cdot}\in[0.03,1.0]\), centers
\((c_{x},c_{y})\in[-8,8]^{2}\), and
scales \(s_{\cdot}\in[0.15,0.8]\) were sampled once and fixed
throughout the study.

\setlength{\parskip}{2pt}
\setlength{\intextsep}{6pt}
\DisplaySpace{210pt}
\section{Empirical Validation and Additional Analyses}
\label{app:diminishing-returns-validation}

Our adaptive strategies rely on the diminishing-returns property described in
Assumption~\ref{assump:concavity}, which states that the improvement obtained
from increasing the fidelity decreases as the fidelity level grows. Although
this assumption is used for the theoretical analysis, we additionally examine
its empirical validity on the benchmark tasks considered in our experiments.

Specifically, we measure the fraction of local fidelity triplets
\((b-1,b,b+1)\) that satisfy
\[
A(x,b+1)-A(x,b)
\leq
A(x,b)-A(x,b-1)+\mathrm{tol},
\]
where \(A(x,b)\) denotes the performance of configuration \(x\) at fidelity
level \(b\). The theoretical assumption corresponds to
\(\mathrm{tol}=0\). To account for small numerical fluctuations that may
appear in the benchmark performance curves, we additionally consider a small
data-independent tolerance of \(\mathrm{tol}=0.01\).

We sample 100 tasks from the three YAHPO benchmark suites proportionally to
their suite sizes and evaluate 64 randomly selected configurations per task
across all available fidelity levels. The percentage of fidelity triplets
satisfying the diminishing-returns condition is summarized in
Table~\ref{tab:diminishing-returns-validation}.
These results demonstrate that the diminishing-returns property holds
strongly across the considered benchmarks, providing empirical support for
the practical validity of Assumption~\ref{assump:concavity}. Note, however,
that the exact zero-tolerance condition of Assumption~\ref{assump:concavity}
is sensitive to tiny numerical fluctuations in the surrogate curves, so the
certified upper bound need not hold for every triplet; the operational
diagnostics below measure directly whether the implemented bound is safe in
the cases that matter for pruning.

\begin{table}[H]
\centering
\normalsize
\caption{Percentage of local fidelity triplets satisfying the
diminishing-returns condition with \(\mathrm{tol}=0.01\).}
\label{tab:diminishing-returns-validation}
\renewcommand{\arraystretch}{1.05}
\setlength{\tabcolsep}{12pt}
\begin{tabular}{@{}lc@{}}
\toprule
Benchmark suite & Satisfied triplets (\%) \\
\midrule
\texttt{lcbench} & 88.58 \\
\texttt{rbv2\_aknn} & 99.13 \\
\texttt{rbv2\_rpart} & 99.74 \\
\midrule
Weighted Average & 98.02 \\
\bottomrule
\end{tabular}
\end{table}

\subsection{Pruning safety diagnostics}
\label{app:pruning-diagnostics}

The safe-pruning argument is conditional on
Assumption~\ref{assump:concavity}, which real learning curves only
approximately satisfy (Table~\ref{tab:diminishing-returns-validation}). We
therefore measure two operational quantities on the benchmark suites.
\emph{Terminal-UB coverage} is the fraction of upper-bound checks in which
the extrapolated bound actually covers the eventual terminal value; coverage
below $100\%$ means the bound underestimates the final value for some
configurations. \emph{False pruning} is the fraction of actual pruning events
that remove the configuration which would eventually have been the best
(the ``eventual best''); its denominator is the set of pruning events, not
all configurations.

The implemented upper bound extrapolates a two-point secant: for a history of
length $b$, it takes the points at indices $\lceil0.3\,b\rceil$ and $b$ and
linearly extrapolates their secant to $T$. Under exact diminishing returns
this is a valid (looser) upper bound; the diagnostics quantify how well it
behaves on real curves.

\begin{table}[H]
\centering
\normalsize
\caption{Terminal upper-bound coverage (\%) and false-pruning / eventual-best removal rate (\%).}
\label{tab:pruning-diagnostics}
\renewcommand{\arraystretch}{1.05}
\setlength{\tabcolsep}{9pt}
\begin{tabular}{lcc}
\toprule
\textbf{Suite} & \textbf{\AdaCent{} coverage / pruning} & \textbf{\eAdaCent{} coverage / pruning} \\
\midrule
\texttt{LCBench}      & 94.50 / 0.83 & 96.90 / 0.18 \\
\texttt{RBV2-RPart}   & 92.26 / 0.44 & 91.75 / 0.36 \\
\texttt{RBV2-AKNN}    & 93.87 / 2.83 & 91.90 / 2.34 \\
\bottomrule
\end{tabular}
\end{table}

Coverage ranges from $91.75\%$ to $96.90\%$, and the eventual-best removal
rate is at most $2.83\%$, so pruning rarely discards the best configuration
even though the exact diminishing-returns assumption does not hold
everywhere. We therefore present the pruning guarantee as conditional on
Assumption~\ref{assump:concavity} and accompanied by these empirical rates.

\setlength{\parskip}{6pt}
\setlength{\intextsep}{10pt plus 2pt minus 2pt}

\subsection{Completion-budget restriction and its empirical cost}
\label{app:completion-reserve}

We re-ran \AdaCent{} and \eAdaCent{} with the completion-budget restriction of Algorithm~\ref{alg:adacent} explicitly enforced, and compared the resulting per-task mean curves against the unrestricted scheduler at every budget point. We report two statistics: the relative mean absolute error (MAE), i.e.\ the mean absolute pointwise difference normalized by the mean of the unrestricted curve, and the absolute relative difference at the terminal budget point (final $\Delta\%$).

\begin{table}[H]
\centering
\fontsize{10}{11}\selectfont
\caption{Effect of explicitly enforcing the completion-budget restriction. Differences are computed per task and averaged over the suite and 30 seeds.}
\label{tab:completion-reserve}
\renewcommand{\arraystretch}{1.02}
\scalebox{0.96}{%
\begin{tabular}{lcc}
\toprule
\textbf{Suite} & \textbf{Final $\Delta\%$ (mean / median)} & \textbf{Relative MAE (mean / median)} \\
\midrule
\texttt{lcbench}      & 0.116 / 0.088 & 0.502 / 0.428 \\
\texttt{rbv2\_rpart}  & 0.270 / 0.091 & 0.794 / 0.428 \\
\texttt{rbv2\_aknn}   & 0.110 / 0.031 & 0.454 / 0.226 \\
\bottomrule
\end{tabular}}
\end{table}

The restriction leaves the trajectories essentially unchanged (final $\Delta<0.3\%$ on average), so the unrestricted runs reported in the main text are representative of the scheduler with the completion-budget restriction.

\subsection{Attained-value ratios}
\label{app:attained-ratios}

To give the approximation guarantees a concrete empirical interpretation, Table~\ref{tab:attained-ratios} reports the final attained value divided by the benchmark reference optimum, together with the median of the smoothness quantity $1-\epsilon r$ (computed from the $95$th-percentile pairwise smoothness estimate and the greedy covering radius $r$).

\begin{table}[H]
\centering
\fontsize{10}{11}\selectfont
\caption{Attained value / reference optimum (mean / median) and the median smoothness-based factor $1-\epsilon r$.}
\label{tab:attained-ratios}
\renewcommand{\arraystretch}{1.05}
\setlength{\tabcolsep}{10pt}
\scalebox{0.96}{%
\begin{tabular}{lccc}
\toprule
\textbf{Suite} & \textbf{\AdaCent{}} & \textbf{\eAdaCent{}} & \textbf{Median $1-\epsilon r$} \\
\midrule
\texttt{LCBench}       & 0.9661 / 0.9713 & 0.9715 / 0.9799 & 0.4464 \\
\texttt{RBV2-RPart}    & 0.9677 / 0.9818 & 0.9669 / 0.9816 & 0.9668 \\
\texttt{RBV2-AKNN}     & 0.9561 / 0.9715 & 0.9586 / 0.9732 & 0.9193 \\
\bottomrule
\end{tabular}}
\end{table}

\subsection{Hyperparameter sensitivity}
\label{app:sensitivity}

Table~\ref{tab:sensitivity} reports one-at-a-time sensitivity of the final performance to the batch size $m$, the exploration fraction $\delta$, and the smoothness parameter $\epsilon$ (shown relative to its default $\epsilon_0$). $A$ denotes \AdaCent{} and $E$ denotes \eAdaCent{}; RPart and AKNN are the \texttt{rbv2\_rpart} and \texttt{rbv2\_aknn} suites, and $*$ marks the defaults.

\begin{table}[H]
\centering
\fontsize{10}{11}\selectfont
\caption{One-at-a-time hyperparameter sensitivity (final validation accuracy, \%).}
\label{tab:sensitivity}
\renewcommand{\arraystretch}{1.02}
\setlength{\tabcolsep}{8pt}
\scalebox{0.96}{%
\begin{tabular}{lllccc}
\toprule
\textbf{Method} & \textbf{Param} & \textbf{Value} & \textbf{LCBench} & \textbf{RPart} & \textbf{AKNN} \\
\midrule
A & $m$ & 10 & 94.1 & 82.3 & 81.3 \\
A & $m$ & 25* & 94.2 & 82.3 & 81.2 \\
A & $m$ & 40 & 94.1 & 82.2 & 81.1 \\
E & $m$ & 10 & 94.4 & 82.6 & 81.7 \\
E & $m$ & 25* & 94.7 & 82.4 & 81.4 \\
E & $m$ & 40 & 94.7 & 82.0 & 81.0 \\
E & $\delta$ & 0.05 & 94.7 & 82.5 & 81.6 \\
E & $\delta$ & 0.10* & 94.7 & 82.4 & 81.4 \\
E & $\delta$ & 0.20 & 94.5 & 82.3 & 81.3 \\
E & $\epsilon/\epsilon_0$ & 0.5 & 94.6 & 82.2 & 81.3 \\
E & $\epsilon/\epsilon_0$ & 1.0* & 94.7 & 82.4 & 81.4 \\
E & $\epsilon/\epsilon_0$ & 2.0 & 94.2 & 82.3 & 81.4 \\
\bottomrule
\end{tabular}}
\end{table}

\section[Comparison with Multi-Fidelity Bayesian Optimization Methods]{Comparison with Multi-Fidelity Bayesian Optimization\\Methods}
\label{app:mfbo-comparison}

UVP and Multi-Fidelity Bayesian Optimization (MFBO) both consider the
problem of allocating a limited evaluation budget across configurations
and fidelity levels. However, they model the optimization problem in
different ways. In MFBO, each configuration--fidelity pair is typically
associated with a scalar objective value, and a probabilistic surrogate
model is used to select the next evaluation through an acquisition
function. In contrast, in UVP, each configuration is associated with an
unknown value function, such as a training curve, which is gradually
revealed as more budget is allocated. Therefore, UVP does not only decide
which configuration to evaluate, but also how much budget should be spent
on that configuration and whether the observed partial curve justifies
continuing the evaluation. This makes early stopping and adaptive budget
allocation natural parts of the problem formulation. In addition, UVP
leverages geometric similarity between configurations and provides
approximation guarantees based on this structure, offering a different
perspective from surrogate-based MFBO approaches.

For completeness, we compare \AdaCent{} and \eAdaCent{} against three
representative MFBO methods: robust multi-fidelity Bayesian optimization
(rMFBO)~\cite{zhang2025rmfbo}, multi-fidelity max-value entropy search
(MES)~\cite{takeno2020mfmes}, and multi-fidelity knowledge
gradient (KG)~\cite{wu2020practicalmfbo}. Following the evaluation protocol
of~\cite{zhang2025rmfbo}, we run the comparison on task \(3945\) from
\texttt{lcbench}, task \(1489\) from \texttt{iaml\_rpart}, and task \(1489\)
from \texttt{iaml\_xgboost}. General MFBO lies outside UVP's
sequential-oracle formulation, because its fidelity queries can access
distinct approximations rather than reveal a single monotone curve; we
therefore keep this comparison supplementary and restrict it to HPO tasks
that expose a compatible learning curve. All methods
are evaluated on the same benchmark instances and search domains, use two
initial observations, and receive the same fidelity grid and total
evaluation budget. The total budgets are (B=300) for \texttt{lcbench} and (B=50) for
the \texttt{iaml} tasks. For the MFBO baselines, we use the default
hyperparameter settings and the implementation provided in the public codebase
accompanying~\cite{zhang2025rmfbo}.
Table~\ref{tab:mfbo-results} reports the mean and standard error of the final
incumbent accuracy over 10 independent runs on each task.

\paragraph{Results.}
As shown in Table~\ref{tab:mfbo-results}, \eAdaCent{} achieves the best mean final performance on \texttt{lcbench} and \texttt{iaml\_rpart}, while \AdaCent{} obtains the second-best mean on both tasks. On \texttt{iaml\_xgboost}, \AdaCent{} and \eAdaCent{} both outperform rMFBO, MES, and KG. Both proposed methods show
smaller deviations across independent runs, indicating more stable
behavior compared with several MFBO baselines. The improvement of
\AdaCent{} over the MFBO methods also shows that the performance gain is
not only due to the enhanced-distance mechanism introduced in
\eAdaCent{}, but also comes from the geometric coverage and adaptive
budget allocation strategy of \AdaCent{}.

\begin{table}[H]
\centering
\caption{Final incumbent accuracy (\%) under the MFBO comparison protocol.}
\label{tab:mfbo-results}
\renewcommand{\arraystretch}{1.15}
\setlength{\tabcolsep}{6pt}
\begin{tabular}{@{}lccc@{}}
\toprule
\textbf{Method} & \textbf{LCBench} & \textbf{\texttt{iaml\_rpart}} & \textbf{\texttt{iaml\_xgboost}} \\
& Task 3945 & Task 1489 & Task 1489 \\
& Mean $\pm$ SE & Mean $\pm$ SE & Mean $\pm$ SD \\
\midrule
\AdaCent{}  & $73.6\pm0.4$ & $85.0\pm0.5$ & $89.2\pm1.7$ \\
\eAdaCent{} & $73.8\pm0.4$ & $85.4\pm0.1$ & $89.2\pm1.5$ \\
rMFBO       & $73.3\pm0.5$ & $81.5\pm1.0$ & $87.5\pm1.7$ \\
MES         & $72.9\pm0.4$ & $76.0\pm8.5$ & $87.1\pm1.4$ \\
Random      & $72.6\pm0.7$ & $76.6\pm1.1$ & $86.4\pm1.3$ \\
KG          & $73.0\pm0.3$ & $37.8\pm8.6$ & $46.1\pm40.9$ \\
\bottomrule
\end{tabular}
\end{table}

\clearpage
%
\section{Benchmark Details and Complete Results}
\label{sec:figs}

This section collects benchmark specifications, complete smoothness diagnostics, and the full accuracy-curve gallery. The completion-budget, attained-value, and sensitivity analyses appear in Appendices~\ref{app:completion-reserve}, \ref{app:attained-ratios}, and \ref{app:sensitivity}, respectively. Table~\ref{tab:scenario_detailed} summarizes the evaluation scenarios used for our algorithms \AdaCent{} and \eAdaCent{} and for Assumption~\ref{assump:assump2}. Table~\ref{tab:lcbench-epsilon} reports percentile values of $\epsilon_{ij} r$ from Assumption~\ref{assump:assump2} across all \texttt{lcbench} tasks, while Tables~\ref{tab:rpart-epsilon} and \ref{tab:aknn-epsilon} provide the corresponding $\epsilon_{ij} r$ percentiles for the \texttt{rbv2\_rpart} and \texttt{rbv2\_aknn} datasets, respectively. Figure~\ref{fig:grid:lcbench} presents validation accuracies for all \texttt{lcbench} tasks using YAHPO Gym surrogates and compares multiple algorithms; the same analysis is shown for \texttt{rbv2\_rpart} in Figures~\ref{fig:grid:rpart1}--\ref{fig:grid:rpart3} and for \texttt{rbv2\_aknn} in Figures~\ref{fig:grid:aknn1}--\ref{fig:grid:aknn4}.

\subsection{Benchmark scenarios}

\begin{table}[H]
\centering
\normalsize
\caption{Comparison of scenarios, including hyperparameter search spaces, dataset coverage, and fidelity parameters.}
\label{tab:scenario_detailed}
\renewcommand{\arraystretch}{1.2}
\setlength{\tabcolsep}{4pt}
\begin{tabular}{@{}>{\raggedright\arraybackslash}p{1.9cm}>{\raggedright\arraybackslash}p{3.0cm}>{\raggedright\arraybackslash}p{1.6cm}>{\raggedright\arraybackslash}p{2.6cm}>{\raggedright\arraybackslash}p{3.8cm}@{}}
\toprule
\textbf{Scenario} & \textbf{Hyperparameter} & \textbf{Type} & \textbf{Range / Values} & \textbf{Notes} \\
\midrule

\multirow{9}{*}{\texttt{lcbench}} 
& \texttt{batch\_size} & int (log) & $[16, 512]$ & Training batch size \\
& \texttt{learning\_rate} & float (log) & $[10^{-4}, 0.1]$ & Step size \\
& \texttt{momentum} & float & $[0.1, 0.9]$ & SGD momentum \\
& \texttt{weight\_decay} & float & $[10^{-5}, 0.1]$ & Regularization \\
& \texttt{num\_layers} & int & $[1, 5]$ & Depth of network \\
& \texttt{max\_units} & int (log) & $[64, 1024]$ & Hidden layer width \\
& \texttt{max\_dropout} & float & $[0, 1]$ & Dropout rate \\
\cmidrule(l){2-5}
& \multicolumn{4}{l}{\textit{Fidelity Parameter:} Number of training epochs} \\
\cmidrule(l){2-5}
& \multicolumn{4}{l}{\textit{Datasets:} 35 OpenML classification tasks} \\
\midrule

\multirow{6}{*}{\texttt{rbv2\_rpart}} 
& \texttt{cp} & float (log) & $[0.001, 1]$ & Complexity parameter \\
& \texttt{maxdepth} & int & $[1, 30]$ & Max tree depth \\
& \texttt{minbucket} & int & $[1, 100]$ & Min terminal samples \\
& \texttt{minsplit} & int & $[1, 100]$ & Min split samples \\
\cmidrule(l){2-5}
& \multicolumn{4}{l}{\textit{Fidelity Parameter:} Fraction of training dataset} \\
\cmidrule(l){2-5}
& \multicolumn{4}{l}{\textit{Datasets:} 117 classification tasks} \\
\midrule

\multirow{6}{*}{\texttt{rbv2\_aknn}} 
& \texttt{k} & int & $[1, 50]$ & Number of neighbors \\
& \texttt{M} & int & $[18, 50]$ & Candidate neighbors \\
& \texttt{ef} & int (log) & $[7, 403]$ & Search parameter \\
& \texttt{ef\_construction} & int (log) & $[7, 403]$ & Index building param \\
\cmidrule(l){2-5}
& \multicolumn{4}{l}{\textit{Fidelity Parameter:} Fraction of training dataset} \\
\cmidrule(l){2-5}
& \multicolumn{4}{l}{\textit{Datasets:} 118 classification tasks} \\
\bottomrule
\end{tabular}
\end{table}

\clearpage
\subsection{Smoothness diagnostic protocol}
The following tables report the $\alpha$-percentiles of $\epsilon_{ij}r$ for $\alpha\in\{90,95,98,99\}$, using the pairwise smoothness ratios and greedy covering radius defined in Section~\ref{sec:val_assump}. These percentiles summarize the empirical distribution; the theoretical guarantees remain conditional on a smoothness constant satisfying Assumption~\ref{assump:assump2} globally.

\subsubsection{LCBench smoothness diagnostics}

\begin{table}[H]
\centering
\normalsize
\caption{\(\alpha\)-percentiles of \(\epsilon \cdot r\) from \texttt{lcbench} tabular data.}
\label{tab:lcbench-epsilon}
\small
\setlength{\tabcolsep}{6pt}
\renewcommand{\arraystretch}{1.08}
\begin{tabular}{@{}l *{4}{S[table-format=1.4]}@{}}
\toprule
\textbf{OpenML Task ID} & {\textbf{$\alpha = 90$}} & {\textbf{$\alpha = 95$}} & {\textbf{$\alpha = 98$}} & {\textbf{$\alpha = 99$}} \\
\midrule
3945 & 0.4051 & 0.5379 & 0.7023 & 0.8220 \\
7593 & 0.6003 & 0.7281 & 0.8853 & 1.0083 \\
34539 & 0.5739 & 0.7018 & 0.8716 & 1.0060 \\
126025 & 0.2938 & 0.4527 & 0.5863 & 0.6845 \\
126026 & 0.3273 & 0.4741 & 0.5987 & 0.6912 \\
126029 & 0.4211 & 0.5878 & 0.7457 & 0.8655 \\
146212 & 0.7347 & 0.8707 & 1.0503 & 1.1970 \\
167083 & 0.0374 & 0.0438 & 0.0522 & 0.0590 \\
167104 & 0.3948 & 0.4629 & 0.5512 & 0.6199 \\
167149 & 0.3578 & 0.4164 & 0.4911 & 0.5467 \\
167152 & 0.7890 & 0.9051 & 1.0615 & 1.1869 \\
167161 & 0.3689 & 0.4447 & 0.5422 & 0.6195 \\
167168 & 0.4730 & 0.5536 & 0.6552 & 0.7337 \\
167181 & 0.6253 & 0.7555 & 0.9239 & 1.0617 \\
167184 & 0.4524 & 0.5617 & 0.6972 & 0.8058 \\
167185 & 0.7263 & 0.8283 & 0.9576 & 1.0582 \\
167190 & 0.3707 & 0.4760 & 0.5949 & 0.6853 \\
167200 & 0.1961 & 0.2300 & 0.2736 & 0.3069 \\
167201 & 0.5263 & 0.6494 & 0.8106 & 0.9397 \\
168329 & 0.8584 & 0.9864 & 1.1633 & 1.3061 \\
168330 & 0.4074 & 0.5253 & 0.6794 & 0.7971 \\
168331 & 0.5892 & 0.6956 & 0.8342 & 0.9453 \\
168335 & 0.3712 & 0.5056 & 0.6365 & 0.7348 \\
168868 & 0.2265 & 0.5595 & 0.8054 & 0.9452 \\
168908 & 0.1974 & 0.2349 & 0.2781 & 0.3103 \\
168910 & 0.6162 & 0.7068 & 0.8216 & 0.9105 \\
189354 & 0.2051 & 0.2419 & 0.2896 & 0.3289 \\
189862 & 0.3253 & 0.3818 & 0.4521 & 0.5051 \\
189865 & 0.3317 & 0.3858 & 0.4529 & 0.5042 \\
189866 & 0.1681 & 0.1958 & 0.2310 & 0.2585 \\
189873 & 0.9417 & 1.0798 & 1.2738 & 1.4320 \\
189905 & 0.7343 & 0.8618 & 1.0343 & 1.1728 \\
189906 & 0.6518 & 0.7592 & 0.8997 & 1.0136 \\
189908 & 0.4438 & 0.5694 & 0.6904 & 0.7772 \\
189909 & 0.3621 & 0.5278 & 0.6848 & 0.8003 \\
\bottomrule
\end{tabular}
\end{table}
\clearpage
\subsubsection{RPart smoothness diagnostics}

\begin{table}[H]
\centering
\normalsize
\caption{$\alpha$-percentiles of $\epsilon \cdot r$ computed from \texttt{rbv2\_rpart} tabular data.}
\label{tab:rpart-epsilon}
\small
\setlength{\tabcolsep}{4.5pt}
\renewcommand{\arraystretch}{1.08}
\begin{tabular}{
l S[table-format=1.4] S[table-format=1.4] S[table-format=1.4] S[table-format=1.4]
@{\quad\vrule\quad}
l S[table-format=1.4] S[table-format=1.4] S[table-format=1.4] S[table-format=1.4]
}
\toprule
\textbf{Task ID} &
{\textbf{\(\alpha = 90\)}} &
{\textbf{\(\alpha = 95\)}} &
{\textbf{\(\alpha = 98\)}} &
{\textbf{\(\alpha = 99\)}} &
\textbf{Task ID} &
{\textbf{\(\alpha = 90\)}} &
{\textbf{\(\alpha = 95\)}} &
{\textbf{\(\alpha = 98\)}} &
{\textbf{\(\alpha = 99\)}} \\
\midrule
3 & 0.0182 & 0.0307 & 0.0468 & 0.0508 & 11 & 0.0280 & 0.0402 & 0.0587 & 0.0801 \\
12 & 0.0713 & 0.1028 & 0.1853 & 0.2241 & 14 & 0.0846 & 0.1043 & 0.1556 & 0.2162 \\
15 & 0.0578 & 0.0901 & 0.1450 & 0.1748 & 16 & 0.0879 & 0.1110 & 0.1581 & 0.2064 \\
18 & 0.1088 & 0.1478 & 0.2079 & 0.2549 & 22 & 0.0914 & 0.1270 & 0.1654 & 0.2299 \\
23 & 0.0451 & 0.0605 & 0.0824 & 0.1017 & 24 & 0.0128 & 0.0148 & 0.0233 & 0.0292 \\
28 & 0.0139 & 0.0217 & 0.0357 & 0.0449 & 29 & 0.0302 & 0.0407 & 0.0622 & 0.0835 \\
31 & 0.0273 & 0.0444 & 0.0612 & 0.0747 & 32 & 0.0312 & 0.0441 & 0.0768 & 0.0889 \\
37 & 0.0292 & 0.0407 & 0.0551 & 0.0674 & 38 & 0.0005 & 0.0010 & 0.0014 & 0.0017 \\
42 & 0.1573 & 0.2180 & 0.3129 & 0.3641 & 44 & 0.0146 & 0.0204 & 0.0311 & 0.0386 \\
46 & 0.0261 & 0.0363 & 0.0631 & 0.0697 & 50 & 0.1065 & 0.1416 & 0.1984 & 0.2652 \\
54 & 0.0831 & 0.1152 & 0.1637 & 0.1975 & 60 & 0.0102 & 0.0148 & 0.0258 & 0.0336 \\
181 & 0.0685 & 0.0932 & 0.1331 & 0.1665 & 182 & 0.0120 & 0.0187 & 0.0320 & 0.0384 \\
188 & 0.0576 & 0.0787 & 0.1073 & 0.1341 & 300 & 0.0311 & 0.0478 & 0.0795 & 0.1006 \\
307 & 0.1993 & 0.2527 & 0.3847 & 0.5045 & 312 & 0.0190 & 0.0290 & 0.0362 & 0.0441 \\
334 & 0.1361 & 0.1933 & 0.2562 & 0.3195 & 375 & 0.1074 & 0.1622 & 0.2023 & 0.2268 \\
377 & 0.1568 & 0.2124 & 0.3321 & 0.3641 & 458 & 0.0577 & 0.1222 & 0.1449 & 0.1569 \\
469 & 0.0521 & 0.0733 & 0.0967 & 0.1173 & 470 & 0.0273 & 0.0400 & 0.0551 & 0.0723 \\
1040 & 0.0067 & 0.0091 & 0.0155 & 0.0180 & 1049 & 0.0098 & 0.0150 & 0.0176 & 0.0195 \\
1050 & 0.0079 & 0.0117 & 0.0143 & 0.0158 & 1053 & 0.0075 & 0.0095 & 0.0156 & 0.0169 \\
1056 & 0.0005 & 0.0007 & 0.0012 & 0.0013 & 1063 & 0.0193 & 0.0265 & 0.0395 & 0.0513 \\
1067 & 0.0082 & 0.0115 & 0.0131 & 0.0154 & 1068 & 0.0064 & 0.0147 & 0.0299 & 0.0370 \\
1111 & 0.0000 & 0.0000 & 0.0000 & 0.0000 & 1220 & 0.0011 & 0.0014 & 0.0017 & 0.0021 \\
1457 & 0.1699 & 0.3120 & 0.4592 & 0.4650 & 1462 & 0.0505 & 0.0732 & 0.1096 & 0.1269 \\
1464 & 0.0233 & 0.0354 & 0.0491 & 0.0590 & 1468 & 0.1872 & 0.3751 & 0.4631 & 0.4796 \\
1475 & 0.0191 & 0.0245 & 0.0416 & 0.0433 & 1476 & 0.0491 & 0.0844 & 0.1131 & 0.1300 \\
1478 & 0.0090 & 0.0133 & 0.0202 & 0.0213 & 1479 & 0.0184 & 0.0251 & 0.0403 & 0.0487 \\
1480 & 0.0196 & 0.0311 & 0.0446 & 0.0535 & 1485 & 0.0216 & 0.0251 & 0.0351 & 0.0387 \\
1486 & 0.0028 & 0.0045 & 0.0067 & 0.0080 & 1487 & 0.0010 & 0.0011 & 0.0018 & 0.0020 \\
1489 & 0.0101 & 0.0146 & 0.0213 & 0.0235 & 1494 & 0.0239 & 0.0391 & 0.0505 & 0.0556 \\
1497 & 0.0256 & 0.0408 & 0.0559 & 0.0694 & 1501 & 0.0819 & 0.1185 & 0.2173 & 0.2600 \\
1510 & 0.0668 & 0.0951 & 0.1370 & 0.1729 & 1515 & 0.1507 & 0.2229 & 0.3526 & 0.3765 \\
4134 & 0.0253 & 0.0338 & 0.0476 & 0.0593 & 4154 & 0.0000 & 0.0000 & 0.0000 & 0.0000 \\
4534 & 0.0045 & 0.0062 & 0.0098 & 0.0106 & 4538 & 0.0222 & 0.0313 & 0.0472 & 0.0629 \\
4541 & 0.0097 & 0.0110 & 0.0117 & 0.0120 & 6332 & 0.0385 & 0.0529 & 0.0728 & 0.0841 \\
23381 & 0.0365 & 0.0514 & 0.0700 & 0.0829 & 40496 & 0.2194 & 0.2962 & 0.4198 & 0.5075 \\
40498 & 0.0247 & 0.0421 & 0.0507 & 0.0600 & 40499 & 0.0261 & 0.0406 & 0.0657 & 0.0842 \\
40536 & 0.0013 & 0.0022 & 0.0037 & 0.0045 & 40670 & 0.0141 & 0.0218 & 0.0369 & 0.0435 \\
40701 & 0.0025 & 0.0031 & 0.0050 & 0.0066 & 40900 & 0.0024 & 0.0037 & 0.0051 & 0.0055 \\
40966 & 0.1255 & 0.2467 & 0.3100 & 0.3314 & 40975 & 0.0183 & 0.0267 & 0.0382 & 0.0503 \\
40978 & 0.0133 & 0.0210 & 0.0359 & 0.0425 & 40979 & 0.0696 & 0.1073 & 0.2020 & 0.2176 \\
40981 & 0.0233 & 0.0306 & 0.0437 & 0.0563 & 40982 & 0.0370 & 0.0496 & 0.0651 & 0.0748 \\
40983 & 0.0076 & 0.0112 & 0.0158 & 0.0181 & 40984 & 0.0722 & 0.1257 & 0.1776 & 0.1779 \\
40994 & 0.0041 & 0.0085 & 0.0139 & 0.0176 & 41138 & 0.0005 & 0.0007 & 0.0009 & 0.0009 \\
41142 & 0.0067 & 0.0092 & 0.0105 & 0.0115 & 41143 & 0.0081 & 0.0134 & 0.0175 & 0.0193 \\
41146 & 0.0122 & 0.0189 & 0.0273 & 0.0286 & 41156 & 0.0119 & 0.0162 & 0.0202 & 0.0225 \\
41157 & 0.0354 & 0.0571 & 0.1054 & 0.1325 & 41159 & 0.0155 & 0.0204 & 0.0237 & 0.0248 \\
41161 & 0.0313 & 0.0332 & 0.0339 & 0.0342 & 41162 & 0.0002 & 0.0002 & 0.0002 & 0.0002 \\
41163 & 0.0264 & 0.0313 & 0.0455 & 0.0543 & 41164 & 0.0561 & 0.0678 & 0.0973 & 0.1097 \\
41165 & 0.0056 & 0.0060 & 0.0063 & 0.0064 & 41212 & 0.0211 & 0.0334 & 0.0417 & 0.0429 \\
41278 & 0.0033 & 0.0053 & 0.0064 & 0.0075 &  &  &  &  &  \\
\bottomrule
\end{tabular}
\end{table}
\clearpage
\subsubsection{AKNN smoothness diagnostics}

\begin{table}[H]
\centering
\normalsize
\caption{$\alpha$-percentiles of $\epsilon \cdot r$ computed from \texttt{rbv2\_aknn} tabular data.}
\label{tab:aknn-epsilon}
\small
\setlength{\tabcolsep}{4.5pt}
\renewcommand{\arraystretch}{1.08}
\begin{tabular}{
l S[table-format=1.4] S[table-format=1.4] S[table-format=1.4] S[table-format=1.4]
@{\quad\vrule\quad}
l S[table-format=1.4] S[table-format=1.4] S[table-format=1.4] S[table-format=1.4]
}
\toprule
\textbf{Task ID} &
{\textbf{\(\alpha = 90\)}} &
{\textbf{\(\alpha = 95\)}} &
{\textbf{\(\alpha = 98\)}} &
{\textbf{\(\alpha = 99\)}} &
\textbf{Task ID} &
{\textbf{\(\alpha = 90\)}} &
{\textbf{\(\alpha = 95\)}} &
{\textbf{\(\alpha = 98\)}} &
{\textbf{\(\alpha = 99\)}} \\
\midrule
3 & 0.0552 & 0.0622 & 0.0649 & 0.0671 & 11 & 0.0849 & 0.1074 & 0.1316 & 0.1564 \\
 12 & 0.2432 & 0.2581 & 0.2750 & 0.2851 & 14 & 0.2579 & 0.2775 & 0.2947 & 0.3039 \\
 15 & 0.1813 & 0.2529 & 0.3697 & 0.4199 & 16 & 0.2460 & 0.2617 & 0.3015 & 0.3143 \\
 18 & 0.2999 & 0.3413 & 0.3772 & 0.3992 & 22 & 0.2621 & 0.2852 & 0.3149 & 0.3356 \\
 23 & 0.1215 & 0.1467 & 0.1793 & 0.2096 & 24 & 0.0442 & 0.0466 & 0.0487 & 0.0495 \\
 28 & 0.0464 & 0.0526 & 0.0554 & 0.0557 & 29 & 0.0975 & 0.1188 & 0.1383 & 0.1530 \\
 31 & 0.0822 & 0.1120 & 0.1394 & 0.1646 & 32 & 0.0947 & 0.1013 & 0.1100 & 0.1106 \\
 37 & 0.0771 & 0.0908 & 0.1102 & 0.1235 & 38 & 0.0019 & 0.0024 & 0.0042 & 0.0049 \\
 42 & 0.4187 & 0.4641 & 0.5148 & 0.5547 & 44 & 0.0440 & 0.0485 & 0.0515 & 0.0517 \\
 46 & 0.0831 & 0.0902 & 0.0987 & 0.1014 & 50 & 0.3057 & 0.3514 & 0.4178 & 0.4671 \\
 54 & 0.2204 & 0.2507 & 0.2844 & 0.2995 & 60 & 0.0379 & 0.0504 & 0.0603 & 0.0663 \\
 181 & 0.1968 & 0.2147 & 0.2362 & 0.2571 & 182 & 0.0413 & 0.0456 & 0.0518 & 0.0526 \\
 188 & 0.1610 & 0.1874 & 0.2181 & 0.2380 & 300 & 0.1023 & 0.1173 & 0.1225 & 0.1289 \\
 307 & 0.5408 & 0.6249 & 0.7372 & 0.8437 & 312 & 0.0500 & 0.0572 & 0.0607 & 0.0642 \\
 334 & 0.4183 & 0.5105 & 0.6080 & 0.6721 & 375 & 0.3231 & 0.3638 & 0.4006 & 0.5224 \\
 377 & 0.4216 & 0.4686 & 0.5103 & 0.5503 & 458 & 0.1694 & 0.1856 & 0.2071 & 0.2308 \\
 469 & 0.1475 & 0.1766 & 0.2224 & 0.2556 & 470 & 0.0817 & 0.1054 & 0.1352 & 0.1535 \\
 1040 & 0.0204 & 0.0221 & 0.0235 & 0.0246 & 1049 & 0.0266 & 0.0385 & 0.0568 & 0.0751 \\
 1050 & 0.0259 & 0.0303 & 0.0408 & 0.0453 & 1053 & 0.0226 & 0.0313 & 0.0370 & 0.0444 \\
 1056 & 0.0017 & 0.0018 & 0.0020 & 0.0026 & 1063 & 0.0533 & 0.0688 & 0.0910 & 0.1165 \\
 1067 & 0.0220 & 0.0275 & 0.0320 & 0.0373 & 1068 & 0.0192 & 0.0480 & 0.0741 & 0.0948 \\
 1111 & 0.0000 & 0.0000 & 0.0000 & 0.0000 & 1220 & 0.0039 & 0.0043 & 0.0046 & 0.0048 \\
 1457 & 0.5027 & 0.5535 & 0.5887 & 0.6737 & 1462 & 0.1445 & 0.1926 & 0.2504 & 0.2783 \\
 1464 & 0.0679 & 0.0870 & 0.1079 & 0.1207 & 1468 & 0.5530 & 0.5690 & 0.6902 & 0.7864 \\
 1475 & 0.0589 & 0.0647 & 0.0696 & 0.0708 & 1476 & 0.1409 & 0.1531 & 0.1690 & 0.1740 \\
 1478 & 0.0233 & 0.0259 & 0.0269 & 0.0276 & 1479 & 0.0579 & 0.0742 & 0.0903 & 0.1012 \\
 1480 & 0.0586 & 0.0801 & 0.1026 & 0.1152 & 1485 & 0.0475 & 0.0570 & 0.0609 & 0.0679 \\
 1486 & 0.0094 & 0.0106 & 0.0112 & 0.0114 & 1487 & 0.0027 & 0.0038 & 0.0052 & 0.0053 \\
 1489 & 0.0259 & 0.0272 & 0.0289 & 0.0294 & 1494 & 0.0638 & 0.0668 & 0.0729 & 0.0770 \\
 1497 & 0.0879 & 0.0999 & 0.1088 & 0.1164 & 1501 & 0.2686 & 0.3027 & 0.3180 & 0.3431 \\
 1510 & 0.1973 & 0.2473 & 0.2974 & 0.3263 & 1515 & 0.5264 & 0.6307 & 0.7049 & 0.7462 \\
 4134 & 0.0672 & 0.0713 & 0.0761 & 0.0831 & 4154 & 0.0000 & 0.0000 & 0.0000 & 0.0000 \\
 4534 & 0.0120 & 0.0129 & 0.0133 & 0.0140 & 4538 & 0.0732 & 0.0807 & 0.0911 & 0.0944 \\
 4541 & 0.0224 & 0.0233 & 0.0237 & 0.0239 & 6332 & 0.1026 & 0.1156 & 0.1340 & 0.1550 \\
 23381 & 0.1058 & 0.1306 & 0.1620 & 0.1786 & 40496 & 0.5907 & 0.6729 & 0.7729 & 0.8465 \\
 40498 & 0.0710 & 0.0958 & 0.1097 & 0.1124 & 40499 & 0.0870 & 0.0955 & 0.1068 & 0.1084 \\
 40536 & 0.0052 & 0.0067 & 0.0105 & 0.0110 & 40670 & 0.0487 & 0.0555 & 0.0590 & 0.0604 \\
 40701 & 0.0084 & 0.0094 & 0.0101 & 0.0106 & 40900 & 0.0065 & 0.0071 & 0.0085 & 0.0095 \\
 40966 & 0.3686 & 0.3947 & 0.4251 & 0.4785 & 40975 & 0.0571 & 0.0683 & 0.0811 & 0.1058 \\
 40978 & 0.0508 & 0.0624 & 0.0958 & 0.1184 & 40979 & 0.2295 & 0.2492 & 0.2728 & 0.2920 \\
 40981 & 0.0676 & 0.0784 & 0.0885 & 0.1001 & 40982 & 0.0913 & 0.1081 & 0.1694 & 0.1911 \\
 40983 & 0.0201 & 0.0217 & 0.0233 & 0.0240 & 40984 & 0.2058 & 0.2160 & 0.2219 & 0.2334 \\
 40994 & 0.0123 & 0.0227 & 0.0327 & 0.0424 & 41138 & 0.0019 & 0.0020 & 0.0020 & 0.0020 \\
 41142 & 0.0212 & 0.0222 & 0.0299 & 0.0312 & 41143 & 0.0243 & 0.0259 & 0.0290 & 0.0333 \\
 41146 & 0.0352 & 0.0394 & 0.0416 & 0.0416 & 41156 & 0.0343 & 0.0439 & 0.0532 & 0.0602 \\
 41157 & 0.1196 & 0.1491 & 0.2029 & 0.2122 & 41159 & 0.0813 & 0.0891 & 0.0943 & 0.0959 \\
 41161 & 0.1058 & 0.1130 & 0.1182 & 0.1199 & 41162 & 0.0008 & 0.0009 & 0.0009 & 0.0009 \\
 41163 & 0.0982 & 0.1054 & 0.1114 & 0.1118 & 41164 & 0.1956 & 0.2208 & 0.2250 & 0.2295 \\
 41165 & 0.0297 & 0.0317 & 0.0329 & 0.0333 & 41212 & 0.0553 & 0.0587 & 0.0630 & 0.0674 \\
 41278 & 0.0132 & 0.0144 & 0.0160 & 0.0168 &  &  &  &  &  \\
\bottomrule
\end{tabular}
\end{table}
\clearpage

\subsection{Complete accuracy curves}
\label{app:complete-curves}
Figures~\ref{fig:grid:lcbench}--\ref{fig:grid:aknn4} collect the complete per-task accuracy trajectories, grouped by benchmark suite.

\begin{figure}[H]
\centering
\includegraphics[width=\textwidth,height=0.85\textheight,keepaspectratio]{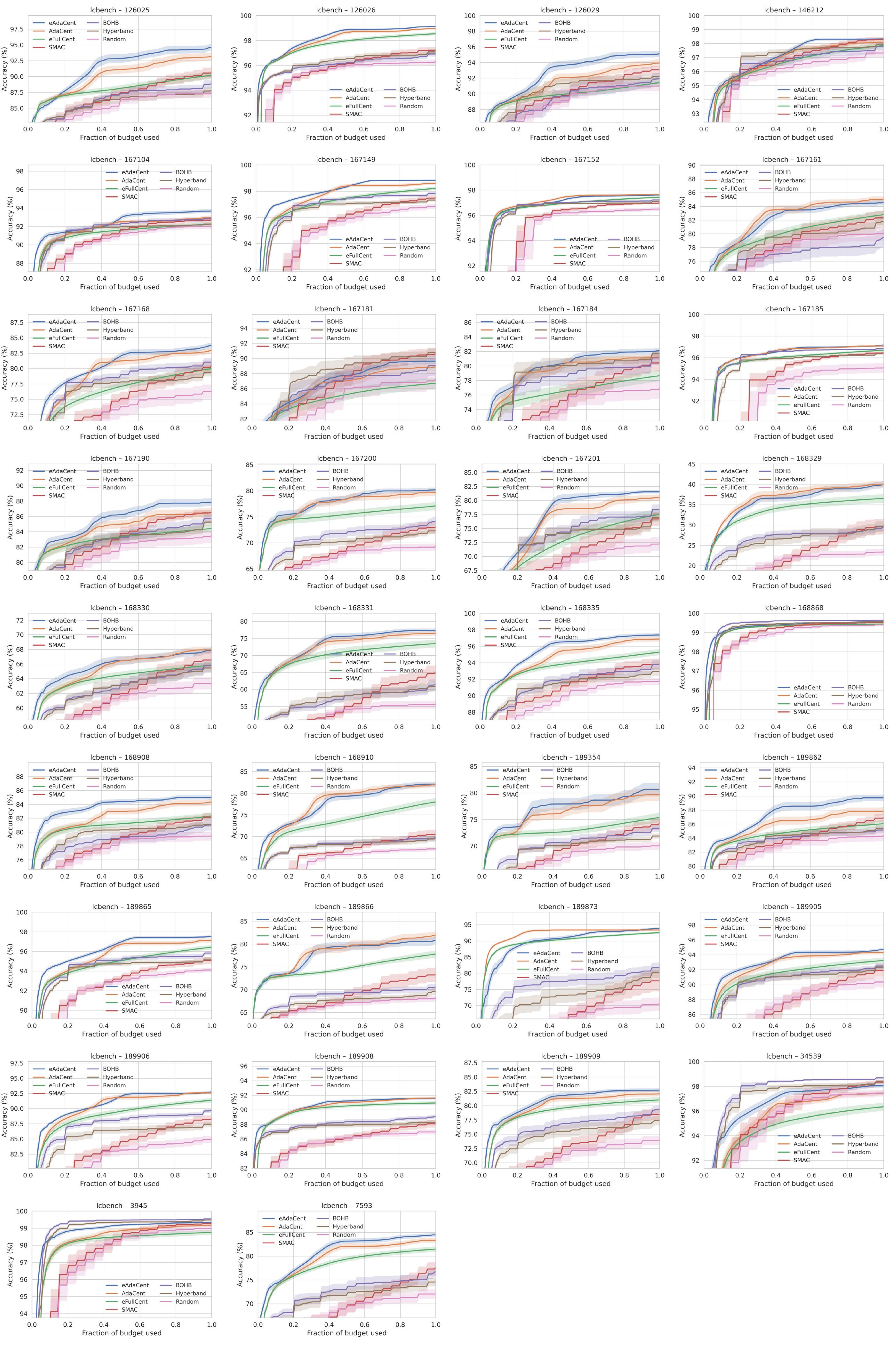}
\caption{\texttt{lcbench} accuracy curves on all instances}
\label{fig:grid:lcbench}
\end{figure}
\clearpage
\begin{figure}[t]
\centering
\includegraphics[width=\textwidth,height=0.91\textheight,keepaspectratio]{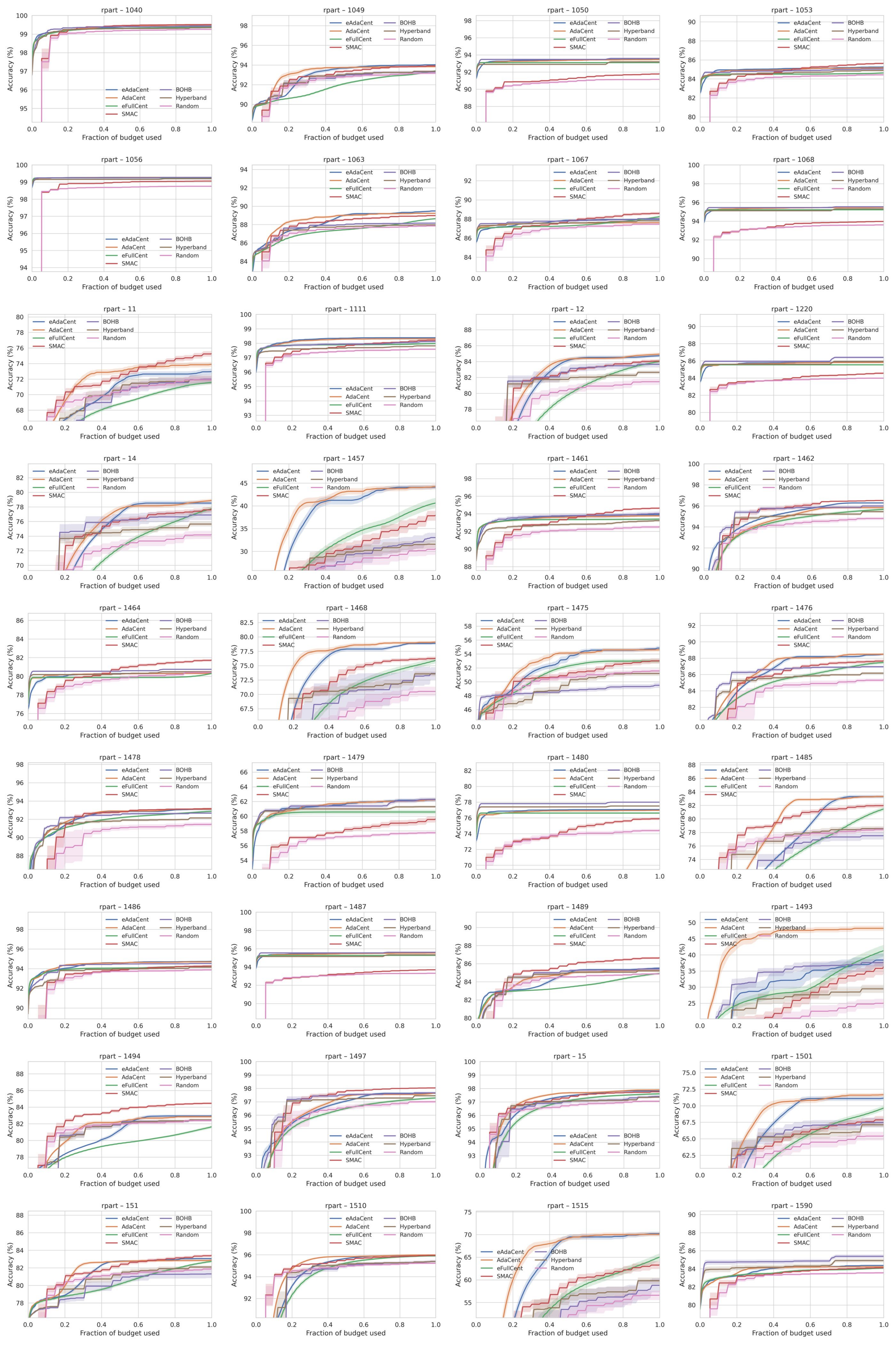}
\caption{\texttt{rbv2\_rpart} accuracy curves on all instances, part 1}
\label{fig:grid:rpart1}
\end{figure}
\clearpage
\begin{figure}[t]
\centering
\includegraphics[width=\textwidth,height=0.91\textheight,keepaspectratio]{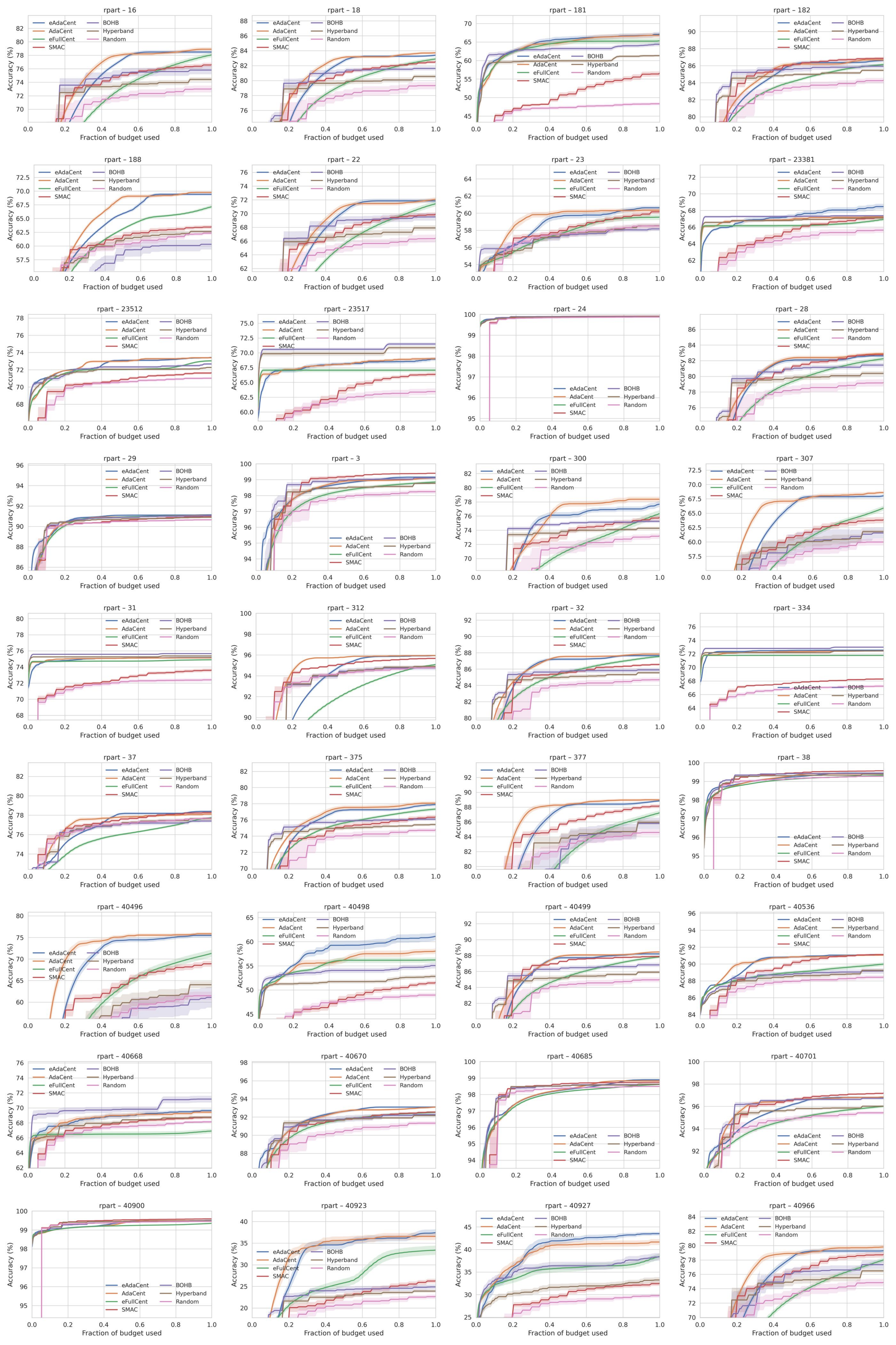}
\caption{\texttt{rbv2\_rpart} accuracy curves on all instances, part 2}
\label{fig:grid:rpart2}
\end{figure}
\clearpage
\begin{figure}[t]
\centering
\includegraphics[width=\textwidth,height=0.91\textheight,keepaspectratio]{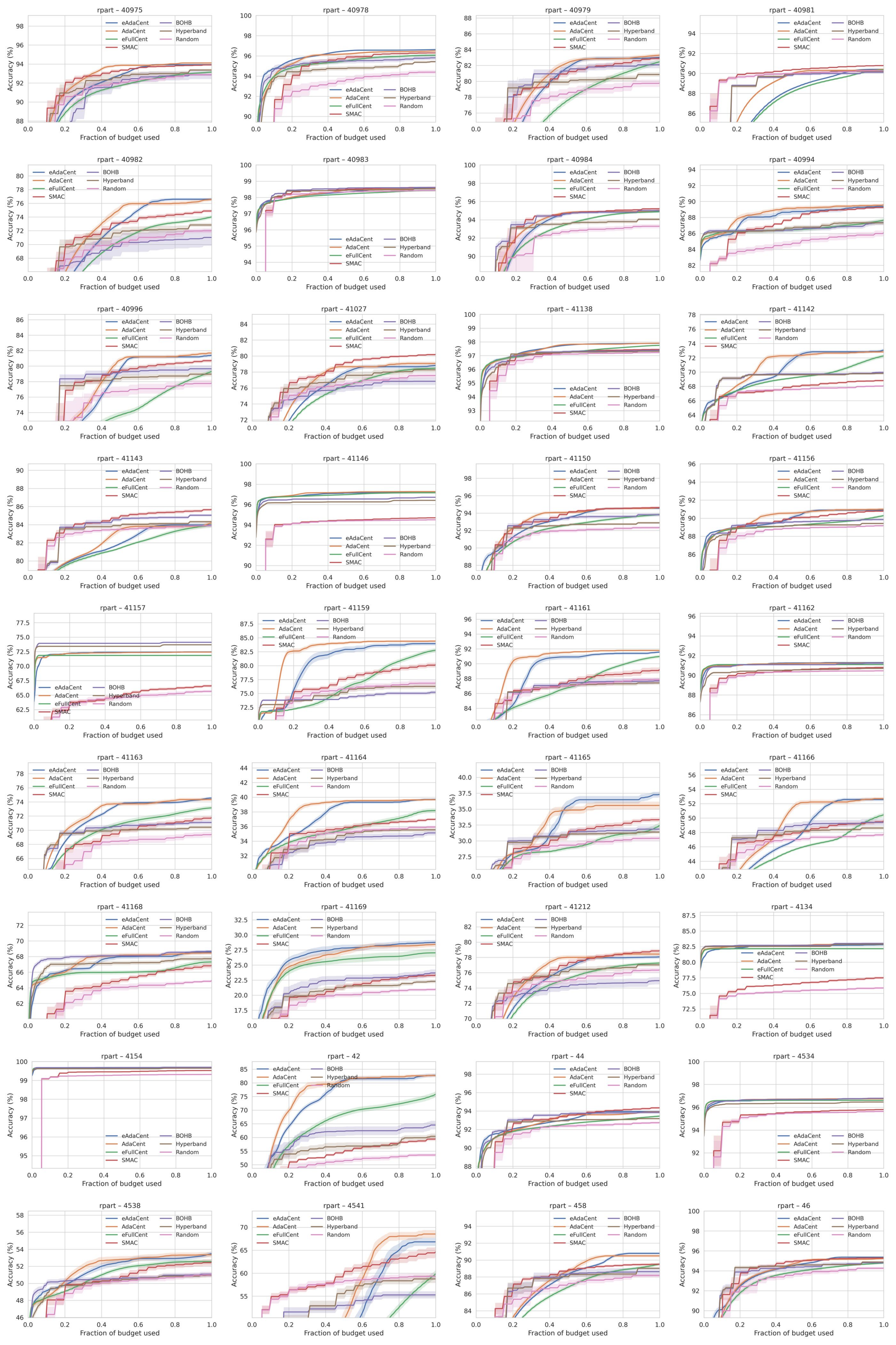}
\caption{\texttt{rbv2\_rpart} accuracy curves on all instances, part 3}
\label{fig:grid:rpart3}
\end{figure}
\clearpage
\begin{figure}[t]
\centering
\includegraphics[width=\textwidth,height=0.91\textheight,keepaspectratio]{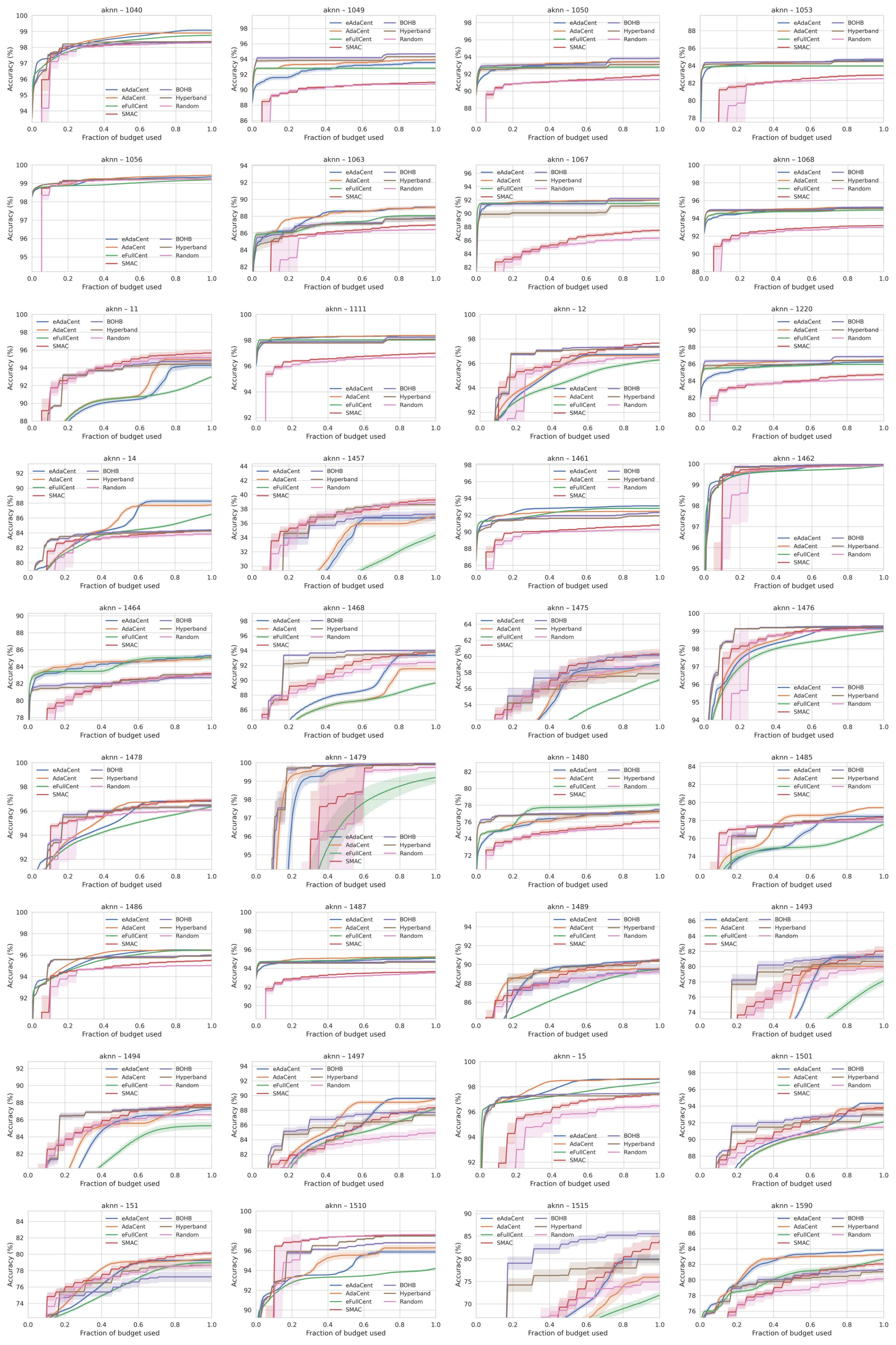}
\caption{\texttt{rbv2\_aknn} accuracy curves on all instances, part 1}
\label{fig:grid:aknn1}
\end{figure}
\clearpage
\begin{figure}[t]
\centering
\includegraphics[width=\textwidth,height=0.91\textheight,keepaspectratio]{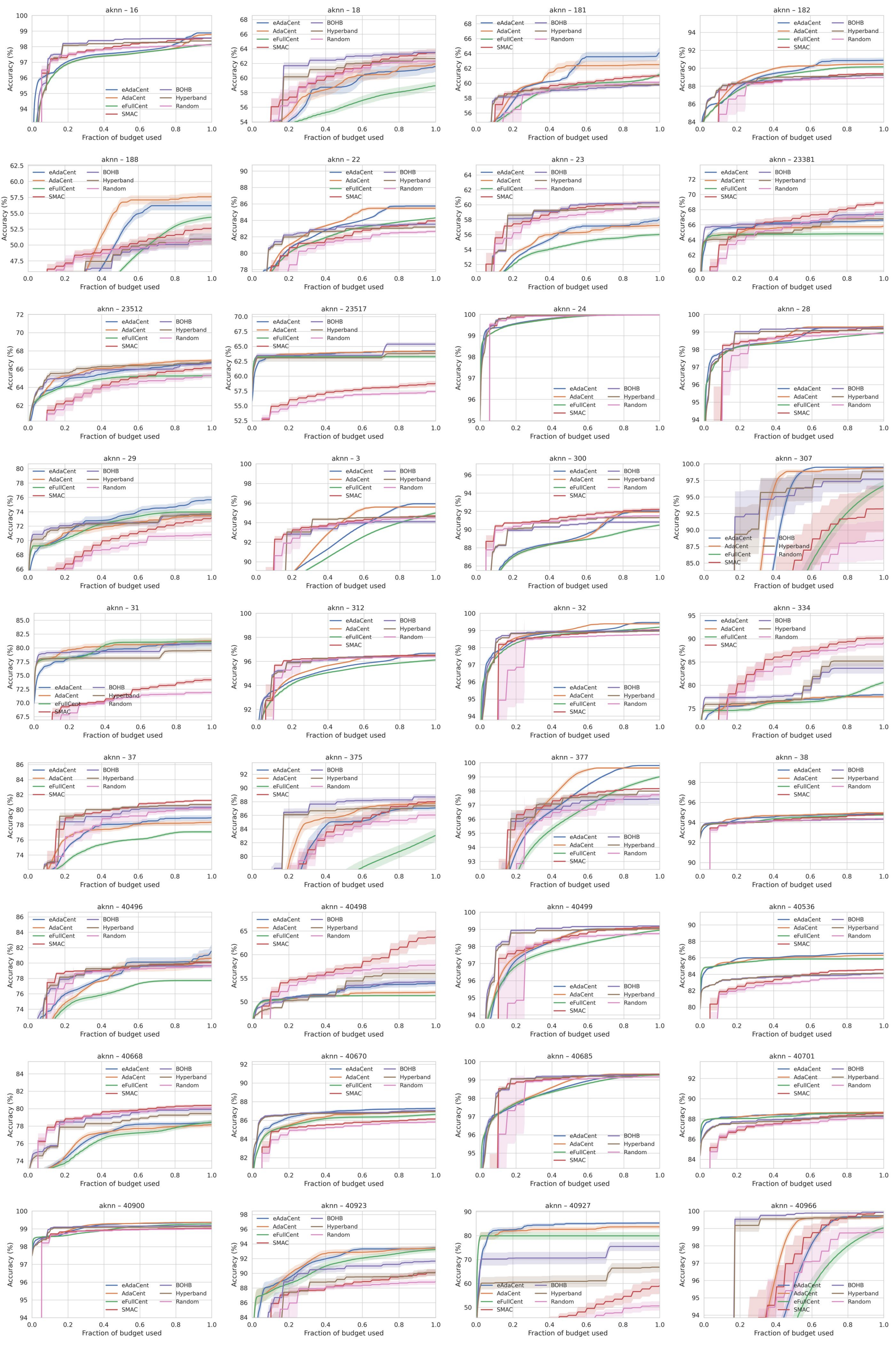}
\caption{\texttt{rbv2\_aknn} accuracy curves on all instances, part 2}
\label{fig:grid:aknn2}
\end{figure}
\clearpage
\begin{figure}[t]
\centering
\includegraphics[width=\textwidth,height=0.91\textheight,keepaspectratio]{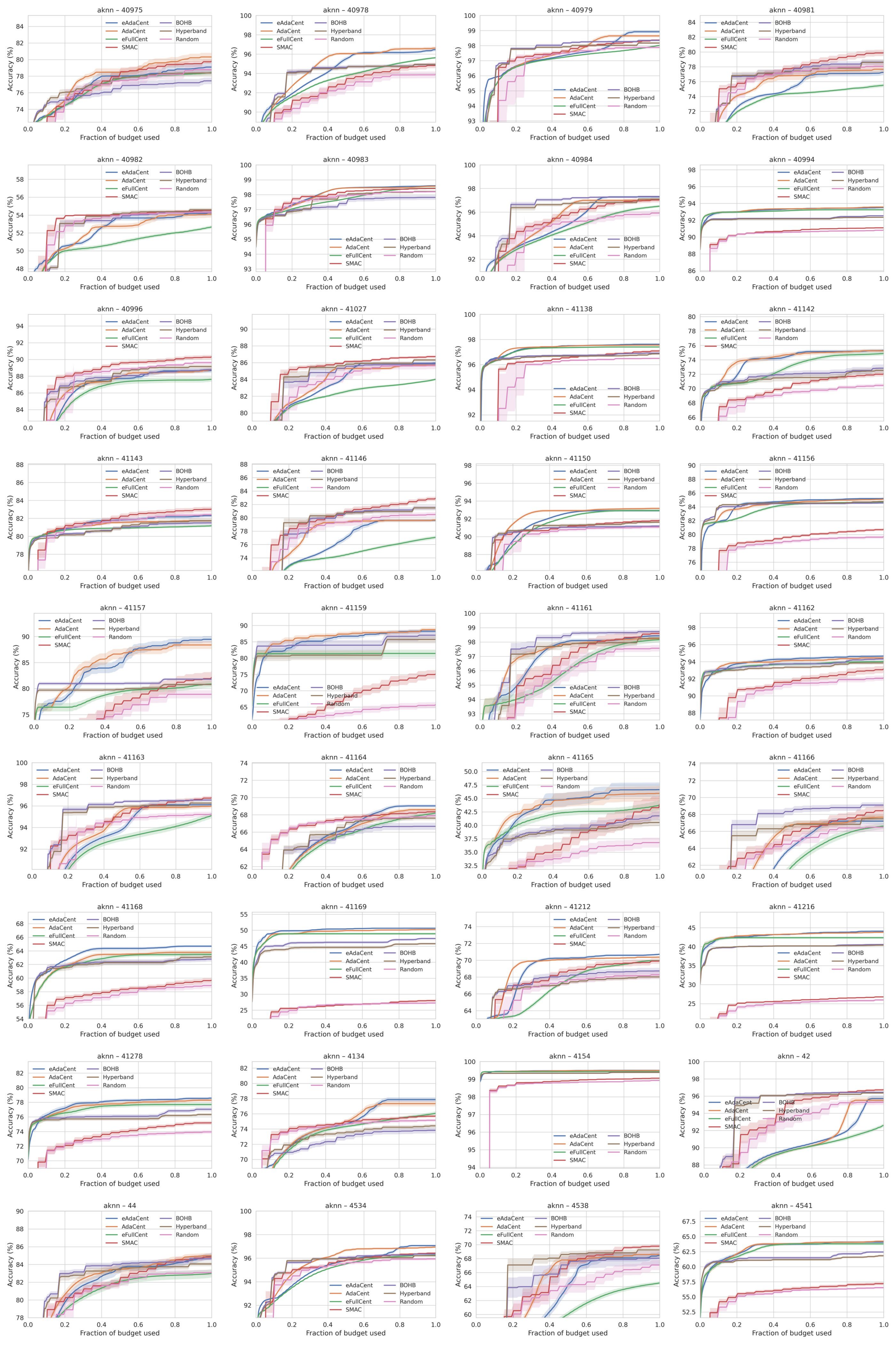}
\caption{\texttt{rbv2\_aknn} accuracy curves on all instances, part 3}
\label{fig:grid:aknn3}
\end{figure}
\clearpage
\begin{figure}[t]
\centering
\setlength{\parskip}{0pt}
\includegraphics[width=0.43\linewidth,viewport={0 934 700 1401},clip]{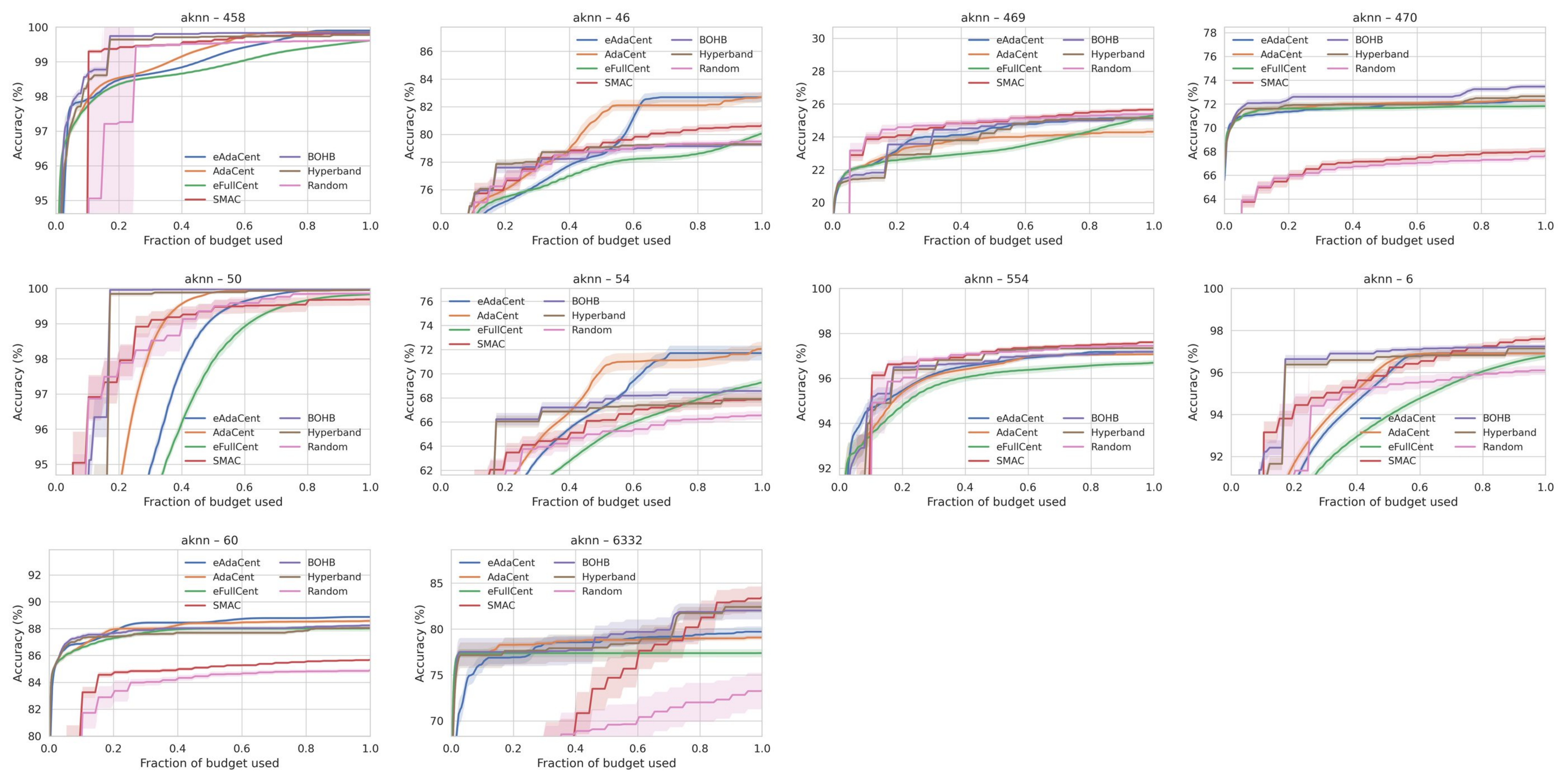}\hspace{0.03\linewidth}
\includegraphics[width=0.43\linewidth,viewport={700 934 1400 1401},clip]{Grids/plots_aknn_grid_4.pdf}\par
\includegraphics[width=0.43\linewidth,viewport={1400 934 2100 1401},clip]{Grids/plots_aknn_grid_4.pdf}\hspace{0.03\linewidth}
\includegraphics[width=0.43\linewidth,viewport={2100 934 2800 1401},clip]{Grids/plots_aknn_grid_4.pdf}\par
\includegraphics[width=0.43\linewidth,viewport={0 467 700 934},clip]{Grids/plots_aknn_grid_4.pdf}\hspace{0.03\linewidth}
\includegraphics[width=0.43\linewidth,viewport={700 467 1400 934},clip]{Grids/plots_aknn_grid_4.pdf}\par
\includegraphics[width=0.43\linewidth,viewport={1400 467 2100 934},clip]{Grids/plots_aknn_grid_4.pdf}\hspace{0.03\linewidth}
\includegraphics[width=0.43\linewidth,viewport={2100 467 2800 934},clip]{Grids/plots_aknn_grid_4.pdf}\par
\includegraphics[width=0.43\linewidth,viewport={0 0 700 467},clip]{Grids/plots_aknn_grid_4.pdf}\hspace{0.03\linewidth}
\includegraphics[width=0.43\linewidth,viewport={700 0 1400 467},clip]{Grids/plots_aknn_grid_4.pdf}\par
\caption{\texttt{rbv2\_aknn} accuracy curves on all instances, part 4}
\label{fig:grid:aknn4}
\end{figure}

\end{document}